\documentclass[11pt]{article}
\usepackage[labelfont=bf]{caption}
\usepackage[colorlinks=true,citecolor=blue]{hyperref} 
\usepackage{hhline}
\usepackage{geometry}

\usepackage[utf8]{inputenc} 
\usepackage[T1]{fontenc}    
\usepackage{hyperref}       
\usepackage{url}            
\usepackage{booktabs}       
\usepackage{color}
\usepackage{amsfonts}       
\usepackage{nicefrac}       
\usepackage{microtype}      
\usepackage{appendix}
\usepackage{hyperref}
\usepackage{amsthm}
\usepackage{amsmath}
\usepackage{amssymb}
\usepackage{amsfonts}
\usepackage{enumerate}
\usepackage{algorithm}
\usepackage{algorithmic} 
\usepackage{xspace}
\usepackage{mathabx}
\usepackage{graphicx}
\usepackage{eufrak}
\usepackage{multirow}
\usepackage{makecell}

\newtheorem{definition}{Definition}

\newtheorem{thm}{Theorem}
\newtheorem{lemma}[thm]{Lemma}
\newtheorem{proposition}[thm]{Proposition}

\newcommand{\UCBS}{{\sc UCB-MultiStage}\xspace}
\newcommand{\UCBSA}{{\sc UCB-MultiStage-Advantage}\xspace}

\def\clip{\mathrm{clip}}
\def\udl#1{\underline{#1
} }
\def\reff {\mathrm{ref}}
\def\Reff {\mathrm{REF}}
\title{Model-Free Reinforcement Learning: from Clipped Pseudo-Regret to Sample Complexity}
\author{%
	Zihan Zhang \\
	Department of Automation\\
	Tsinghua University\\
	\texttt{\quad zihan-zh17@mails.tsinghua.edu.cn    \quad        } \\
	\and
	Yuan Zhou \\
	Department of ISE\\
	University of Illinois at Urbana-Champaign \\
	\texttt{yuanz@illinois.edu} \\
	\and
	Xiangyang Ji \\
	Department of Automation\\
	Tsinghua University \\
	\texttt{xyji@tsinghua.edu.cn} \\
}
\begin{document}
\maketitle
\begin{abstract}
In this paper we consider the problem of learning an $\epsilon$-optimal policy for a discounted Markov Decision Process (MDP). Given an MDP with $S$ states, $A$ actions, the discount factor $\gamma \in (0,1)$, and an approximation threshold $\epsilon > 0$, we provide a model-free  algorithm to learn an $\epsilon$-optimal policy with sample complexity   $\tilde{O}(\frac{SA\ln(1/p)}{\epsilon^2(1-\gamma)^{5.5}})$ \footnote{In this work, the notation $\tilde{O}(\cdot)$ hides poly-logarithmic factors of $S,A,1/(1-\gamma)$, and $1/\epsilon$.} and success probability $(1-p)$. For small enough $\epsilon$, we show an improved algorithm with sample complexity   $\tilde{O}(\frac{SA\ln(1/p)}{\epsilon^2(1-\gamma)^{3}})$. While the first bound improves upon all known model-free algorithms and model-based ones with tight dependence on $S$, our second algorithm beats all known sample complexity bounds and matches the information theoretic lower bound up to logarithmic factors.

\end{abstract}
\allowdisplaybreaks
\section{Introduction}
Reinforcement learning (RL) \cite{Burnetas1997Optimal} studies the problem of how to make sequential decisions to learn and act in unknown environments (which is usually modeled by a Markov Decision Process (MDP)) and maximize the collected rewards. There are mainly two types of algorithms to approach the RL problems: model-based algorithms and model-free algorithms. Model-based RL algorithms keep explicit description of the learned model and make decisions based on this model. In contrast, model-free algorithms only maintain a group of value functions instead of the complete model of the system dynamics. Due to their space- and time-efficiency, model-free RL algorithms have been getting popular in a wide range of practical tasks (e.g., DQN~\cite{mnih2015human}, TRPO~\cite{schulman2015trust}, and A3C~\cite{mnih2016asynchronous}). 

In RL theory, model-free algorithms are explicitly defined to be the ones whose space complexity is always sublinear relative to the space required to store the MDP parameters \cite{jin2018q}. For tabular MDPs (i.e., MDPs with finite number of states and actions, usually denoted by $S$ and $A$ respectively), this requires that the space complexity to be $o(S^2A)$. Motivated by the empirical effectiveness of model-free algorithms, the intriguing question of whether model-free algorithms can be rigorously proved to perform as well as the model-based ones has attracted much attention and been studied in the settings such as regret minimization for episodic MDPs~\cite{azar2017minimax,jin2018q,zhang2020almost}).


In this work, we study the \textsc{Probably-Approximately-Correct-RL} (\textsc{PAC-RL}) problem, i.e., to designing an algorithm for learning an approximately optimal policy. We will focus on designing the model-free algorithms, and under the model of discounted tabular MDPs with a discount factor $\gamma$. The RL algorithm runs for infinitely many time steps. At each time step $t$, the RL agent learns a policy $\pi_t$ based on the information collected before time $t$, observes the current state $s_t$, makes an action $a_t = \pi_t(s_t)$, receives the reward $r_t$ and transits to the next state $s_{t+1}$ according to the underlying environments. The goal of the agent is to learn the policy $\pi_t$ at each time $t$ so as to maximize the \emph{$\gamma$-discounted accumulative reward} $V^{\pi_t}(s_t)$. 
More concretely, we wish to minimize the \emph{sample complexity} for the agent to learn an $\epsilon$-optimal policy, which is defined to be the number of time steps that $V^{\pi_t}(s_t) < V^*(s_t) - \epsilon$, where $V^*$ is the optimal discounted accumulative reward that starts with $s_t$, and the formal definitions of both $V^{\pi}$ and $V^*$ can be found in Section~\ref{sec:prelim}.

The \textsc{PAC-RL} addresses the important problem about how many trials are required to learn a good policy. We also note that in the \textsc{PAC-RL} definition, the exploration at each time step has to align with the learned policy (i.e., $a_t = \pi_t(s_t)$). This is stronger than the usual PAC learning definition in other online learning settings such as multi-armed bandits (see, e.g., \cite{even2006action}) and \textsc{PAC-RL} with a simulator (see Section~\ref{sec:additional-related-work}), where the exploration actions can be arbitrary and may incur a large regret compared to the optimum.

Quite a few algorithms have been proposed over the past nearly two decades for the \textsc{PAC-RL} problem. For model-based algorithms,  MoRmax \cite{szita2010model-based} achieves the $\tilde{O}(\frac{SA \ln (1/p)}{\epsilon^2(1-\gamma)^6})$ sample complexity, and UCRL-$\gamma$  \cite{lattimore2012pac} achieves $\tilde{O}(\frac{S^2A \ln (1/p)}{\epsilon^2(1-\gamma)^3})$. It is also worthwhile to mention that R-max \cite{brafman2003r-max} was designed for learning the more general stochastic games and achieves the $\tilde{O}(\frac{S^2A \ln (1/p)}{\epsilon^3(1-\gamma)^6})$ sample complexity in our setting (as analyzed in \cite{kakade2003sample}). Unfortunately, none of these algorithms matches the information theoretical lower bound $\Omega(\frac{SA}{\epsilon^2(1-\gamma)^3})$ proved by \cite{lattimore2012pac}. On the model-free side, known bounds are even less optimal -- the delayed $Q$-learning algorithm proposed by \cite{strehl2006pac} achieves the sample complexity of $\tilde{O}(\frac{SA\ln (1/p)}{\epsilon^4 (1-\gamma)^8})$, and recent work \cite{dong2019q} made an improvement to $\tilde{O}(\frac{SA\ln (1/p)}{\epsilon^2 (1-\gamma)^7})$ via a more carefully designed $Q$-learning variant. Besides the results above, \cite{pazis2016improving} provided $\tilde{O}\left( \frac{S^2A}{\epsilon^2(1-\gamma)^4}\right)$ sample complexity. However, their algorithm consumes $\tilde{O}(\frac{SAH^4}{\epsilon^2})$ space cost and $\tilde{O}\left( \frac{SA^2H^4}{\epsilon^2} \right)$ computational cost each step, which is far beyond the cost of both model-based and model-free algorithms when $\epsilon$ is small.


\subsection{Our Results}

We design a model-free algorithm that achieves asymptotically optimal sample complexity, as follows.
\begin{thm}\label{thm2}
	We present a model-free algorithm \UCBSA, such that given a discounted MDP with $S$ states, $A$ actions, and the discount factor $\gamma$, for any approximation threshold $\epsilon \in (0, 1/\mathrm{poly}(S, A, 1/(1-\gamma)))$ and failure probability parameter $p$, with probability $(1-p)$, the sample complexity to learn an $\epsilon$-optimal policy with \UCBSA is bounded by $\tilde{O}(\frac{SA\ln(1/p)}{ \epsilon^2 (1-\gamma)^3 })$.
\end{thm}

In the theorem statement, $\mathrm{poly}(S, A, 1/(1-\gamma))$ stands for a universal polynomial that is independent of the MDP. Our \UCBSA algorithm is model-free, which uses only $O(SA)$ space , and its time complexity per time step is $O(1)$. In contrast, the model-based algorithms have to consume $\Omega(S^2A)$ space.
For asymptotically small $\epsilon$, the sample complexity of \UCBSA matches the information theoretic lower bound of $\Omega(\frac{SA}{\epsilon^2(1-\gamma)^3})$ up to poly-logarithmic terms, and improves upon all known algorithms in literature, even including the model-based ones. In Appendix~\ref{app:comp}, we present a tabular view of the comparison between our algorithms and the previous works.

To prove Theorem~\ref{thm2}, we make two main technical contributions. The first one is a novel relation between sample complexity and the so-called \emph{clipped pseudo-regret}, which can also be viewed as the clipped Bellman error of the learned value function and policy at each time step. This relation enables us to reduce the sample complexity analysis to bounding the clipped pseudo-regret. Our second technique is a \emph{multi-stage update rule}, where the visits to each state-action pair are partitioned according to two types of stages. An update to the $Q$-function is triggered only when a stage of either type has concluded. The lengths of the two types of stages are set by different choices of parameters so that we can reduce the clipped pseudo-regret while still maintaining a decent rate to learn the value function. Finally, we also spend much technical effort to incorporate the variance reduction technique for RL via \emph{reference-advantage decomposition} introduced in the recent work \cite{zhang2020almost}.

A more detailed overview of our techniques is available in Section~\ref{sec:tech-overview}. Since the proof of Theorem~\ref{thm2} is rather involved, we will first provide a proof of the following weaker statement, and defer the full proof of Theorem~\ref{thm2} to Appendix~\ref{app:proof-thm-2}.


\begin{thm} \label{thm1}
	We present a simpler model-free algorithm \UCBS, such that for any approximation threshold $\epsilon\in (0,\frac{1}{1-\gamma}]$ and any failure probability parameter $p$, with probability $(1-p)$,
the sample complexity to learn an $\epsilon$-policy with \UCBS is bounded by $\tilde{O}(\frac{SA \ln (1/p)}{ \epsilon^2 (1-\gamma)^{5.5} })$. 
\end{thm}

We highlight that the sample complexity bound in Theorem~\ref{thm1} holds for every possible $\epsilon\in (0,\frac{1}{1-\gamma}]$. Although the dependency on $\gamma$ becomes $(1-\gamma)^{-5.5}$,
\UCBS still beats all known model-free and model-based algorithms.
The proof of Theorem~\ref{thm1} does not rely on the variance reduction technique based on reference-advantage decomposition \cite{zhang2020almost}, but is sufficient to illustrate both of our main technical contributions.

\subsection{Additional Related Works}\label{sec:additional-related-work}

The \textsc{PAC-RL} problem has also been extensively studied under the setting of finite-horizon episodic MDPs~\cite{dann2015sample,dann2017unifying, dann2019policy}, where the sample complexity is defined as the number of episodes in which the policy is not $\epsilon$-optimal. Assuming $H$ is the length of an episode, the optimal sample complexity bound is $\tilde{O}(\frac{SAH^2 \ln(1/p)}{\epsilon^2})$, proved by \cite{dann2019policy}. 
Note that the sample complexity bounds for finite-horizon episodic MDP do not imply sample complexity bounds for infinite-horizon discounted MDP because one $\epsilon$-optimal episode may contain  non-$\epsilon$-optimal steps.
Also we note that existing algorithms for the finite-horizon case are model-based. It is still an open problem whether model-free algorithm can achieve near-optimal sample complexity bound for the finite-horizon case.

Much effort has also been made to study the PAC learning problem for discounted infinite-horizon MDPs, with the access to a generative model (a.k.a., a simulator). In this problem, the agent can query the simulator to draw a sample $s' \sim P(\cdot |s, a)$ for any state-action pair $(s, a)$, and the goal is to output an $\epsilon$-optimal policy (with probability $(1-p)$) at the end of the algorithm. This problem has been studied in \cite{even2003learning, ghavamzadeh2011speedy,azar2012sample, sidford2018variance, sidford2018near}, and \cite{sidford2018near} achieves the almost tight sample complexity $\tilde{O}(\frac{SA \ln (1/p)}{(1-\gamma)^3})$.



\section{Preliminaries} \label{sec:prelim}
A discounted Markov Decision Process is given by the five-tuple $M = \langle   \mathcal{S},\mathcal{A},P,r,\gamma  \rangle$, where $\mathcal{S}\times\mathcal{A}$ is the state-action space, $P$ is the transition probability matrix, $r$ is the deterministic reward function\footnote{It is easy to generalize our results to stochastic reward functions.} and $\gamma\in (0,1)$ is the discount factor. The RL agent interacts with the environment for infinite number of times. At the $t$-th time step, the agent learns a policy $\pi_t$ based on the samples collected before time $t$, observes $s_{t}$, executes $a_{t} = \pi_t(s_t)$, receives the reward $r(s_{t},a_{t})$, and then transits to $s_{t+1}$ according to $P(\cdot|s_{t},a_{t})$.  

Given a deterministic\footnote{In this work, we mainly consider deterministic policies since the optimal value function can be achieved by a deterministic policy.} stationary policy $\pi :\mathcal{S}\to \mathcal{A}$, the  value function  and $Q$ function  are defined as
	\begin{align}
	&V^{\pi}(s) = \mathbb{E}\left[ \sum_{t=1}^{\infty} \gamma^{t-1}r(s_{t},\pi(s_{t})  ) \Big|s_{1}=s,s_{t+1}\sim P(\cdot|s_{t},\pi(s_{t}))     \right] \nonumber
	\\& Q^{\pi}(s,a)  = r(s,a)+\gamma P(\cdot|s,a)^{\top}V^{\pi}=r(s,a)+P_{s,a}V^{\pi},\nonumber
	\end{align}
	where we use $xy$ to denote $x^{\top}y$ for $x$ and $y$ of the same dimension and use $P_{s,a}$ to denote $P(\cdot|s,a)$ for simplicity. The optimal value function is given by $V^*(s) = \sup_{\pi}V^{\pi}(s) $ and the optimal $Q$-function is defined to be $Q^*(s,a)=r(s,a)+ P_{s,a}V^*$ for any $(s,a)\in \mathcal{S}\times\mathcal{A}$. We present below the formal definitions for  sample complexity and \textsc{PAC-RL} .

\begin{definition}[$\epsilon$-sample complexity] Given an algorithm $\mathcal{G}$ and $\epsilon \in (0,\frac{1}{1-\gamma}]$, the $\epsilon$-sample complexity for $\mathcal{G}$ is $\sum_{t\geq 1}\mathbb{I}\left[    V^*(s_t)-V^{\pi_{t}}(s_t)>\epsilon   \right]$.
\end{definition}

\begin{definition}[$(\epsilon,p)$-\textsc{PAC-RL}]
An algorithm $\mathcal{G}$ is said to be  $(\epsilon,p)$-\textbf{PAC-RL} (Probably Approximately Correct in RL) if for any $\epsilon\in (0,\frac{1}{1-\gamma}],p>0$, with probability $1-p$, the sample complexity of $\mathcal{G}$ is bounded by some polynomial in $(S,A,\frac{1}{\epsilon},\frac{1}{1-\gamma},\ln(\frac{1}{p}))$.
\end{definition}

When $\epsilon$ and $p$ are clear in the context, we simply write $(\epsilon,p)$-\textsc{PAC-RL} and $\epsilon$-sample complexity as \textsc{PAC-RL} and sample complexity respsectively.
	The goal is to propose an \textsc{PAC-RL} algorithm to minimize the sample complexity.

\section{Technical Overview} \label{sec:tech-overview}

Both of our algorithms are variants of $Q$-learning, where the value function $V$ and the $Q$-function  are maintained. For each time $t$, we use $V_t$ and $Q_t$ to denote the corresponding functions at the beginning of the time step. The learned policy $\pi_t$ will always be the greedy policy based on $Q_t$, i.e., $\pi_t(s) = \arg\max_a Q_{t}(s, a)$ for all $s \in \mathcal{S}$.

\paragraph{Reducing Sample Complexity to  Bounding the Clipped Pseudo-Regret.}
For any time $t$, define the \emph{pseudo-regret} vector $\phi_{t}$ to be the vector such that $\phi_{t}(s) =V_{t}(s) -( r(s,\pi_{t}(s))+ \gamma P_{s,\pi_{t}(s)}V_{t} )$. We now outline our first technical idea that the sample complexity can be bounded by the total clipped pseudo-regret, approximately in the form of \eqref{eq_sec3_4} (up to a $\epsilon^{-1}$ factor and an additive error term).

Note that $\phi_t$ can also be viewed as the Bellman error vector of the value function $V_t$ and the policy $\pi_t$. Let $P_{\pi_{t}}$ be the matrix such that $P_{\pi_{t}}(s) = P_{s,\pi_{t}(s)}$ for any $s\in \mathcal{S}$. By Bellman equation we have that
\begin{align}
V_{t}-V^{\pi_{t}}& = \gamma P_{\pi_{t}}(V_{t}-V^{\pi_{t}})   +\phi_{t}  = (\gamma P_{\pi})^2(V_{t}-V^{\pi_{t}})  + \gamma P_{\pi_{t}}\phi_{t}  +\phi_{t}   = \dots  =  \sum_{i=0}^{\infty}(\gamma P_{\pi_{t}})^i \phi_{t}  .  \nonumber
\end{align}

Therefore, if $V_{t}(s_{t})-V^{\pi_{t}}(s_{t})>\epsilon$, then by an averaging argument we have that for any $M > 1$, $\mathbf{1}_{s_{t}}^{\top} \sum_{i=0}^{\infty}(\gamma P_{\pi_{t}})^i  \mathrm{clip}( \phi_{t},\frac{\epsilon(1-\gamma)}{M}   )>\frac{(M-1)\epsilon}{M}$, where $\mathbf{1}_{s_{t}}$ is the unit vector with the only non-zero entry at $s_t$, and we define $\mathrm{clip}(x,y) = x\mathbb{I}\left[ x\geq  y\right]$ for $x,y\in \mathbb{R}$ and $\mathrm{clip}(x,y) = [\mathrm{clip}(x_1,y),\dots,\mathrm{clip}(x_n,y) ]^{\top}$ for $x = [x_1,\dots,x_n]^{\top}\in \mathbb{R}^n$. 
For any $H = \Theta(\ln (((1-\gamma)\epsilon)^{-1})/(1-\gamma))$,  it then follows that
\begin{align}
\mathbb{I}\left[  V_{t}(s_{t})-V^{\pi_{t}}(s_{t})>\epsilon  \right]\epsilon \leq O \left(\mathbf{1}_{s_{t}}^\top \sum_{i=0}^{H-1}(\gamma P_{\pi_{t}}  )^{i}\mathrm{clip}(\phi_{t},\epsilon(1-\gamma)/M)\right).  \label{eq_sec_3_3}
\end{align}
We now sum up \eqref{eq_sec_3_3} over all time steps $t$. If we can carefully design the algorithm so that $\pi_t$, $V_t$ (and therefore $\phi_t$) do not change frequently, we have $\pi_{t} = \pi_{t+i}$ and $\phi_t = \phi_{t+i}$ for small enough $i$ and most $t$, and therefore we can upper bound  $\sum_{t\geq 1}\mathbb{I}\left[  V_{t}(s_{t})-V^{\pi_{t}}(s_{t})>\epsilon  \right]\epsilon$ by the order of
\begin{align}
\sum_{t\geq 1}\mathbf{1}_{s_{t}}^\top\sum_{i=0}^{H-1}(\gamma P_{\pi_{t+i}}  )^{i}\mathrm{clip}(\phi_{t+i},\epsilon(1-\gamma)/M  ) &\leq \sum_{t\geq 1}\mathbf{1}_{s_{t}}^\top\sum_{i=0}^{H-1}( P_{\pi_{t+i}}  )^{i}\mathrm{clip}(\phi_{t+i},\epsilon(1-\gamma)/M  ) \nonumber \\
& \approx O(H) \cdot  \sum_{t\geq 1} \mathrm{clip}(\phi_{t}(s_t),\epsilon(1-\gamma)/M  ), \label{eq_sec3_4}
\end{align}
where the approximation \eqref{eq_sec3_4} also uses the assumption that $\pi_{t} = \pi_{t+i}$ and $\phi_t = \phi_{t+i}$ hold for most $t$ and $i$. In Lemma~\ref{lemma_bd_beta}, we formalize this intuition and show that if we set $M = 8H(1-\gamma)$, the sample complexity $\sum_{t\geq 1}\mathbb{I}\left[  V_{t}(s_{t})-V^{\pi_{t}}(s_{t})>\epsilon  \right]$ can be upper bounded by $O(H/\epsilon) \cdot  \sum_{t\geq 1} \mathrm{clip}(\phi_{t}(s_t),\epsilon(1-\gamma)/M  )$ (plus an additive error), and therefore we only need to upper bound the total clipped pseudo-regret.



\paragraph{The Multi-Stage Update Rule.} We propose a multi-stage update rule for the value and $Q$-function. For each state-action pair $(s, a)$,  the samples are partitioned into consecutive stages. When a stage is filled, we update $Q(s, a)$ and $V(s)$ according to the samples in the stage via the usual value iteration method. The most interesting aspect about our method is that two types of stages, namely the \emph{type-\uppercase\expandafter{\romannumeral1} and type-\uppercase\expandafter{\romannumeral2} stages}, are introduced. More concretely, the length of the $j$-th type-\uppercase\expandafter{\romannumeral1} stage is roughly $\check{e}_j \approx H (1 + 1/H)^{j/B}$ and the length of the $j$-th type-\uppercase\expandafter{\romannumeral2} stage is roughly $\bar{e}_j \approx H (1 + 1/H)^{j}$, where the more precise definition and detailed description of how these stages are incorporated in the algorithm are provided in Section~\ref{sec:UCBS-alg} and $B \geq 1$ will be set later. (Also note that throughout the paper we will use ` $\check{~~}$' to denote the quantities related to the type-\uppercase\expandafter{\romannumeral1} stage, and use `$\bar{~~}$' to denote the quantities related to the type-\uppercase\expandafter{\romannumeral2} stage.)

We note that the recent work \cite{zhang2020almost} designed a (single-)stage-based model-free RL algorithm for regret minimization. Our type-\uppercase\expandafter{\romannumeral2} stage is similar to their work, and its goal is to make sure that the value function is learned at a decent rate. In contrast, our type-\uppercase\expandafter{\romannumeral1} stage is new: it is shorter than the type-\uppercase\expandafter{\romannumeral2} stage, so that triggers more frequent updates and helps to reduce the difference between the value functions learned in neighboring type-\uppercase\expandafter{\romannumeral1} stages. The two types of stages work together to reduce the clipped pseudo-regret, and therefore achieve low sample complexity.

To better explain the intuition and motivate the type-\uppercase\expandafter{\romannumeral1} stage, let us consider a fixed state-action pair $(s, a)$. Suppose at time step $(t-1)$, $(s, a)$ is visited and the visit number reaches the end of a type-\uppercase\expandafter{\romannumeral1} stage, then the following update is triggered:
 \begin{align}
 Q_{t}(s,a) \leftarrow \min \{ r(s,a)+ \check{b}+\frac{\gamma}{\check{n}}\sum_{i=1}^{\check{n}}V_{\check{l}_i}(s_{\check{l}_i}+1),  \ Q_{t-1}(s,a) \} ,  \nonumber 
 \end{align}
 where $\check{n}$ is the number of samples in this stage, $\check{l}_i$ is time of the $i$-th sample in the stage, and $\check{b}$ denotes the exploration bonus. Thanks to the update rule, $V_{t}$ and $Q_t$ are non-increasing in $t$. By concentration inequalities and the proper design of $\check{b}$, we get 
\begin{align}
Q_{t}(s,a)  \leq   r(s,a)+ 2\check{b} +   P_{s,a}(\frac{\gamma}{\check{n}}\sum_{i=1}^{\check{n}}V_{\check{l}_i}) &\leq  r(s,a)+ 2\check{b}+\gamma P_{s,a}V_{t}+ \gamma P_{s,a}(\frac{1}{\check{n}}\sum_{i=1}^{\check{n}}V_{\check{l}_i}-V_{t}) \label{eq_sec3_1a} 
\\ & \leq r(s,a)+ 2\check{b} +\gamma P_{s,a}V_{t} +\gamma P_{s,a}(V_{\underline{t}}-V_{\overline{t}}),
\end{align}
where $\underline{t} = \min_{i}\check{l}_i$ is the start time of the stage and  $\overline{t}$ is the start time of the next stage. By the definition of  $\phi_{t}(s)$ and an averaging argument, we have that
\begin{align}
\mathrm{clip}(\phi_{t}(s), \epsilon(1-&\gamma)/M) \leq  \mathrm{clip}(2\check{b} + \gamma P_{s,a}(V_{\underline{t}}-V_{\overline{t}} ), \epsilon(1-\gamma)/M) \nonumber\\
&\leq  2\mathrm{clip}(2\check{b}, \epsilon(1-\gamma)/(2M)) + O(\gamma) \cdot  P_{s,a} \clip(V_{\underline{t}}-V_{\overline{t}}, \epsilon(1-\gamma)/(2M)).  \label{eq_sec3_2}
\end{align}

We now discuss how to deal with the two terms in \eqref{eq_sec3_2}, and how the parameter $B$ affects the bounds.

\noindent \underline{\it Bounding the second term of \eqref{eq_sec3_2}.} We first focus on the second term ($P_{s,a} \clip(V_{\underline{t}}-V_{\overline{t}}, \epsilon(1-\gamma)/(2M))$) in \eqref{eq_sec3_2}. For each $j$, let $t_j = t_j(s,a)$ be the start time of the $j$-th stage of $(s,a)$. The total contribution of the second term in \eqref{eq_sec3_2} is bounded by the order of
\begin{align}
\sum_{s, a} \sum_j \check{e}_j \cdot P_{s,a} \clip((V_{t_{j-1}(s,a)}-V_{{t_{j+1}(s,a)}} ), \epsilon(1-\gamma)/(2M)). \label{eq_sec3_2a}
\end{align}
Thanks to the updates triggered by the type-\uppercase\expandafter{\romannumeral2} stages, $V_t$ converges to $V^*$ at a rate that is independent of $B$. Increasing $B$ will shorten the length of the type-\uppercase\expandafter{\romannumeral1} stages, making $V_{t_{j-1}(s,a)}$ closer to $V_{{t_{j+1}(s,a)}}$, and reduce the magnitude of \eqref{eq_sec3_2a}. In Lemma~\ref{lemma_bound_alpha}, we formalize this intuition and show that when $M = 8H(1-\gamma)$, \eqref{eq_sec3_2a} can be upper bounded by $\tilde{O}(SAH^5\ln (1/p)/(\epsilon B))$. Therefore, choosing a large enough $B$ will eliminate the $H$ factors in the numerator. 

\noindent \underline{\it Bounding the first term of \eqref{eq_sec3_2}.} On the other hand, however, a larger $B$ means smaller number of samples in the type-\uppercase\expandafter{\romannumeral1} stages,  leads to a bigger estimation variance, and therefore forces us to choose a greater exploration bonus $\check{b}$. More precisely, using the design of $\check{b}$ defined in Algorithm~\ref{alg1}, the total contribution of the first term in \eqref{eq_sec3_2} is $\tilde{O}(SAB\ln (1/p)/(\epsilon(1-\gamma)^4))$. We have to choose $B = \Theta(\sqrt{H})$ to achieve the optimal balance between the two terms in \eqref{eq_sec3_2}. Together with the $H$ factor in \eqref{eq_sec3_4}, this leads to the $(1-\gamma)^{-5.5}$ factor in Theorem~\ref{thm1}. 

To utilize the full power of our multi-stage update rule, we would like to set $B = \Theta(H^3)$, so that \eqref{eq_sec3_2a} can be upper bounded by $\tilde{O}(SAH^2 \ln (1/p)/\epsilon)$ (plus lower order terms). However, the first term in \eqref{eq_sec3_2} becomes much bigger. In the next subsection, we discuss how to deal with this problem via the variance reduction method, which leads to the asymptotically near-optimal bound in Theorem~\ref{thm2}.

\paragraph{Variance Reduction via Reference-Advantage Decomposition.} As discussed above, when $B$ is set large, we suffer bigger estimation variance, as fewer samples are allowed in the type-\uppercase\expandafter{\romannumeral1} stages. In model-free regret minimization tasks, similar problem arises where the algorithm (e.g., \cite{jin2018q}) can only use the recent tiny fraction of the samples and incurs sub-optimal dependency on the episode length. Recent work~\cite{zhang2020almost} resolves this problem via the \emph{reference-advantage decomposition} technique. 

The high-level idea is that, assuming we have a $\delta$-accurate estimation of $V^*$, namely the \emph{reference value function} $V^{\mathrm{ref}}$, such that $\|V^{\mathrm{ref}} - V^*\|_{\infty} \leq \delta$, we only need to use the samples to estimate the difference $V^{\mathrm{ref}} - V^*$, which is called the \emph{advantage}. Therefore, the estimation error (incurred in places such as \eqref{eq_sec3_1a}) will be much smaller when $\delta$ is small. Choosing $\delta = 1/\sqrt{B}$, and together with the Bernstein-type exploration bonus (see, e.g., \cite{azar2017minimax,jin2018q}), we are able to bound the total contribution of the first term in \eqref{eq_sec3_2} \footnote{More precisely, we refer to the total contribution related to the exploration bonus, which is actually in a different form from the first term in \eqref{eq_sec3_2}. This is because $\check{b}$ has to be re-designed using the Bernstein-type exploration bonus technique and evolves to a more complex expression. Please refer to Appendix~\ref{app:proof-thm-2} for more explanation.}  by $\tilde{O}(SA/(\epsilon(1-\gamma)^2)$, which (together with the $H$ factor in \eqref{eq_sec3_4}) aligns with the $(1-\gamma)^{-3}$ factor in the bound of Theorem~\ref{thm2}. The discussion till now is based on the access of the reference value function $V^{\mathrm{ref}}$. In reality, however, we need to learn the reference value function on the fly. This will incur an additive warm-up cost that polynomially depends on $1/\delta$. However, since $\delta$ is independent of $\epsilon$, the extra cost is only a lower-order term. This technique is only used in the proof of Theorem~\ref{thm2}, which is deferred to Appendix~\ref{app:proof-thm-2} due to space constraints.





\section{The \UCBS Algorithm}\label{sec:UCBS-alg}
In this section, we introduce the \UCBS algorithm. The algorithm takes $\mathcal{S},\mathcal{A},\gamma,\epsilon$, sets $H =\max\{ \frac{\ln({8}/{((1-\gamma)\epsilon)})}{\ln({1}/{\gamma})}  ,\frac{1}{1-\gamma} \}$ and $B = \sqrt{H}$. Throughout the paper, we set $\iota = \ln(2/p)$. The algorithm is described in Algorithm~\ref{alg1}, where a few related notations are explained as follows.


 \paragraph{The precise definition of the stages.} Let $ d_{1}=H$, $d_{j+1}=\lfloor(1+\frac{1}{H}) d_{j}\rfloor$ for all $j\geq 1$.	
  The sizes of the $j$-th type-\uppercase\expandafter{\romannumeral1} and type-\uppercase\expandafter{\romannumeral2} stage are given by 
	 $\check{e}_j = d_{\left\lceil  j/B \right \rceil} $ and  $\bar{e}_j = d_j$ respectively. 
 Let $N_0 = c_1 \cdot \frac{S^3AH^5 \ln({4H^2S}/{\epsilon})\iota}{\epsilon^2}$ for some large enough constant $c_1$. We stop updating $Q(s,a)$ if the number of visits to $(s,a)$ is greater than $N_0$, since the value functions will be sufficiently learned by that time. Therefore, the time steps when an update is triggered by the  type-\uppercase\expandafter{\romannumeral1} and type-\uppercase\expandafter{\romannumeral2} stages are respectively given by 
$\check{\mathcal{L}} =\{\sum_{i=1}^{j} \check{e}_{i}|1\leq j\leq \check{J} \}  \text{~and~}  \bar{\mathcal{L}} = \{ \sum_{i=1}^j \bar{e}_{i}| 1\leq j\leq \bar{J} \}$,
	 where $\check{J} = \max\{j |   \sum_{i=1}^{j-1}\check{e}_{i}\leq N_0  \}$ and   $\bar{J} = \max\{j |   \sum_{i=1}^{j-1}\bar{e}_{i}\leq  N_0  \}$ . Without loss of generality, we assume that $\sum_{i=1}^{\check{J}}\check{e}_i = N_0 $.
	 
\paragraph{The statistics.} We maintain the following statistics during the algorithm: for each $(s, a)$, we use $N(s,a)$, $\check{N}(s,a)$, and $\bar{N}(s,a)$ to respectively denote the total visit number, the visit number in the current type-\uppercase\expandafter{\romannumeral1} stage and 
	 the visit number in the current type-\uppercase\expandafter{\romannumeral2} stage of $(s,a)$.
We also maintain $\check{\mu}(s,a)$ and $\bar{\mu}(s,a)$, which are respectively the accumulators for state values $V(s')$ (where $s'$ is the next state observed after $(s,a)$) during the current type-\uppercase\expandafter{\romannumeral1} and type-\uppercase\expandafter{\romannumeral2} stages. 




\begin{algorithm}[tb]
	\caption{\UCBS}
	\begin{algorithmic}\label{alg1}
		\STATE{\textbf{Initialize:}	$\forall (s,a)\in \mathcal{S}\times \mathcal{A}$: $Q(s,a)\leftarrow \frac{1}{1-\gamma}$, $N(s,a), \check{N}(s,a), \bar{N}(s,a), \check{\mu}(s,a), \bar{\mu}(s,a)\leftarrow 0$;}

		\FOR{$t=1,2,3, \dots$}
		\STATE{Observe $s_{t}$;}
		\STATE{Take action $ a_{t}= \arg\max_{a}Q(s_{t},a)$ and observe $s_{t+1}$;}
		\STATE{\verb|\\| \emph{Maintain the statistics}}
		\STATE{ $(s,a,s')\leftarrow (s_{t},a_{t},s_{t+1})$;}
		\STATE{ $n: = N(s,a)\leftarrow N(s,a)+1$;}
		\STATE{ $\check{n}:=\check{N}(s,a)\leftarrow \check{N}(s,a)+1,\quad \quad \check{\mu}:= \check{\mu}(s,a)\leftarrow \check{\mu}(s,a)+ V(s')$;}
		\STATE{$\bar{n}:=  \bar{N}(s,a)\leftarrow \bar{N}(s,a)+1,\quad \quad  \bar{\mu} :=\bar{\mu}(s,a) \leftarrow \bar{\mu}(s,a)+V(s')$;}

		\STATE{ \verb|\\| \emph{Update triggered by a type-\uppercase\expandafter{\romannumeral1} stage}}
		\IF{$n\in \check{\mathcal{L}} $}
				  \vspace{-3ex}
		\STATE{\begin{align}
			 &   \check{b}\leftarrow   \min\{ 2\sqrt{ H^2\iota /\check{n}   }  ,1/(1-\gamma) \}; & \qquad \qquad \qquad \qquad \qquad \qquad \qquad \nonumber \\ 
			 & \displaystyle{Q(s,a) \leftarrow \min\{  r(s,a)+\gamma \big(\check{\mu}/\check{n} \big)+\check{b} , Q(s,a) \};}  \label{equpdate1}\\
			 &\check{N}(s,a)  \leftarrow0; \quad  \check{\mu}(s,a )\leftarrow 0; \quad V(s) \leftarrow \max_{a}Q(s,a); \nonumber 
			\end{align}}
		\vspace{-3ex}
		
		\ENDIF
		\STATE{ \verb|\\| \emph{Update triggered by a type-\uppercase\expandafter{\romannumeral2} stage}}
		\IF{$n\in \bar{\mathcal{L}}$}
		  \vspace{-3ex}
		\STATE{\begin{align}
			 &  \bar{b} \leftarrow \min \{ 2\sqrt{ H^2\iota /\bar{n}}  ,1/(1-\gamma)  \}; & \qquad \qquad \qquad \qquad \qquad \qquad \qquad \nonumber \\ 
			 & \displaystyle{Q(s,a) \leftarrow \min\{  r(s,a)+\gamma \big(\bar{\mu}/\bar{n} \big)+\bar{b} , Q(s,a) \} ;}  \label{equpdate2}\\
			&\bar{N}(s,a)  \leftarrow0 ;\quad \bar{\mu}(s,a )\leftarrow 0;  \quad V(s) \leftarrow \max_{a}Q(s,a);\nonumber 
		\end{align}}
		\vspace{-3ex}
		\ENDIF
		\ENDFOR
	\end{algorithmic}
\end{algorithm}

\section{Analysis of Sample Complexity}\label{analysis}
In this section, we prove Theorem~\ref{thm1} for \UCBS. We start with a few notations: we use $N_{t}(s,a)$, $\check{N}_{t}(s,a)$,$\bar{N}_{t}(s,a)$, $Q_{t}(s,a)$, $V_{t}(s)$ to denote respectively the values of $N(s,a)$, $\check{N}(s,a)$, $\bar{N}(s,a)$, $Q(s,a)$, $V(s)$ before the $t$-th time step.  Let $\check{n}_t(s,a)$, $\check{\mu}_{t}(s,a)$ and $\check{b}^t(s,a)$ be the values of $\check{n}(s,a)$, $\check{\mu}(s,a)$ and $\check{b}(s,a)$ (respectively) in the latest type-\uppercase\expandafter{\romannumeral1}  update of $Q(s,a)$ before the $t$-th time step. In other words, 
$\check{n}_t(s,a)$ is the length of the type-\uppercase\expandafter{\romannumeral1} stage immediately before the current type-\uppercase\expandafter{\romannumeral1} stage with respect to $(s,a)$; $\check{b}_t(s,a) = \min\{ 2\sqrt{{H^2\iota}/{\check{n}_t(s,a)}} ,{1}/{(1-\gamma)} \}$; and 
\begin{align}
\check{\mu}_{t}(s,a) =\sum_{i=1}^{\check{n}_t(s,a)}  V_{\check{l}_{t,i}(s,a) }(s_{\check{l}_{t,i}(s.a)+1 }),
\end{align}
where $\check{l}_{t,i}(s,a)$ is the time step of the $i$-th visit among the $\check{n}_t(s,a)$ visits mentioned above. When $t$ belongs to the first type-\uppercase\expandafter{\romannumeral1} stage of $(s, a)$, we define $\check{n}_t(s,a) = 0$, $\check{\mu}_t(s,a) = 0$, and $\check{b}_t(s,a) = {1}/{(1-\gamma)}$.

Given $(s,a)$ and a time step $t$ such that $(s_t, a_t) = (s, a)$, we use $j_t(s,a)$ to denote the index of the type-\uppercase\expandafter{\romannumeral1} which (the beginning of) the $t$-th time step belongs to with respect to $(s,a)$. For $1\leq j \leq \check{J}$, we use $\rho(j,s,a)$ to denote the start time of the $j$-th type-\uppercase\expandafter{\romannumeral1} with respect to $(s,a)$. Besides, we define $\rho(\check{J}+1,s,a)$ to be the  time $t$ such that $N_t(s,a) = N_0$.
We also define $\underline{\rho}_{t}(s,a):= \rho(j_t(s,a)-1,s,a) $ if $j_t(s,a)\geq 2$ and $0$ otherwise, and $\overline{\rho}_{t}(s,a): = \rho(j_t(s,a)+1,s,a)$.

The following statement shows that $\{Q_t\}$ is a sequence of non-increasing optimistic estimates of $Q^*$.
\begin{proposition}\label{pro1}
	Conditioned on the event $E_1$ specified in \eqref{eq_def_e1} (which is explicitly described in Appendix~\ref{sec_pfpro1}, and happens with probability at least $(1-SAH(\check{J}+\bar{J})p)$), it holds that $Q_{t}(s,a)\geq Q^*(s,a)$ and $Q_{t+1}(s,a)\leq Q_{t}(s,a)$ for all $t\geq 1$ and $(s,a)$.
\end{proposition}

The proofs of Proposition~\ref{pro1} and all the lemmas in the remaining part of this section can be found in Appendix~\ref{app:proof-ucbs}. Throughout the rest of this section, the analysis will be done assuming the successful event $E_1$.

\subsection{Using Clipped Pseudo-Regret to Bound Sample Complexity}\label{sec5.1}

By the update rule \eqref{equpdate1},  for any $t\geq 1$ and $s$, letting $a = \pi_{t}(s)$, we have that
\begin{align}
V_{t}(s)-V^{\pi_{t}}(s) &\leq \check{b}_{t}(s,a) +\frac{\gamma}{\check{n}_{t}(s,a)} \sum_{u=1}^{\check{n}_{t}(s,a)   } V_{ \check{l}_{t,u}(s,a) }(s_{ \check{l}_{t,u}(s,a) +1   }) -\gamma P_{s,a}V^{\pi_{t}} \nonumber
\\ & \leq 2\check{b}_{t}(s,a)  +\gamma P_{s,a}\left( \frac{1}{\check{n}_{t} (s,a )}   \sum_{u=1}^{ \check{n}_{t} (s,a )     } V_{ \check{l}_{t,u}(s,a) } -V^{\pi_{t}}  \right)   \label{eq_pat_ex0}
\\ & \leq 2\check{b}_{t}(s,a)  +\gamma P_{s,a}(V_{\underline{\rho}_{t}(s,a) }-V^{\pi_{t}}) \label{eq_pat_ex0.5} 
\\ & =2\check{b}_{t}(s,a) +\gamma P_{s,a}(V_{\underline{\rho}_{t}(s,a) }-V_{t})+\gamma P_{s,a}(V_{t }-V^{\pi_{t}}).\label{eq_pat_ex0.51} 
\end{align} 
where Inequality \eqref{eq_pat_ex0} is due to the concentration inequality, which is part of the successful event $E_1$ defined in \eqref{eq_def_e1}, and Inequality \eqref{eq_pat_ex0.5} holds because $\underline{\rho}_{t}(s_t,a_t)\leq \check{l}^{t}_{u}$ for any $1\leq u\leq \check{n}^{t}$ and the fact $V_{t}$ is non-increasing in $t$ (Proposition~\ref{pro1}).

On the other hand, we also have 
\begin{align}
    V_t(s) - V^{\pi_t}(s) &  = Q_t(s,a) - Q^*(s,a) + Q^{*}(s,a) - Q^{\pi_t}(s,a) \nonumber
    \\ & =  Q_t(s,a) - Q^*(s,a) + \gamma P_{s,a}(V^* - V^{\pi_t})\nonumber
    \\ & \leq Q_t(s,a)-Q^*(s,a) + \gamma P_{s,a}(V_t - V^{\pi_t}).\label{eq_pat_ex0.6}
\end{align}
Combining \eqref{eq_pat_ex0.51} and \eqref{eq_pat_ex0.6}, we have that
\begin{align}
     V_t(s) - V^{\pi_t}(s) \leq \min \left\{ 2\check{b}_{t}(s,a) +\gamma P_{s,a}(V_{\underline{\rho}_{t}(s,a) }-V_{t}),Q_t(s,a)-Q^*(s,a)  \right\} +\gamma P_{s,a}(V_t - V^{\pi_t}) .\label{eq_pat_ex0.61}
\end{align}
Iterating \eqref{eq_pat_ex0.61} for $H$ times, we obtain that
\begin{align}
 &V^*(s_{t})-V^{\pi_{t}}(s_{t})  \nonumber
 \\& \leq \sum_{s,a}w_{t}(s,a) \left( \min\left\{2\check{b}_{t}(s,a) + \gamma P_{s,a}(V_{\underline{\rho}_{t}(s,a) } -V_{t}) , Q_t(s,a)-Q^*(s,a) \right\} \right)   +\frac{\epsilon}{8} \label{eq_pat_ex3.5}
\\& \leq  \sum_{s,a}w_{t}(s,a) \left(\min\left\{ 2\mathrm{clip}(\check{b}_{t}(s,a), \frac{\epsilon}{8H})  + \gamma P_{s,a} \mathrm{clip}( V_{\underline{\rho}_t(s,a)} -V_{t},\frac{\epsilon}{8H} ) , \right.\right. \nonumber\\
&\qquad\qquad\qquad\qquad\qquad\qquad\qquad\qquad\qquad \left.\left. \mathrm{clip}(Q_t(s,a)-Q^*(s,a) , \frac{3\epsilon}{4H}) \right\}  \right)   +\frac{7\epsilon }{8}, \label{eq_pat1}
\end{align}
where $ w_{t}(s,a)  = \mathbb{I}[\pi_t(s) = a] \cdot  \sum_{i=0}^{H-1} \mathbf{1}_{s_t}^{\top} (\gamma P_{\pi_t})^i \mathbf{1}_{s}$ is the expected discounted visit number of $(s,a)$ in the next $H$ steps following $\pi_{t}$ (recall that $P_{\pi_{t}}$ is the matrix such that $P_{\pi_{t}}(s) = P_{s,\pi_{t}(s)}$ for any $s\in \mathcal{S}$); and Inequality \eqref{eq_pat1} is due to an averaging argument and the fact that $\sum_{s,a}w_{t}(s,a)\leq H$. Let
\begin{align}
  \beta_{t} & :=   \sum_{s,a}w_{t}(s,a) \min\left\{ \left(  2\mathrm{clip}(\check{b}_{t}(s,a), \frac{\epsilon}{8H}) + \gamma P_{s,a} \mathrm{clip}( V_{\underline{\rho}_t(s,a)} -V_{t},\frac{\epsilon}{8H} )  \right) , \right.\nonumber\\
  & \qquad\qquad\qquad\qquad\qquad\qquad\qquad\qquad\qquad\qquad\qquad\left.\mathrm{clip}(Q_t(s,a)-Q^*(s,a) , \frac{3\epsilon}{4H})\right\}. \label{eq:def-beta}  
\end{align}
Define $\mathcal{T} = \{t\geq 1| \beta_t > \frac{1}{8}\epsilon \}$. By \eqref{eq_pat1} we have that the sample complexity of \UCBS is bounded by
\begin{align}
    \sum_{t\geq 1}\mathbb{I}\left[ V^*(s_{t})-V^{\pi_{t}}(s_{t})>\epsilon\right]\leq\sum_{t\geq 1}\mathbb{I}\left[\beta_{t}>\frac{1}{8}\epsilon \right] = |\mathcal{T}|.\nonumber
\end{align}

 
To bound $|\mathcal{T}|$, we consider bounding $\sum_{t\in \mathcal{T}} \beta_t$ instead, since $\sum_{t\in \mathcal{T}}\beta_{t}\geq \frac{|\mathcal{T}|\epsilon}{8}$ and therefore $|\mathcal{T}| \leq (8/\epsilon) \cdot \sum_{t\in \mathcal{T}} \beta_t$. Let 
 \begin{align}
     \tilde{\beta}_t :=\min\left\{ 2\mathrm{clip}(\check{b}_{t}(s_t,a_t),\frac{\epsilon}{8H}) + \gamma P_{s_t,a_t} \mathrm{clip} ( V_{\underline{\rho}_{t}(s_t,a_t)      } -V_t ,\frac{\epsilon}{8H}), \mathrm{clip}(Q_t(s_t,a_t)-Q^*(s_t,a_t) , \frac{3\epsilon}{4H}) \right\}, \label{eq:def-tbeta}
 \end{align}
and if $\pi_t$ does not change very frequently, we have the approximation that $\beta_t \approx \sum_{i=0}^{H-1} \tilde{\beta}_{t+i}$. More formally, we prove the following statement.
\begin{lemma}\label{lemma_bd_beta} For any $K\geq 1$,	it holds that
\begin{align}
&\mathbb{P}\Big[  \sum_{t\in  \mathcal{T}  } \beta_t \geq 12KH^3\iota+24SAH^4B\ln(N_0)  \text{~and~} \sum_{t\geq 1} \tilde{\beta}_t < 3KH^2\iota\Big]\leq Hp   .\nonumber
\end{align}
\end{lemma}

 By Lemma \ref{lemma_bd_beta} and the discussion above, if we are able to bound $\sum_{t\geq 1} \tilde{\beta}_t \leq X$ (for $X \geq 3 H^2\iota$), then with high probability, the sample complexity of \UCBS is bounded by roughly $O(H/\epsilon) \cdot X$.

\subsection{Bounding the Clipped Pseudo-Regret}


We now turn to bound $\sum_{t\geq 1} \tilde{\beta}_t$. By \eqref{eq:def-tbeta}, we have that
\begin{align}
  \tilde{\beta}_t  & \leq \mathbb{I}[N_t(s_t,a_t)< N_0]\cdot  \left( 2\mathrm{clip}(\check{b}_{t}(s_t,a_t),\frac{\epsilon}{8H}) + \gamma P_{s_t,a_t} \mathrm{clip} ( V_{\underline{\rho}_{t}(s_t,a_t)      } -V_t ,\frac{\epsilon}{8H}) \right) \nonumber
  \\ & \quad + \mathbb{I}[N_t(s_t,a_t)\geq N_0]\cdot \mathrm{clip}(Q_t(s_t,a_t)-Q^*(s_t,a_t) , \frac{3\epsilon}{4H})\nonumber
  \\ &  =  \mathbb{I}[N_t(s_t,a_t)< N_0] \cdot 2\mathrm{clip}(\check{b}_{t}(s_t,a_t),\frac{\epsilon}{8H}) +  \mathbb{I}[N_t(s_t,a_t)< N_0] \cdot \gamma P_{s_t,a_t} \mathrm{clip} ( V_{\underline{\rho}_{t}(s_t,a_t)      } -V_t ,\frac{\epsilon}{8H})\nonumber
  \\ & \quad +\mathbb{I}[N_t(s_t,a_t)\geq N_0]\cdot \mathrm{clip}(Q_t(s_t,a_t)-Q^*(s_t,a_t) , \frac{3\epsilon}{4H})  .\label{eq:decom_tbeta}
\end{align}
For the first term in  \eqref{eq:decom_tbeta}, we have the following lemma.
\begin{lemma}\label{lemma_bound_b}
\begin{align}
\sum_{t\geq 1}\mathrm{clip}(\check{b}_t(s_t,a_t),\frac{\epsilon}{8H}) \leq  O\left(\frac{ SAB \iota }{\epsilon(1-\gamma)^4}\right)	.\nonumber
\end{align}
\end{lemma}

For the second term in  \eqref{eq:decom_tbeta}, let $\alpha_t= \mathbb{I}[N_t(s_t,a_t)\leq N_0]P_{s_t,a_t} \mathrm{clip} ( V_{\underline{\rho}_{t}(s_t,a_t)      } -V_t ,\frac{\epsilon}{8H})$ for short. By a baseline result for learning the value function (see Lemma~\ref{lemma_lvf}), we have that
\begin{lemma}\label{lemma_bound_alpha}
		 With probability $1-(1+ 2SAH(\check{J}+\bar{J}))p$, it holds that
	\begin{align}
	\sum_{t\geq 1}\alpha_{t}\leq O\left(\frac{SAH^5\ln(\frac{4H}{\epsilon})\iota}{\epsilon B} +SABH^3+SAH\ln(N_0)\right)      . \nonumber
	\end{align}
\end{lemma}

For the last term in \eqref{eq:decom_tbeta}, we have that
\begin{lemma}\label{eq:lemma_add}
With probability $1-(1+ 2SAH(\check{J}+\bar{J}))p$, for any $t\geq 1$ it holds that
\begin{align}
    \mathbb{I}[N_t(s_t,a_t)\geq N_0]\cdot \mathrm{clip}(Q_t(s_t,a_t)-Q^*(s_t,a_t) , \frac{3\epsilon}{4H})  =0.\nonumber
\end{align}
\end{lemma}

Combining Lemma~\ref{lemma_bound_b},  Lemma~\ref{lemma_bound_alpha} and Lemma~\ref{eq:lemma_add}, and by the definition of $\tilde{\beta}_t$, we have that
\begin{lemma} \label{lemma_srbd}
			 With probability $1- (2+ 6SAH(\check{J}+\bar{J})) p$, it holds that
			   \begin{align} \sum_{t\geq 1}\tilde{\beta}_{t} \leq  O\left(\frac{ SABH^4\iota }{\epsilon}+\frac{SAH^5\ln(\frac{4H}{\epsilon})\iota}{\epsilon B} +SABH^3\ln(N_0) \right) .\nonumber \end{align}
	\end{lemma}

\subsection{Putting Everything Together}\label{subsection:pet}

Invoking Lemma \ref{lemma_bd_beta} with 
$K =  \frac{c_{2}}{3H^2\iota}\left( \frac{ SABH^4\iota }{\epsilon}+\frac{SAH^5\ln(\frac{4H}{\epsilon})\iota}{\epsilon B} +SABH^3\ln(N_0) \right)\geq 1$
for some large enough universal constant $c_{2}$,  we have that conditioned on the successful event $E_1$, 
\begin{align}
&\mathbb{P}\left[  \sum_{t\in \mathcal{T}}\beta_{t}\geq  12KH^3\iota+24SAH^4B  \ln(N_0)    \right] \nonumber
\\ & \leq\mathbb{P}\left[  \sum_{t\in  \mathcal{T} } \beta_t \geq 12KH^3\iota+24SAH^4B \ln(N_0)  ,  \sum_{t\geq 1}\tilde{\beta}_{t} < 3KH^2\iota\right]   + \mathbb{P}\left[          \sum_{t\geq 1} \tilde{\beta}_{t} \geq    3KH^2\iota       \right]       \label{eq_pat_ex9}
\\ & \leq (4SAH(\check{J} + \bar{J}) + H+2)p,\label{eq_pat_ex10}
\end{align}
where the second term in \eqref{eq_pat_ex9} bounded due to Lemma~\ref{lemma_srbd}.
Combining Proposition \ref{pro1} with \eqref{eq_pat_ex10},
we obtain that    with probability $1-(8SA(\check{J}+\bar{J})+ (H+3) )p$, it holds that
 \begin{align}
\frac{|\mathcal{T}|\epsilon}{2}\leq \sum_{t\in \mathcal{T}}\beta_{t}\leq  O\left(         \frac{ SABH^5\iota }{\epsilon}+\frac{SAH^6\ln(\frac{4H}{\epsilon})\iota}{\epsilon B} + SAH^4B\ln(N_0)      \right) .\label{eq_pat_af}
 \end{align}

Noting that $B  =\sqrt{H}$, we conclude that the number of $\epsilon$-suboptimal steps is bounded by 
\begin{align}
O\left(            \frac{SAH^{5.5}\ln(\frac{4H}{\epsilon})\iota }{\epsilon^2}    +\frac{SAH^{4.5}\ln(N_0)}{\epsilon}    \right)  \leq  O\left(            \frac{SAH^{5.5}\ln(\frac{4H}{\epsilon})(\ln(N_0)+\iota) }{\epsilon^2}      \right) \nonumber
\end{align}
for any $\epsilon \in (0, \frac{1}{1-\gamma}]$.
Noting that $H = \tilde{O}(\frac{1}{1-\gamma})$, $\check{J}=O(SAH\ln(N_0) )$ and $\bar{J} = O(SAHB\ln(N_0))$, we finish the proof of Theorem~\ref{thm1} by replacing $p$ with $\frac{p}{8SA(\check{J}+\bar{J})+ H+3 }$.


\newpage
\bibliographystyle{}
\bibliography{ref}

\newpage 
\appendix
\appendixpage
\renewcommand{\appendixname}{Appendix~\Alph{section}}
\setlength{\parindent}{0pt}
\setlength{\parskip}{0.2\baselineskip}

\section{Comparison with Previous Works} \label{app:comp}
\begin{table}[!h]\label{table1}
	\centering
	\caption{Comparisons of \textsc{PAC-RL} algorithms for discounted MDPs}
	\begin{tabular}{|c|c|c|c|}
		\hline
  & \textbf{Algorithm}	 & \textbf{Sample complexity}	 & \textbf{Space complexity}\\
  \hline
		\multirow{3}{*}{\makecell{\\ \\Model-based}} & R-max~\cite{brafman2003r-max,kakade2003sample} & $\displaystyle{\tilde{O}\left(\frac{S^2A\ln(1/p)}{\epsilon^3(1-\gamma)^6}\right)}$ & \multirow{3}{*}{\makecell{\\ \\$O(S^2A)$}} \\
		\cline{2-3}
		& MoRmax~\cite{szita2010model-based} & $\displaystyle{\tilde{O}\left(\frac{SA\ln(1/p)}{\epsilon^2(1-\gamma)^6}\right)}$ &  \\
		\cline{2-3}
		& UCRL-$\gamma$ \cite{lattimore2012pac} & $\displaystyle{\tilde{O}\left(\frac{S^2A\ln(1/p)}{\epsilon^2 (1-\gamma)^3}\right)}$ & \\
		\hline 
		\multirow{4}{*}{\makecell{\\ \\ \\ \\ Model-free}}  & Delayed $Q$-learning \cite{strehl2006pac}  & $\displaystyle{\tilde{O}\left(\frac{SA\ln(1/p)}{\epsilon^4(1-\gamma)^8}\right)}$ & \multirow{4}{*}{\makecell{\\ \\ \\ \\ $O(SA)$}}\\
		\cline{2-3}
		& \makecell{Infinite $Q$-learning \\ with UCB \cite{dong2019q}} & $\displaystyle{\tilde{O}\left(\frac{SA\ln(1/p)}{\epsilon^2(1-\gamma)^7}\right)}$ & \\
		\cline{2-3}
		&  \makecell{\UCBSA \\ (Theorem \ref{thm2})}  & \makecell{$\displaystyle{\tilde{O}\left(\frac{SA\ln(1/p)}{\epsilon^2(1-\gamma)^3}\right)}$\\ (for $\epsilon < \frac{1}{\mathrm{poly}(S,A,1/(1-\gamma))}$)} & \\
		\cline{2-3}
		& \UCBS  (Theorem \ref{thm1})  & $\displaystyle{\tilde{O}\left(\frac{SA\ln(1/p)}{\epsilon^2(1-\gamma)^{5.5}}\right)}$  & \\
		\hline
		& \textsc{Median-PAC}\cite{pazis2016improving}  & $\displaystyle{\tilde{O}\left(\frac{SA\ln(1/p)}{\epsilon^2(1-\gamma)^4}\right)}$  & $\tilde{O}\left(\frac{SAH^4}{\epsilon^2}\right)$ \\
		\hline
		Lower bound &  & $\displaystyle{\Omega\left(\frac{SA}{\epsilon^2(1-\gamma)^3}\right)}$ \cite{lattimore2012pac} & \\
		\hline
	\end{tabular}	
\end{table}

\section{Technical Lemmas}

\begin{lemma}\label{lemma_berbound}
	Let $M_1,M_2,...,M_k,...$ be a series of random variables which range in $[0,1]$ and $\{\mathcal{F}_k\}_{k\geq 0} $  be a filtration such that $M_k$ is measurable with respect to $\mathcal{F}_{k}$ for $k\geq 1$. Define $\mu_k : = \mathbb{E}\left[ M_k |\mathcal{F}_{k-1}  \right]$.
	
	For any $p \in (0,1)$ and $c\geq 1$,
	it holds that
	\begin{align}
	\mathbb{P}\left[ \exists n,   \sum_{k=1}^{n} \mu_{k}\geq 4c\iota, \sum_{k=1}^{n} M_k \leq c\iota\right]\leq p .              \nonumber
	\end{align}
\end{lemma}

\begin{proof}
	Let $\lambda<0$ be fixed. Let $M$ be a random variable taking values in $[0,1]$ with mean $\mu$. By convexity of $e^{\lambda x}$ in $x$, we have that $\mathbb{E}\left[ e^{\lambda M}\right]\leq \mu e^{\lambda}+(1-\mu) = 1+\mu(e^{\lambda}-1)\leq e^{  \mu(e^\lambda -1) }$. Then we obtain that for any $k\geq 1$
	\begin{align}
	&\mathbb{E}\left[ e^{\lambda M_k - (e^{\lambda}-1)\mu_k }| \mathcal{F}_{k-1}  \right]    \leq  1,\nonumber 
	\end{align}

	
	which means $\{Y_k  :=  e^{\lambda\sum_{i=1}^k M_i - (e^{\lambda }-1)\sum_{i=1}^k \mu_i  }  \}_{k\geq 0}$ is a super-martingale with respect to $\{ \mathcal{F}_k\}_{k\geq 0}$. Let $\tau$ be the least $n$ with $\sum_{k=1}^n \mu_k \geq 4c\iota$.  It is easy to verify that $|Y_{ \min\{\tau , n\} }|\leq e^{(1-e^{\lambda})(4c\iota+1)}$ for any $n$.
	By the optional stopping theorem, 
	we have that
	$\mathbb{E}\left[Y_{\tau} \right]\leq 1$.  Then
	\begin{align}
	&\mathbb{P}\left[ \exists n,   \sum_{k=1}^{n} \mu_{k}\geq 4c\iota, \sum_{k=1}^{n} M_k \leq c\iota   \right] \nonumber
	\\ &   \leq   \mathbb{P}\left[ \sum_{k=1}^{\tau} M_k \leq c\iota\right] \nonumber
	\\ & \leq \frac{1}{e^{   (1-e^{\lambda})4c\iota+ \lambda c\iota   }} .\label{eqlemma_lambda}
	\end{align}
	By setting $\lambda = -\frac{1}{2}$, we obtain that $\frac{1}{e^{   (1-e^{\lambda})4c\iota+ \lambda c\iota   }} \leq \frac{1}{e^{c\iota}} = (\frac{p}{2})^c \leq p$. The proof is completed.
\end{proof}

\begin{lemma}[Freedman's Inequality, Theorem 1.6 of \cite{freedman1975tail}]\label{freedman}
	Let $(M_{n})_{n\geq 0}$ be a  martingale such that $M_{0}=0$ and $|M_{n}-M_{n-1}|\leq c$. Let $\mathrm{Var}_{n}=\sum_{k=1}^{n}\mathbb{E}[(M_{k}-M_{k-1})^{2}|\mathcal{F}_{k-1}]$ for $n\geq 0$,
	where $\mathcal{F}_{k}=\sigma(M_0,M_{1},M_{2},\dots,M_{k})$. Then, for any positive $x$ and for any positive $y$,
	\begin{equation}\label{Bernstein2}
	\mathbb{P}\left[ \exists n:  M_{n}\geq x ~\text{and}~\mathrm{Var}_{n}\leq y \right]  \leq \exp\left(-\frac{x^{2}}{2(y+cx)} \right).
	\end{equation}
\end{lemma}

\begin{lemma}\label{self-norm}
	Let $(M_{n})_{n\geq 0}$ be a  martingale  such that $M_0 = 0$  and $|M_{n}-M_{n-1}|\leq c$ for some $c>0$ and any $n\geq 1$. Let $\mathrm{Var}_{n}=\sum_{k=1}^{n}\mathbb{E}[(M_{k}-M_{k-1})^{2}|\mathcal{F}_{k-1}]$ for $n\geq 0$,
	where $\mathcal{F}_{k}=\sigma(M_{1},M_{2},...,M_{k})$. Then for any positive integer $n$, and any $\epsilon,p>0$, we have that
	\begin{equation}\label{self-bernstein}
	\mathbb{P}\left[|M_{n}|\geq   2\sqrt{2}\sqrt{\mathrm{Var}_{n}\log(\frac{1}{p})}+2\sqrt{\epsilon\log(\frac{1}{p} )} +2c\log(\frac{1}{p}) \right] \leq 2\left( \log_{2}(\frac{nc^2}{\epsilon})   +1\right)p.
	\end{equation}

\end{lemma}

\begin{proof}
	For any fixed $n$, we apply Lemma \ref{freedman} with  $y = 2^{i}\epsilon$ and $x= \pm(2\sqrt{y \log(\frac{1}{p})}+2c\log(\frac{1}{p}))$. For each $i=  0,1,2,\dots, \log_{2}(\frac{nc^2}{\epsilon})$, we get that
	\begin{align}
	&  \mathbb{P}\left[|M_{n}|\geq 2\sqrt{2}\sqrt{2^{i-1}\epsilon\log(\frac{1}{p})}+ 2c\log(\frac{1}{p}),\mathrm{Var}_{n}\leq 2^i\epsilon  \right] \nonumber
	\\& = \mathbb{P}\left[ |M_{n}|\geq 2\sqrt{2^i\epsilon\log(\frac{1}{p})}+2c\log(\frac{1}{p}),\mathrm{Var}_{n}\leq 2^i\epsilon  \right] \nonumber
	\\& \leq 2p.
	\end{align}
	Then via a union bound, we have that
	\begin{align}
	&\mathbb{P}\left[|M_{n}|\geq  2\sqrt{2}\sqrt{\mathrm{Var}_{n}\log(\frac{1}{p})} +2\sqrt{\epsilon\log(\frac{1}{p})}+2c\log(\frac{1}{p}) \right] \nonumber
	\\& \leq \sum_{i=1}^{\log_{2}(\frac{nc^2}{\epsilon}) }\mathbb{P}\left[ |M_{n}|\geq 2\sqrt{2}\sqrt{2^{i-1}\epsilon\log(\frac{1}{p})}+2c\log(\frac{1}{p}) ,2^{i-1}\epsilon\leq \mathrm{Var}_{n}\leq 2^{i}\epsilon\right] \nonumber
	\\ & \quad + \mathbb{P}\left[   |M_{n}|\geq 2\sqrt{\epsilon\log(\frac{1}{p})} +2c\log(\frac{1}{p}) ,   \mathrm{Var}_{n}\leq \epsilon  \right]
	\\& \leq  \sum_{i=1}^{\log_{2}(\frac{nc^2}{\epsilon}) }\mathbb{P}\left[|M_{n}|\geq 2\sqrt{(i-1)\epsilon\log(\frac{1}{p})}+ 2\sqrt{\epsilon\log(\frac{1}{p})}+2c\log(\frac{1}{p}),\mathrm{Var}_{n}\leq i\epsilon  \right] +2p \nonumber
	\\ & \leq 2\left( \log_{2}(\frac{nc^2}{\epsilon})   +1\right)p .
	\end{align} 
\end{proof}

\section{Missing Proofs in Section~\ref{analysis}}\label{app:proof-ucbs}
\subsection{Proof of Proposition \ref{pro1}}\label{sec_pfpro1}

\begin{proof}[Proof of Proposition \ref{pro1}] Let $(s,a)$ and $j$ be fixed. 
 Let $\check{l}_{i}$  be the time when the $i$-th visit in the $j$-th type-\uppercase\expandafter{\romannumeral1} stage of $(s,a)$  occurs. Define $\check{b}^{(j)} = \min\{2\sqrt{\frac{H^2\iota}{\check{e}_{j}}},\frac{1}{1-\gamma} \}$ for $j\geq 2$. By Azuma's inequality, we obtain that for any $1\leq j\leq \check{J}$ and $(s,a)$,   with probability $1- 2p$, it holds that
\begin{align}
&\frac{1}{\check{e}_j} \sum_{i= 1}^{\check{e}_j} V^*(s_{\check{l}_{i}(s,a)+1  }) + \check{b}^{(j)}\geq P_{s,a}V^*;\label{eq_pro4_1}
\\ &\left| \frac{1}{\check{e}_j} \sum_{i= 1}^{\check{e}_j} \left(V_{\check{l}_{i}(s,a)}(s_{\check{l}_{i}(s,a)+1  })- P_{s,a}V_{\check{l}_i(s,a)}\right) \right| \leq  \check{b}^{(j)}  .\label{eq_pro4_1.1}
\end{align}


Similarly, letting $\bar{l}_{i}(s,a)$ be the time when the $i$-th visit in the $j$-th type-\uppercase\expandafter{\romannumeral2} stage of $(s,a)$ occurs, and defining $\bar{b}^{(j)} = \min\{2\sqrt{\frac{H^2\iota}{\bar{e}_{j}}},\frac{1}{1-\gamma} \}$ for $j\geq 1$, we have that  for any $1\leq j'\leq \bar{J}$ and $(s,a)$, with probability $1-2p$, it holds that 

\begin{align}
&\frac{1}{\bar{e}_{j'}} \sum_{i= 1}^{\bar{e}_{j'}} V^*(s_{\bar{l}^{(j')}_{i}(s,a)+1  }) +\bar{b}^{(j')}  \geq P_{s,a}V^*; \label{eq_pro4_2}
\\ & \left| \frac{1}{\bar{e}_{j'}} \sum_{i= 1}^{\bar{e}_{j'}} \left(V_{\bar{l}^{(j')}(s,a) }(s_{\bar{l}^{(j')}_{i}(s,a)+1  }) -P_{s,a}V_{\bar{l}^{(j')}(s,a) } \right) \right| \leq \bar{b}^{(j')}.\label{eq_pro4_2.1}
\end{align}

Define $\check{E}^{(j)}(s,a)$ be the event \eqref{eq_pro4_1} and \eqref{eq_pro4_1.1} hold for $(s,a,j)$, and $\bar{E}^{(j')}(s,a)$ be the event \eqref{eq_pro4_2} and \eqref{eq_pro4_2.1} hold for $(s,a,j')$. Let
\begin{align}E_{1}=(\cap_{s,a,1\leq j\leq \check{J}}\check{E}^{(j)}(s,a)) \cap (\cap_{s,a,1\leq j'\leq \bar{J}} \bar{E}^{(j')}(s,a))  .\label{eq_def_e1}
\end{align}
Then $\mathbb{P}\left[ E_{1} \right] \geq 1-S2A( \check{J}+\bar{J})p.$
We will  prove by induction conditioned on this event. 

For $t = 1$, $Q_{t}(s,a) = \frac{1}{1-\gamma} \geq Q^*(s,a)$ for any $(s,a)$. For $t\geq 2$, assume $Q_{t'}(s,a)\geq Q^*(s,a)$ for $1\leq t'< t$ and all $(s,a)$ pairs. If there exists $(j,s,a)$ such that the $j$-th  type-\uppercase\expandafter{\romannumeral1} update of $(s,a)$ happens at the $(t-1)$-th step, by \eqref{eq_pro4_1} we have that
\begin{align}
Q_{t}(s,a) &=\min\{ r(s,a)  + \frac{\gamma }{ \check{e}_j }\sum_{i= 1}^{\check{e}_j} V_{ \check{l}^{(j)}_{i}(s,a) }(   s_{\check{l}^{(j)}_{i}(s,a)+1  })+ \check{b}^{(j)} ,Q_{t-1}(s,a) \} \nonumber
\\  &  \geq \min \{ r(s,a) +    \frac{\gamma }{ \check{e}_j } \sum_{i= 1}^{\check{e}_j} V^*(s_{\check{l}^{(j)}_{i}(s,a)+1  }) + \check{b}^{(j)} ,Q_{t-1}(s,a)  \}\nonumber
\\ & \geq \min \{ r(s,a)+\gamma P_{s,a}V^*,Q_{t-1}(s,a) \} \nonumber
\\ & \geq  Q^*(s,a). \nonumber 
\end{align}
In a similar way, if there exists $(j,s,a)$ such that the $j$-th  type-\uppercase\expandafter{\romannumeral1} update of $(s,a)$ happens at the $(t-1)$-th step, by \eqref{eq_pro4_2}, it holds that $Q_{t}(s,a\geq Q*(s,a))$.
Otherwise, $Q_{t}(s,a) = Q_{t-1}(s,a)\geq Q^*(s,a)$ for any $(s,a)$.  The proof is completed.

\end{proof}




\subsection{Proof of Lemma \ref{lemma_bd_beta}}
We split $\mathcal{T}$ into $H$ separate subsets by define $\mathcal{V}_{k} = \{t\in \mathcal{T}: t \mod H = k\}$ for $k=0,1,2,\dots,H-1$. We will prove Lemma \ref{lemma_bd_beta} by showing that for each $k$, 
it holds that 
\begin{align}
\mathbb{P}\Big[  \sum_{t\in  \mathcal{V}_{k}  } \beta_t \geq 12KH^2\iota+24SAH^3B\ln(N_0)  ,\quad \sum_{t\geq 1} \tilde{\beta}_t < 3KH^2\iota\Big]\leq p   .   \label{eq_lemma_bd_beta1} 
\end{align}
If \eqref{eq_lemma_bd_beta1} holds for each $k$, then we have 
\begin{align}
   & \mathbb{P}\Big[  \sum_{t\in  \mathcal{T}  } \beta_t \geq 12KH^3\iota+24SAH^4B\ln(N_0)  ,\quad \sum_{t\geq 1} \tilde{\beta}_t < 3KH^2\iota\Big]   \nonumber
   \\ & \leq \sum_{k=0}^{H-1} \mathbb{P}\Big[  \sum_{t\in  \mathcal{V}_{k}  } \beta_t \geq 12KH^2\iota+24SAH^3B\ln(N_0)  ,\quad \sum_{t\geq 1} \tilde{\beta}_t < 3KH^2\iota\Big] \nonumber
   \\ & \leq Hp.
\end{align}

Let 
\[
U_t = \mathbb{I }\left[  \exists t'\in \{t,t+1,...,t+H-1  \} \text{ and } (s,a) \text{ such that } Q_{t'+1}(s,a)\neq Q_{t'}(s,a)      \right].
\]

We define  
\[ \hat{\beta}_t := 3H^2U_t +  (1-U_{t})\sum_{i=0}^{H-1} \gamma^i \left(2\mathrm{clip} ( \check{b}_{t}(s_{t+i},a_{t+i}),\frac{\epsilon}{8H})+ \gamma P_{s_{t+i},a_{t+i}  } \mathrm{clip}  (V_{ \underline{\rho}_{t}(s_{t+i},a_{t+i})            } -V_t, \frac{\epsilon}{8H}     )   \right ). \]
For fixed $k \in \{0, 1, 2, \dots, H-1\}$, we let
\[
\hat{\beta}^{k}_{t}: = \frac{\hat{\beta}_{tH+k} \mathbb{I} \left[  tH+k\in \mathcal{T}\right]  }{3H^2}.
\]
Noting  that $\hat{\beta}^{k}_{t}\in [0,1]$  is measurable with respect to $\mathcal{F}^{k}_{t}:=\mathcal{F}_{(t+1)H+k-1}$
and $\mathbb{E}\left[   \hat{\beta}^{k}_{t}  | \mathcal{F}^{k}_{t-1} \right] \geq \beta_{t}^{k}:= \frac{ \beta_{tH+k} \mathbb{I}[ tH+k\in \mathcal{T} ] }{3H^2}$, by Lemma~\ref{lemma_berbound} we obtain that for any $K\geq 1$,
\begin{align}
    \mathbb{P}\left[ \exists n,\sum_{t=1}^{n}\beta_{t}^{k} \geq 4K\iota + 16SAHB\ln(N_0), \quad \sum_{t=1}^{n}\hat{\beta}^{k}_{t} \leq K\iota +4SAHB\ln(N_0) \right]\leq p,\nonumber
\end{align}
which is equivalent to
\begin{align}
&\mathbb{P}\Big[ \exists n, \sum_{t=1}^n \beta_{t}\mathbb{I}\left[t\in \mathcal{V}_{k} \right] \geq 12KH^2\iota  +24SAH^3B\ln(N_0) , \nonumber
\\& \qquad \qquad \qquad \qquad \qquad \qquad \sum_{t=1}^n \hat{\beta}_{t} \mathbb{I} \left[  t\in \mathcal{V}_{k} \right]  \leq 3KH^2\iota +6SAH^3B\ln(N_0)  \Big]\leq p.  \label{eq_pat3}
\end{align}


By definition of $\hat{\beta}_t $, and noting that if $U_{t}=0$, $\check{b}_{t}(s_{t+i},a_{t+i} ) =\check{b}_{t+i}(s_{t+i},a_{t+i}) $ and 
$V_{ \underline{\rho}_{t}(s_{t+i},a_{t+i}) }     = V_{ \underline{\rho}_{t+i}(s_{t+i},a_{t+i})    } $ for any $0\leq i \leq H-1$, we have 
\begin{align}
\hat{\beta}_t &= 3H^2U_t +  (1-U_{t})\sum_{i=0}^{H-1} \gamma^i \left(2\mathrm{clip} ( \check{b}_{t}(s_{t+i},a_{t+i}),\frac{\epsilon}{8H})+ \gamma P_{s_{t+i},a_{t+i}  } \mathrm{clip}  (V_{ \underline{\rho}_{t}(s_{t+i},a_{t+i})            } -V_t, \frac{\epsilon}{8H}     )   \right ) \nonumber
\\ & \leq 3H^2U_{t}+     (1-U_{t})\sum_{i=0}^{H-1} \left(2\mathrm{clip} ( \check{b}_{t}(s_{t+i},a_{t+i}),\frac{\epsilon}{8H})+ \gamma P_{s_{t+i},a_{t+i}  } \mathrm{clip}  (V_{ \underline{\rho}_{t}(s_{t+i},a_{t+i})            } -V_t, \frac{\epsilon}{8H}     )   \right )     \nonumber
\\ & \leq 3H^2U_{t}+     \sum_{i=0}^{H-1} \left(2\mathrm{clip} ( \check{b}_{t+i}(s_{t+i},a_{t+i}),\frac{\epsilon}{8H})+ \gamma P_{s_{t+i},a_{t+i}  } \mathrm{clip}  (V_{ \underline{\rho}_{t+i}(s_{t+i},a_{t+i})            } -V_t, \frac{\epsilon}{8H}     )   \right ) .\nonumber
\end{align}
Then it follows that 
\begin{align}
\sum_{t\in \mathcal{V}_{k}}\hat{\beta}_t &\leq   \sum_{t\in \mathcal{V}_{k} } \sum_{i=0}^{H-1}\left(2\mathrm{clip}(\check{b}_{t+i}(s_{t+i},a_{t+i} ) ,\frac{\epsilon}{8H}) +P_{s_{t+i},a_{t+i}} \mathrm{clip}(     V_{        \underline{\rho}_{t+i}    (s_{t+i},a_{t+i})  }  -V_t      , \frac{\epsilon}{8H}   )  \right)   \nonumber 
\\ & \quad  +3H^2\sum_{t\in \mathcal{V}_{k} } U_{t}  \nonumber
\\ & \leq \sum_{t\geq 1} \left(2\mathrm{clip}(\check{b}_{t}(s_{t},a_{t}) ,\frac{\epsilon}{8H}) +P_{s_{t},a_{t}} \mathrm{clip}(     V_{        \underline{\rho}_{t}    (s_{t},a_{t})  }  -V_t      , \frac{\epsilon}{8H} )   \right)   +   6SAH^3B\ln(N_0) \label{eq_pat_ex6}
\\ &   = \sum_{t\geq 1} \tilde{\beta}_{t}   +   6SAH^3B\ln(N_0). \label{eq_pat_us1} 
\end{align} 
Here Inequality \eqref{eq_pat_ex6} holds because for each update, there is at most one element $t\in \mathcal{T}'$, such that $U_{t}=1$ due to this update.

By \eqref{eq_pat3} and \eqref{eq_pat_us1}, we have that 
\begin{align}
&\mathbb{P}\left[  \sum_{t\in  \mathcal{V}_{k}  } \beta_t \geq 12CH^2\iota+24SAH^3B\ln(N_0)  ,\quad  \sum_{t\geq 1} \tilde{\beta}_{t} < 3CH^2\iota\right]  \nonumber
\\ & \leq \mathbb{P}\left[ \sum_{t\in  \mathcal{V}_{k}  } \beta_t \geq 12CH^2\iota+24SAH^3B\ln(N_0)    , \quad \sum_{t\geq 1} \hat{\beta}_{t} <3CH^2\iota +6SAH^3B\ln(N_0) \right] 
\nonumber
\\ & \leq p. \nonumber
\end{align}
The proof is completed.

\subsection{Proof of  Lemma \ref{lemma_bound_b}}
\begin{proof}[Proof of  Lemma \ref{lemma_bound_b}]
	Recall that  $\check{b}_{t}(s_t,a_t) = 2\sqrt{\frac{H^2}{\check{n}_t(s_t,a_t) } \iota}$, so $\mathrm{clip}(\check{b}_t(s_t,a_t),\frac{\epsilon}{8H}) \leq 2\sqrt{\frac{H^2\iota}{\check{n}_t(s_t,a_t) } } \mathbb{I}[ \check{n}_t <  256\frac{H^4\iota}{\epsilon^2}  ]$. Noting that $\check{n}_t \geq \frac{n_t}{2HB}$, we obtain that
	\begin{align}
	\sum_{t\geq 1}\mathrm{clip}(\check{b}_{t}(s_t,a_t) ,\frac{\epsilon}{8H}) &  \leq SAH^2+ \sum_{t\geq 1}2\sqrt{\frac{2H^3B\iota}{n_t(s_t,a_t) }} \mathbb{I}\left[ n_{t}<  512\frac{H^5B\iota}{\epsilon^2}\right]   \nonumber
	\\& 	\leq   SAH^2+ 182\frac{SAH^4B\iota  }{\epsilon} .\nonumber
	\end{align}
	
\end{proof}

\subsection{Proof of Lemma \ref{lemma_bound_alpha}}

\begin{proof}[Proof of Lemma \ref{lemma_bound_alpha}]
	We fix $(s,a)$ and consider to bound $\alpha(s,a):=\sum_{t\geq 1}\alpha_{t}\mathbb{I}[(s_t,a_t)=(s,a),N_t(s,a)< N_0]$.
	Define $T(j,s,a)$ to be the set of indices of samples in the $j$-th  type-\uppercase\expandafter{\romannumeral1} stage with respect to $(s,a)$, i.e., $T(j,s,a):  = \{ t\geq 1|  (s_t,a_t)  = (s,a),  \sum_{i=1}^{j-1}\check{e}_i \leq N_{t}(s,a)<  \sum_{i=1}^{j}\check{e}_i  \}$.
	It  is then clear that for any $t\in T(j,s,a)$, $\underline{\rho}_{t}(s,a) = \rho(j-1,s,a)$ and $\overline{\rho}_{t}(s,a) = \rho(j+1,s,a)$. (The definitions of $\rho$, $\underline{\rho}_t$ and $\overline{\rho}_t$ are at the beginning of Section~\ref{analysis}.)

For $j\geq 2$, by the definition of $\alpha_{t}$ and  the fact $V_t$ is non-increasing in $t$, we obtain that 
	\begin{align}
	\sum_{t\in T(j,s,a)}\alpha_{t}\mathbb{I}[(s_t,a_t)=(s,a)] \leq  \check{e}_j P_{s,a}\left(  \mathrm{clip} (V_{\rho(j-1,s,a)}- V_{\rho(j+1,s,a)} ,\frac{\epsilon}{8H}  )\right), \nonumber
	\end{align}
	and therefore
	\begin{align}
	\alpha(s,a)\leq H\sum_{i=1}^{HB}\check{e}_i+ \sum_{HB+1\leq j \leq j_{\infty}(s,a)}\check{e}_j  P_{s,a} \left(  \mathrm{clip} (V_{\rho(j-1,s,a)}- V_{\rho(j+1,s,a)} ,\frac{\epsilon}{8H}  ) \right)  .												\label{eqbound_alpha} 
	\end{align}
	Here also recall that $j_t(s, a)$ is defined at the beginning of Section~\ref{analysis}, and $j_{\infty}(s, a)$ is defined to be $\max_{t \geq 1} j_{t}(s, a) \leq \check{J}$.
	
	We next define 
	\[
	j(s,a,s',\epsilon') :=\max\{j\leq j_{\infty}(s,a) | V_{\rho(j,s,a)}(s')-V^*(s')>\epsilon'    \}
	\]
	and
	\[
	\tilde{\tau}(s,a,s',\epsilon') := \sum_{i=1}^{j(s,a,s',\epsilon')} \check{e}_{i}
	\]
	for $s'\in \mathcal{S}$ and $\epsilon'>0$.  Let $\epsilon_{i}  = \frac{2^i \epsilon}{H}$ for $i= 0,1,2,\dots,k$ where $k = \lceil\log_2(\frac{H}{(1-\gamma)\epsilon})\rceil$.
	By \eqref{eqbound_alpha}, we have that
	\begin{align}
	\alpha(s,a)&\leq H\sum_{i=1}^{HB}\check{e}_i +\sum_{s'}\sum_{HB+1\leq j< j(s,a,s',\frac{\epsilon}{8H})+1 }\check{e}_j P_{s,a}(s') \left(V_{\rho(j-1,s,a)}(s') -V_{\rho(j+1,s,a)}(s')\right) \nonumber
	\\ & \leq O(BH^2\check{e}_1)  + \sum_{s'}\sum_{i=1}^{k} \sum_{ \max\{j(s,a,s',\epsilon_i), HB\}< j \leq j(s,a,s',\epsilon_{i-1}) } \check{e}_{j+1}P_{s,a}(s') \theta(s,a,s',j)\nonumber
	\\ & \leq O(BH^2 \check{e}_1) + \sum_{s'}\sum_{i=1}^{k} \frac{2\sum_{1\leq j\leq j(s,a,s',\epsilon_{i-1})}\check{e}_{j} }{HB}P_{s,a}(s') \sum_{ j(s,a,s',\epsilon_i)< j \leq j(s,a,s',\epsilon_{i-1}) } \theta(s,a,s',j)  \label{eqbound_alphaex1} 
	\\& =  O(BH^2\check{e}_1)  + \sum_{i=1}^{k} \frac{2 \tilde{\tau}(s,a,s',\epsilon_{i-1}) }{HB}P_{s,a}(s') \psi(s,a,s',i)  \nonumber
	\\ & \leq  O(BH^2\check{e}_1)  +\frac{4}{HB} \sum_{i=1}^{k}\tilde{\tau}(s,a,s',\epsilon_{i-1}) P_{s,a}(s') \epsilon_{i}, \label{eqbound_alphaex2}
	\end{align}
		where
	\begin{align} 
	&\theta(s,a,s',j): = V_{\rho(j,s,a)}(s')-V_{\rho(j+2,s,a)}(s') ,\nonumber \\
	&\psi(s,a,s',i):=\sum_{j(s,a,s',\epsilon_i)< j\leq  j(s,a,s',\epsilon_{i-1}) } \theta(s,a,s',j)\leq 2\epsilon_{i}. \nonumber 
	\end{align}
Here Inequality \eqref{eqbound_alphaex1} is by the fact $\check{e}_{j+1}\leq \frac{2}{HB}\sum_{i=1}^j \check{e}_{i}$ for $j\geq HB$ and Inequality \eqref{eqbound_alphaex2} is by the definition of $j(s,a,s',\epsilon_i)$.

In the next subsection, we will prove the following lemma.
	\begin{lemma}\label{martingale_gap}
	For any $\epsilon>0$,	with probability $1-(1+ SA(\check{J}+\bar{J}))p$ it holds that
	\begin{align}
	\sum_{s,a,s'}\tilde{\tau}(s,a,s',\epsilon)P_{s,a}(s')\leq O\left(\frac{SAH^5\ln(\frac{4H}{\epsilon})\iota }{\epsilon^2}+SAHB\ln(N_0)\right). \nonumber
	\end{align}
\end{lemma}
	
	Now, by \eqref{eqbound_alphaex2} and Lemma \ref{martingale_gap} we have that
	\begin{align}
  \sum_{t\geq 1}\alpha_{t}& = \sum_{s,a}\alpha(s,a)\nonumber
	\\ & \leq \sum_{s,a}\left(   BH^2\check{e}_1  +\frac{4}{HB} \sum_{s'}\sum_{i=1}^{k}\tilde{\tau}(s,a,s',\epsilon_{i-1}) P_{s,a}(s') \epsilon_{i} \right) \nonumber
	\\ &\leq O(SABH^3)  +O\left(\frac{4}{HB}\sum_{i=1}^k \left(\frac{SAH^5\ln(\frac{4H}{\epsilon})\iota}{\epsilon_{i-1}^2}+SAHB\ln(N_0) \right)\epsilon_{i} \right)  \label{eqbound_alphaex3}
	\\ & \leq O(SAB H^3) + O\left( \frac{1}{HB}\cdot \frac{SAH^6\ln(\frac{4H}{\epsilon})\iota}{\epsilon} +\frac{SA\ln(N_0)}{1-\gamma} \right) \nonumber 
	\\ & \leq     O\left(\frac{SAH^5\ln(\frac{4H}{\epsilon})\iota}{\epsilon B} +SABH^3+SAH\ln(N_0)\right)      .        \nonumber 
	\end{align}
The proof is completed.

\end{proof}

\subsection{Proof of Lemma \ref{martingale_gap}}

We first state the following auxiliary lemma, which  implies that we can learn the value function efficiently. The lemma is similar to Lemma 5 in \cite{zhang2020almost}, and is proved using the type-\uppercase\expandafter{\romannumeral2} updates. The proof of Lemma~\ref{lemma_lvf} will be presented immediately after this subsection.
\begin{lemma}\label{lemma_lvf}
	Conditioned on the successful event of $E_1$ defined in \eqref{eq_def_e1}, for any $\epsilon_{1} \in[\epsilon, \frac{1}{1-\gamma}]$ it holds that 
	\begin{align}\sum_{t=1}^{\infty}\mathbb{I} \left[ V_{t}(s_{t}) - V^*(s_{t}))\geq \epsilon_{1}  \right]\leq \sum_{t=1}^{\infty}\mathbb{I} \left[ Q_{t}(s_{t},a_t) - Q^*(s_{t},a_t))\geq \epsilon_{1}  \right] \leq O\left(\frac{SAH^5\ln(\frac{4H}{\epsilon})\iota}{\epsilon_{1}^2}\right) .\label{eq_lemma_lvf}\end{align}
\end{lemma}

With the help of Lemma \ref{lemma_lvf}, we prove Lemma \ref{martingale_gap} as follows.

\begin{proof}[Proof of Lemma \ref{martingale_gap}]
	We start with defining \[
	\tau(s,a,s',\epsilon) := \sum_{t\geq 1}\mathbb{I}\left[(s_t,a_t)=(s,a), V_t(s')-V^*(s')>\epsilon\right].
	\]
	Recalling that 	$\tilde{\tau}(s,a,s',\epsilon) = \sum_{i=1}^{j(s,a,s',\epsilon)}\check{e}_i$, we have 
	\[  \tilde{\tau}(s,a,s',\epsilon)=\sum_{i=1}^{j(s,a,s',\epsilon)}\check{e}_i \leq H+ (1+\frac{2}{H})\sum_{i=1}^{j(s,a,s',\epsilon)-1}\check{e}_i\leq H +(1+\frac{2}{H}) \tau(s,a,s',\epsilon). \]
	So it suffices to prove that 
	\begin{align}
	\sum_{s,a,s'}\tau(s,a,s',\epsilon)P_{s,a}(s')\leq  O\left(\frac{SAH^5\ln(\frac{4H}{\epsilon})\iota }{\epsilon^2}+SAHB\ln(N_0)\right). \label{eqlemma_mg}
	\end{align}
	To prove \eqref{eqlemma_mg}, we define $\lambda_t$ to be the vector such that $\lambda_{t}(s) =\mathbb{I}\left[ V_t(s)-V^*(s)>\epsilon  \right]$.  
	Note that  
	\[\sum_{s,a,s'}\tau(s,a,s',\epsilon)P_{s,a}(s') = \sum_{t\geq 1}P_{s_t,a_t}\lambda_{t}\]
	and due to the infrequent updates, we have that
	\[\sum_{t\geq 1}\left(\lambda_t(s_{t+1})-\lambda_{t+1}(s_{t+1}) \right)\leq \sum_{t\geq 1} \mathbb{I}\left[ V_{t}(s_{t+1})\neq V_{t+1}(s_{t+1})\right] \leq 2SAHB\ln(N_0).\] 
	For $C$ a large enough constant,  we obtain that
	\begin{align}
	&	\mathbb{P}\left[  \sum_{s,a,s'}\tau(s,a,s',\epsilon)P_{s,a}(s')\geq 4C\frac{SAH^5\ln(\frac{4H}{\epsilon})\iota}{\epsilon^2}+8SAHB\ln(N_0)   \right]   \nonumber
	\\ &  =  \mathbb{P} \left[  \sum_{t\geq 1}P_{s_t,a_t}\lambda_{t} \geq 4C\frac{SAH^5\ln(\frac{4H}{\epsilon})\iota}{\epsilon^2}+8SAHB\ln(N_0)  \right] \nonumber 
	\\ & \leq  \mathbb{P} \left[  \sum_{t\geq 1}P_{s_t,a_t}\lambda_{t} \geq 4C\frac{SAH^5\ln(\frac{4H}{\epsilon})\iota}{\epsilon^2}+8SAHB\ln(N_0) , \sum_{t\geq 1}\lambda_{t}(s_{t+1})\leq C\frac{SAH^5\ln(\frac{4H}{\epsilon})\iota}{\epsilon^2} +2SAHB \ln(N_0)   \right] \nonumber 
	\\ & \quad  +  \mathbb{P}\left[   \sum_{t\geq 1}\lambda_{t}(s_{t+1})> C\frac{SAH^5\ln(\frac{4H}{\epsilon})\iota}{\epsilon^2} +2SAHB \ln(N_0)    \right]  \nonumber 
	\\ & \leq  p+ \mathbb{P}\left[ \sum_{t\geq 1}\lambda_{t}(s_{t}) \geq  C\frac{SAH^5\ln(\frac{4H}{\epsilon})\iota}{\epsilon^2} \right] \label{eq_boundreg_ex1}
	\\ & \leq p+ \mathbb{P}\left[E_{1}\right]\label{eq_boundreg_ex2}
	\\ &\leq p+ SA(\check{J}+\bar{J})p \label{eq_boundreg_ex3},
	\end{align}
	where Inequality \eqref{eq_boundreg_ex1} is by Lemma \ref{lemma_berbound} with $M_k = \lambda_{k}(s_{k+1})$ and $\mathcal{F}_{k}  = \sigma(s_1,a_1,....,s_{k},a_{k},s_{k+1})$ for $k\geq 1$, Inequality \eqref{eq_boundreg_ex2} is by Lemma \ref{lemma_lvf} and Inequality \eqref{eq_boundreg_ex3} is by Proposition \ref{pro1}.
	The proof is completed.
\end{proof}

\subsection{Proof of Lemma~\ref{lemma_lvf}}

The proof of Lemma \ref{lemma_lvf} uses similar techniques as presented in  in Appendix.B of \cite{dong2019q}  and  Appendix.B.2 of \cite{zhang2020almost}. However, it requires more twists since the $Q$ function is only updated by at most $SA(\check{J}+ \bar{J})$  times for each state-action pair.

We first introduce a few simplified notations. Define $\delta^{t} := Q_{t}(s_t,a_t)-Q^*(s_{t},a_t)$. Clearly $\delta^t \geq V_t(s_t)-V^*(s_t)$ and $\sum_{t\geq 1}\mathbb{I}[\delta^t \geq x]\geq \sum_{t\geq 1}\mathbb{I}[V_t(s_t)-V^*(s_t) \geq x ] $ for any $x\geq 0$. Throughout this subsection, we use $\bar{n}^t$, $\bar{b}^t$ and $\bar{l}^{t}_{i}$ as short hands of $\bar{n}_t(s_t,a_t)$, $\bar{b}_t(s_t,a_t)$ and  $\bar{l}_{t,i}(s_t,a_t)$ respectively.

Conditioned on  $E_{1}$  defined in \eqref{eq_def_e1}, 
we note that \eqref{eq_pro4_1} and \eqref{eq_pro4_2} hold for any $j\geq 1$ and $j'\geq 1$ respectively.
We will use these inequalities without additional explanation.

Let $\mathcal{T}_{1} := \{t\geq 1 | N_{t}(s_t,a_t) \geq N_0   \}$.
 We then have the following lemma.
\begin{lemma} \label{lemma_lemma2_aux}
	Conditioned on successful event $E_1$ defined in \eqref{eq_def_e1}, it holds that for any $t\in \mathcal{T}_{1}$ (if $\mathcal{T}_1$ is not empty)
	\begin{align}  Q_{t}(s_{t},a_t)-Q^*(s_{t},a_t)\leq \frac{\epsilon}{2H} . \nonumber \end{align}
\end{lemma}
\begin{proof}
	For each $i = 1, 2, \dots, S$, if there are at least $i$ states with total visit number greater or equal to $N_0$, we let $s^{(i)}$ be the $i$-th such state (sorted in the order of time to reach $N_0$) and let $T_i$ be the corresponding time (i.e., $n_{T_i}(s^{(i)}) = N_0 \text{ and } s_{T_i}= s^{(i)}$ ). Otherwise we let $s^{(i)}$ be a random state in $\mathcal{S} \setminus \{s^{(1)},...,s^{(i-1)} \}$ and set $T_i   = \infty$.  
	 

It suffices prove that $V_{T_{i}}(s^{(i)})-V^*(s^{(i)})\leq \frac{\epsilon}{2H}$ for $s^{(i)}$ with finite $T_i$. We prove this by applying induction on $i$ to prove the stronger statement that $V_{T_{i}}(s^{(i)})-V^*(s^{(i)})\leq \frac{\epsilon i}{2HS}$.
	
\noindent \underline{\it Base case ($i = 1)$:} Note that for any $t\notin  \mathcal{T}_{1}$, we have following inequality by the update rule \eqref{equpdate2} and event $E_1$,
	\begin{align}
	\delta^{t} & =  Q_{t}(s_{t},a_{t})-Q^*(s_{t},a_{t})\nonumber
	\\ & \leq  \frac{\mathbb{I}\left[ \bar{n}^t= 0 \right]}{1-\gamma} + \left(\bar{b}^t + \frac{\gamma}{\bar{n}^t}\sum_{i=1}^{\bar{n}^t }V_{  \bar{l}^{t}_{i}    }( s_{ \bar{l}^{t}_{i} +1   })  - P_{s_{t},a_{t}}V^*      \right)  \nonumber
	\\ & \leq \frac{\mathbb{I}\left[ \bar{n}^t= 0 \right]}{1-\gamma} + \left(2\bar{b}^t + \frac{\gamma }{\bar{n}^t}\sum_{i=1}^{\bar{n}^t } \left(V_{  \bar{l}^{t}_{i}    }(  s_{ \bar{l}^{t}_{i} +1 }    ) - V^*( s_{ \bar{l}^{t}_{i} +1 }  )   \right)   \right)  \nonumber
	\\ &  \leq  \frac{\mathbb{I}\left[ \bar{n}^t= 0 \right]}{1-\gamma}  + 2\bar{b}^{t} +\frac{\gamma}{\bar{n}^{t}}\sum_{i=1}^{ \bar{n}^t } \left( V_{  \bar{l}^{t}_{i}+1    }(  s_{ \bar{l}^{t}_{i} +1 }  )-Q^*( s_{ \bar{l}^{t}_{i} +1 },a_{\bar{l}^{t}_{i} +1})  + \theta^{ \bar{l}^{t}_{i}+1}      \right) \nonumber
	\\ &  = \frac{\mathbb{I}\left[ \bar{n}^t= 0 \right]}{1-\gamma}  + 2\bar{b}^{t} +\frac{\gamma}{\bar{n}^{t}}\sum_{i=1}^{ \bar{n}^t } (\delta^{  \bar{l}^{t}_{i}+1     }  + \theta^{ \bar{l}^{t}_{i}+1}      )  \label{eq_lemma2_24},
	\end{align}
	where we define $\theta^{ \bar{l}^{t}_{i}+1}: = V_{\bar{l}^t_i}(  s_{ \bar{l}^{t}_{i} +1 }  )- V_{\bar{l}^t_i+1}(  s_{ \bar{l}^{t}_{i} +1 }  ) $ . 

	It is obvious that $t\notin \mathcal{T}_1$ if $t<T_{1}$. Then for  any non-negative weights $\{w_{t}\}_{t\geq 1}$, we have that 
	\begin{align}
	\sum_{t< T_1 }w_{t}\delta^{t} &\leq \sum_{t< T_1 }\frac{ w_{t} \mathbb{I}[\bar{n}^t=0]}{1-\gamma}+2\sum_{t < T_{1} }w_{t}\bar{b}^t 	 + \sum_{t< T_1}w'_{t}(\delta^{t}+\theta^{t}) , \label{eq_lemma2_25}
	\end{align}
	where 
	\begin{align}
	w'_{t} = \gamma\sum_{u< T_1} \frac{1}{\bar{n}^u}\sum_{i=1}^{\bar{n}^t}\mathbb{I}\left[ t = \bar{l}^u_i +1 \right].\label{eq_lemma2_su}
	\end{align}
	If we choose a sequence of non-negative weights $\{w_{t}\}_{t\geq 1}$ such that $\sup_{t< T_1 }w_t\leq C$ and $\sum_{t< T_1 }w_t \leq W$ for two positive constant $C$ and $W$, then for all $t \geq 1$, we have that
	\begin{align}
	w'_{t}\leq \gamma(1+\frac{1}{H})C\leq (1-\frac{1}{2H})C,
	\end{align}
	and
	\begin{align}
	\sum_{t< T_1}w'_{t}\leq  \gamma(1+\frac{1}{H})W\leq (1-\frac{1}{2H})W.
	\end{align}
	
\begin{lemma}\label{lemma_lemma2_aux_2}
    Let $\{w_{t}\}_{t\geq 1}$ be a sequence of non-negative weights such that $0\leq w_{t}\leq C$ for any $t\notin \mathcal{T}_{1}$ and $\sum_{t\notin \mathcal{T}_{1}}w_{t}\leq W$, then it holds that
    	\begin{align}
	 \sum_{t\notin \mathcal{T}_{1}}\frac{ w_{t} \mathbb{I}[\bar{n}^t=0]}{1-\gamma} &\leq \frac{CSAH}{1-\gamma}\leq CSAH^2,\\
	 2\sum_{t\notin \mathcal{T}_{1}}w_{t}\bar{b}^t &\leq  40(1+\frac{1}{H})\sqrt{SAH^3WC\iota} \leq 60\sqrt{SAH^3WC\iota}  ,\label{eq_lemma2_ex2}\\
	 \sum_{t\notin \mathcal{T}_{1}}w_{t}\theta^{t}&\leq  \frac{SAC}{1-\gamma}\leq SCH .
	\end{align}
\end{lemma}
\begin{proof}
The first inequality holds because $\sum_{t\geq 1}\mathbb{I}[\bar{n}^{t}=0]\leq SAH$, and the third inequality holds because $\sum_{t\geq 1}\mathbb{I}\left[s_{t}=s \right]\theta^{t}\leq 1/(1-\gamma)$. For the second inequality, we note that $\bar{b}^{t}\leq 2\sqrt{H^2\iota/\bar{n}^{t}}$, it then follows that 
\begin{align}
\sum_{t\notin \mathcal{T}_{1}}w_{t}\bar{b}^{t} &\leq 2\sqrt{H^2\iota}\sum_{t \notin \mathcal{T}_{1}} w_{t}\sqrt{1/\bar{n}^{t} }   \nonumber
\\ & =2\sqrt{H^2\iota}\sum_{s,a}\sum_{t\notin \mathcal{T}_{1}}\mathbb{I}\left[(s_{t},a_{t})=(s,a) \right]w_{t}\sqrt{1/\bar{n}^{t} }.  \nonumber
\end{align}
Let $\tilde{w}(s,a)   = \sum_{t\notin \mathcal{T}_{1}}w_{t}\mathbb{I}\left[(s_{t},a_{t})=(s,a)\right]$. We fix $\tilde{w}(s,a)$ and consider to maximize 
\begin{align}\sum_{t\notin \mathcal{T}_{1}}\mathbb{I}\left[(s_t,a_t)=(s,a) \right]w_{t}\sqrt{1/\bar{n}^t} . 
\nonumber\end{align}
Define $\bar{T}(j,s,a):= \{ t\geq 1|  (s_t,a_t)  = (s,a),  \sum_{i=1}^{j-1}\bar{e}_j \leq  N_{t}(s,a)< \sum_{i=1}^{j}\bar{e}_j  \}$. Note that for each $j\geq 2$, $\sum_{t\notin \mathcal{T}_{1},t\in \bar{T}(j,s,a)}w_{t}\leq (1+\frac{1}{H})C\bar{e}_{j-1} $. By rearrangement inequality we have that,
\begin{align}
\sum_{t\notin \mathcal{T}_{1}}\mathbb{I}\left[(s_t,a_t)=(s,a) \right]w_{t}\sqrt{1/\bar{n}^t} &=\sum_{j\geq 2}\left(\sum_{t\notin \mathcal{T}_{1},t\in \bar{T}(j,s,a)}w_{t}\right) \sqrt{1/\bar{e}_{j-1}} \nonumber
\\ & \leq C(1+\frac{1}{H})\sum_{j\geq 1}\sqrt{e_{j}}\mathbb{I}\left[ \sum_{i=1}^{j-1}Ce_{i}\leq \tilde{w}(s,a)   \right] \nonumber
\\ &\leq 10(1+\frac{1}{H})\sqrt{HC\tilde{w}(s,a)}.\nonumber
\end{align}
By Cauchy-Schwartz inequality, we obtain that
\begin{align}
\sum_{t\notin \mathcal{T}_{1}}w_{t}\bar{b}^{t}\leq 20(1+\frac{1}{H})\sqrt{H^3C\iota}\sum_{s,a}\sqrt{\tilde{w}(s,a)}\leq 20(1+\frac{1}{H})\sqrt{SAH^3WC\iota}.\nonumber
\end{align}
The proof is completed.

\end{proof}

By Lemma \ref{lemma_lemma2_aux_2}	we derive that
	\begin{align}
	\sum_{t< T_{1}}w_{t}\delta^{t} \leq \sum_{t< T_1}w'_{t}\delta^t +2SACH^2 +60 \sqrt{SAH^3WC\iota} \label{eq_lemma2_26}.
	\end{align}
	By iteratively unrolling \eqref{eq_lemma2_26} for $2H\ln(\frac{4H^2S}{\epsilon})$ times and setting  the initial weights by $w_{t} = \mathbb{I}\left[ s_{t}= s^{(1)} \right]$ so that $C=1$ and $W = N_0$, we have
	\begin{align}
	\sum_{t<T_1}\mathbb{I}\left[s_{t}=s^{(1)}\right]\delta^{t}\leq 2H\ln(\frac{4H^2S}{\epsilon})\left(2SAH^2 +60\sqrt{SAH^3N_0\iota} \right) +\frac{\epsilon\sum_{t<T_{1}}\mathbb{I}\left[s_{t}=s^{(1)} \right] }{4HS}.
	\end{align}
	If $V_{T_1}(s^{(1)})- V^*(s^{(1)})>\frac{\epsilon}{2HS}$, 	then $\mathbb{I}\left[ s_t=s^{(1)} \right] \delta^{t}>\frac{\epsilon}{2HS}$	for $t<T_1$ due to the fact that $V_t$ is non-increasing in $t$, which implies that
	\begin{align}
 \frac{\epsilon 	N_0}{4HS} \leq 2H\ln(\frac{4H^2S}{\epsilon})(2SAH^2 +60\sqrt{SAH^3N_0\iota} ),
	\end{align}
	which contradicts to the definition of $N_0$ ($N_0 = c_1 \frac{SAH^5S^2 \ln(\frac{4H^2S}{\epsilon})\iota}{\epsilon^2}$) .
As a result, we have that $V_{T_1}(s^{(1)})\leq  V^*(s^{(1)})+\frac{\epsilon}{2HS}$. 
	 
\noindent \underline{\it Induction step:} Now suppose that $V_{T_i}(s^{(i)})-V^*(s^{(i)}) \leq \frac{k\epsilon}{2HS}$ holds for all $1\leq i \leq k $ for some $k\geq 1$. We will prove that $V_{T_{k+1}}(s^{(k+1)})-V^{*}(s^{(k+1)})\leq \frac{(k+1)\epsilon}{2HS}$ assuming that $T_{k+1}\neq \infty$. 
	 
Note that if $t<T_{k+1}$ and $T\in \mathcal{T}_{1}$, $\delta^{t} \leq \frac{k\epsilon}{2HS}$. It then follows that for  non-negative weights  $\{w_{t}\}_{t\geq 1}$ such that $\sup_{t< T_{k+1} }w_t\leq C$ and $\sum_{t< T_{k+1} }w_t \leq W$,
	 \begin{align}
	 \sum_{t<T_{k+1}}w_{t}\delta^{t}  &\leq  \sum_{t<T_{k+1},t\notin \mathcal{T}_{1}}w_{t}\delta^{t}+ \sum_{t<T_{k+1},t\in \mathcal{T}_{1}}\frac{w_{t}k\epsilon}{2HS} \nonumber \\
	 & \leq \sum_{t<T_{k+1},t\notin \mathcal{T}_{1}} \left(\frac{w_{t} \mathbb{I}\left[\bar{n}^t = 0 \right] }{1-\gamma}  +2w_{t}\bar{b}^{t} \right) +\sum_{t<T_{k+1}}w'_{t}(\delta^{t}+\theta^{t})+ \sum_{t<T_{k+1},t\in \mathcal{T}_{1}}\frac{w_{t}k\epsilon}{2HS} \label{eq_lemma2_27} \\
	 & \leq 2SACH^2 +60\sqrt{SAH^3 W_{1}} +\sum_{t<T_{k+1}}w'_{t}\delta^{t}+\sum_{t<T_{k+1},t\in \mathcal{T}_{1}}\frac{w_{t}k\epsilon}{2HS} \label{eq_lemma2_28} \\
	 & \leq  2SACH^2 +60\sqrt{SAH^3 W_{1}} +\sum_{t<T_{k+1}}w'_{t}\delta^{t}+ \frac{(W-W_{1} )  k\epsilon  }{2HS},\label{eq_lemma2_29}
	 \end{align}
	 where
	 $W_{1} = \sum_{t<T_{k+1},t\notin \mathcal{T}_{1} }w_{t}$  and
	 $w'_{t} = \gamma\sum_{u< T_{k+1},u\notin \mathcal{T}_{1}} \frac{1}{\bar{n}^u}\sum_{i=1}^{\bar{n}^t}\mathbb{I}\left[ t = \bar{l}^u_i +1 \right].$ Here, Inequality \eqref{eq_lemma2_28} is by Lemma \ref{lemma_lemma2_aux_2}.
	 Because $w'_{t}\leq (1-\frac{1}{2H})C  ,\forall t\geq 1$ and $\sum_{t<T_{k+1},t\notin \mathcal{T}_{1}}w'_{t}\leq (1-\frac{1}{2H})W_{1}$, by iteratively applying \eqref{eq_lemma2_29} for $2H\ln(\frac{3H^2S}{\epsilon})$ times, we have that
	 \begin{align}
	 	\sum_{t<T_{k+1}}w_{t}\delta^{t} \leq  2H\ln(\frac{4H^2S}{\epsilon}) \left(2SAH^2 +60\sqrt{SAH^3N_0\iota} \right) + \frac{Wk\epsilon}{2HS} +\frac{ W\epsilon  }{4HS}. \label{eq_lemma2_30}
	 	\end{align}
	 	If $V_{T_{k+1}}(s^{(k+1)})-V^*(s^{(k+1)})> \frac{(k+1)\epsilon}{2HS}$, 
	 	choosing $w_{t} = \mathbb{I}\left[s_t = s^{(k+1)},t<T_{k+1} \right]$ so that $C=1$ and $W=N_0$ in \eqref{eq_lemma2_30}, we obtain that
	 	\begin{align}
	 	\frac{N_0(k+1)\epsilon}{2HS} \leq   2H\ln(\frac{4H^2S}{\epsilon}) \left(2SAH^2 +60\sqrt{SAH^3N_0\iota} \right) + \frac{N_0k\epsilon}{2HS} +\frac{ N_0\epsilon  }{4HS}, \nonumber
	 	\end{align}
	 which again contradicts to the definition of $N_0$. Therefore we have proved that $V_{T_{k+1}}(s^{(k+1)})-V^{*}(s^{(k+1)})\leq \frac{(k+1)\epsilon}{2HS}$.
\end{proof}


\begin{proof}[Proof of Lemma~\ref{lemma_lvf}]
Let $\epsilon_{1}\in [\epsilon,\frac{1}{1-\gamma}]$ be fixed. Let $\{ w_{t}\}_{t\geq 1}$ be a non-negative sequence such that $\sup_{t\geq 1}w_{t}\leq C$ and $\sum_{t\geq 1}w_{t}\leq W$. Following the derivation of \eqref{eq_lemma2_26} we have that

\begin{align}
\sum_{t \geq 1}w_{t}\delta^{t} & = \sum_{t\geq 1,t\notin \mathcal{T}_{1}}w_{t}\delta^{t} +\sum_{t\geq 1,t\in \mathcal{T}_{1}}w_{t}\delta^{t} \nonumber
\\&\leq  \sum_{    t\geq 1, t\notin \mathcal{T}_{1}       }   w_{t}\delta^t +\frac{W_{1}\epsilon}{2H}  \label{eq_lemma2_30.5}
\\ & \leq \sum_{t\geq 1}w'_{t}\delta^t   +2SACH^2 +60 \sqrt{SAH^3WC\iota}  +\frac{W_{1}\epsilon}{2H} . \label{eq_lemma2_31}
\end{align}
where $\{w'_{t}\}_{t\geq 1}= \gamma\sum_{u\geq 1 ,u\notin \mathcal{T}_1} \frac{1}{\bar{n}^u}\sum_{i=1}^{\bar{n}^t}\mathbb{I}\left[ t = \bar{l}^u_i +1 \right]$ and $W_{1} = \sum_{t\in \mathcal{T}_{1}}w_{t}$. Similarly, it holds that $w'_{t}\leq (1-\frac{1}{2H})C,\forall t\geq 1$ and $\sum_{t\geq 1}w'_{t}\leq (1-\frac{1}{2H})(W-W_{1})$. Here Inequality \eqref{eq_lemma2_30.5} holds by Lemma \ref{lemma_lemma2_aux} and Inequality \eqref{eq_lemma2_31} holds by Lemma \ref{lemma_lemma2_aux_2}. 
Again by applying \eqref{eq_lemma2_31} iteratively for $2H\ln(\frac{4H}{\epsilon})$ times, we have that
\begin{align}
\sum_{t\geq 1}w_{t}\delta^{t}\leq 2H\ln(\frac{4H}{\epsilon})\left(  2SACH^2 +60 \sqrt{SAH^3WC\iota}   \right)+ \frac{W\epsilon}{2H}+\frac{W\epsilon}{4}. \label{eq_lemma2_38}
\end{align}
By choosing $w_{t} = \mathbb{I}\left[ \delta^{t}>\epsilon_{1}\right] $ so that $C=1$ and $W =N(\epsilon_{1}):= \sum_{t\geq 1}\mathbb{I}\left[ \delta^t>\epsilon_{1}\right]$ into \eqref{eq_lemma2_38}, we obtain that
\begin{align}
\frac{N(\epsilon_{1})\epsilon_{1}}{2} \leq  2H\ln(\frac{4H}{\epsilon})\left(  2SAH^2 +60 \sqrt{SAH^3N(\epsilon_{1})\iota}   \right),
 \end{align} 
which means that $N(\epsilon_{1})\leq O(\frac{SAH^5\ln(\frac{4H}{\epsilon})\iota}{\epsilon_{1}^2})$. The proof is completed.
\end{proof}

\subsection{Proof of Lemma~\ref{eq:lemma_add}}
\begin{proof}[Proof of Lemma~\ref{eq:lemma_add}]
By Lemma~\ref{lemma_lemma2_aux}, conditioned on the successful event $E_1$,  for any $t$ such that $N_t(s_t,a_t)\geq N_0$, it holds that $Q_t(s_t,a_t)- Q^*(s_t,a_t)\leq \frac{\epsilon}{2H}<\frac{3\epsilon}{4H}$, which implies that $\mathrm{clip}(Q_t(s_t,a_t)- Q^*(s_t,a_t),\frac{3\epsilon}{4H}) = 0$.
\end{proof}

\section{Achieving Asymptotically Near-Optimal Sample Complexity } \label{app:proof-thm-2}

As mentioned in Section~\ref{sec:tech-overview}, in the \UCBSA algorithm, we set $B$ to be a much larger value (indeed, $B = H^3$), an employ the reference-advantage decomposition variance reduction technique~\cite{zhang2020almost}, and re-design the exploration bonus $\check{b}$ to incorporate the Bernstein-type variance estimation. To prove Theorem~\ref{thm2} (the sample complexity bound for \UCBSA), in the analysis we split the error incurred due to the exploration bonus into two parts: the \emph{bandit loss} $b^*_{t}(s_{t},a_{t})$ (defined in \eqref{eq_def_banditloss}) and the rest part that is due to the estimation variance of the real bandit loss. While the second part can be dealt with the variance reduction technique (Lemma~\ref{lemma_bd_b_1}), the bandit loss contributes the main $\tilde{O}(SAH^3 \iota/\epsilon^2)$ term in the sample complexity (Lemma~\ref{boundnew}).




The rest of this section is organized as follows. In Appendix~\ref{sec:alg-ucbsa}, we present the details of the \UCBSA algorithm. In Appendix~\ref{App:C.2}, we prove Theorem~\ref{thm2}, while the proofs of all technical lemmas are deferred to Appendix~\ref{App:C.3}.

\subsection{The \UCBSA Algorithm} \label{sec:alg-ucbsa}
The \UCBSA  algorithm (Algorithm \ref{alg2}) has almost the same updating structure as \UCBS.  
More specifically,
the stopping condition and update triggers of \UCBSA are the same as that of \UCBS. 
The main difference between these two algorithms is 1) that \UCBSA utilized a more delicate exploration bonus with the help of a reference value function in the type-\uppercase\expandafter{\romannumeral1} updates; 2) we set $B = H^3$ in \UCBSA. 
Recall  $\check{\mathcal{L}} =\{\sum_{i=1}^{j} \check{e}_{i}|1\leq j\leq \check{J} \}  \text{~and~}  \bar{\mathcal{L}} = \{ \sum_{i=1}^j \bar{e}_{i}| 1\leq j\leq \bar{J} \}$.

\paragraph{The Statistics.} Besides the statistics maintained in \UCBS, we let $\mu^{\reff}$ and  $\sigma^{\reff}$ be the accumulators of the reference value function and square of the reference value function respectively. Different from \UCBS, in \UCBSA we use $\check{\mu}$ and $\check{\sigma}$ denote respectively the accumulator of the advantage function and square of the advantage function in the current type-\uppercase\expandafter{\romannumeral1} stage.

\begin{algorithm}[tb]
	\caption{\UCBSA}
	\begin{algorithmic}\label{alg2}
		\STATE{\textbf{Initialize:}	$\forall (s,a)\in \mathcal{S}\times \mathcal{A}$: $Q(s,a),Q^{\reff}(s,a)\leftarrow \frac{1}{1-\gamma}$, $N(s,a), \check{N}(s,a), \bar{N}(s,a), \check{\mu}(s,a),$ $ \bar{\mu}(s,a)\leftarrow 0$;}
		\FOR{$t=1,2,3,\dots$}
		\STATE{Observe $s_{t}$;}
		\STATE{Take action $ a_{t}= \arg\max_{a}Q(s_{t},a)$ and observe $s_{t+1}$;}
		\STATE{\verb|\\| \emph{Maintain the statistics}}
		\STATE{ $(s,a,s')\leftarrow (s_{t},a_{t},s_{t+1})$;}
		\STATE{ $n: = N(s,a)\stackrel{+}{\leftarrow} 1$; \; $\check{n}:=\check{N}(s,a)\stackrel{+}{\leftarrow}1$; \; $\bar{n}:=  \bar{N}(s,a)\stackrel{+}{\leftarrow} 1$;} 
		\STATE{ $ \check{\mu}:= \check{\mu}(s,a)\stackrel{+}{\leftarrow}  V(s')-V^{\reff}(s')$; \; $\mu^{\reff}: =\mu^{\reff}(s,a)\stackrel{+}{\leftarrow} V^{\reff}(s')$; \;  $  \bar{\mu} :=\bar{\mu}(s,a) \stackrel{+}{\leftarrow}V(s')$; }
	\STATE{$\check{\sigma} := \check{\sigma}(s,a)\stackrel{+}{\leftarrow} ( V(s')-V^{\reff}(s') )^2$; \;  $\sigma^{\reff}:= \sigma^{\reff}(s,a)\stackrel{+}{\leftarrow} (V^{\reff}(s'))^2$;  }
		
		\STATE{ \verb|\\| \emph{Update triggered by a type-\uppercase\expandafter{\romannumeral1} stage}}
		\IF{$n\in \check{\mathcal{L}} $}
		\vspace{-3ex}
		\STATE{\begin{align}
			 & \check{b} \resizebox{.84\hsize}{!}{   $\leftarrow\min\{ 2\sqrt{2}\left(\sqrt{\frac{\check{\sigma}/\check{n}- (\check{\mu}/\check{n})^2  }{\check{n}}\iota } +\sqrt{\frac{\sigma^{\reff}/n -(\mu^{\reff}/n)^2  }{n} \iota} \right)+7\left(\frac{H\iota^{3/4}}{ n^{3/4}}+\frac{H\iota^{3/4}}{\check{n}^{3/4}}  \right)  +4\left( \frac{H\iota}{n} +   \frac{H\iota}{\check{n}}  \right) ,\frac{1}{1-\gamma} \};$} \label{eq_thm2_checkb} \\ 
		    & \displaystyle{Q(s,a) \leftarrow \min\{  r(s,a)+\gamma \big(\check{\mu}/\check{n}+\mu^{\reff}/n +\check{b}\big ) , Q(s,a) \}}  \label{equpdate_alg2_1}\\
        &\check{N}(s,a)  \leftarrow0; \quad  \check{\mu}(s,a )\leftarrow 0;\quad V(s)\leftarrow \max_{a}Q(s,a); \qquad\qquad\qquad\qquad\qquad\qquad\qquad\qquad\qquad \qquad\nonumber
			\end{align}}
		\vspace{-3ex}
		
		\ENDIF
				\STATE{ \verb|\\| \emph{Update triggered by a type-\uppercase\expandafter{\romannumeral2} stage}}
		\IF{$n\in \bar{\mathcal{L}}$}
		\vspace{-3ex}
		\STATE{\begin{align}
			&  \bar{b}  \leftarrow \min \{ 2\sqrt{ H^2\iota/ \bar{n} }  ,1/(1-\gamma)  \}; \nonumber \\ 
			 & \displaystyle{Q(s,a) \leftarrow \min\{  r(s,a)+\gamma \big(  \bar{\mu}/\bar{n} +\bar{b}\big) , Q(s,a) \} ;}  \label{equpdate_alg2_2}\\
		     &\bar{N}(s,a)  \leftarrow0 ;\quad \bar{\mu}(s,a )\leftarrow 0; \quad V(s)\leftarrow \max_{a}Q(s,a); \qquad\qquad\qquad\qquad\qquad\qquad\qquad\nonumber
			\end{align}}
		\vspace{-3ex}
		\ENDIF
		
				  \STATE{ \textbf{if} $\sum_{a'}N(s,a')= N_{1}$ \textbf{then} $V^{\mathrm{ref}}(s)\leftarrow V(s)$; } \COMMENT{{\it Learn the reference value function}}
		\ENDFOR

	\end{algorithmic}
\end{algorithm}

\subsection{ Proof of Theorem \ref{thm2}}\label{App:C.2}

We start from showing that the $Q$ function is optimistic and non-increasing. 
\begin{proposition}\label{pro2}
	With probability $\left(1-SA\left(4\check{J}(2\log_2(N_0 H)  +1 )  + \bar{J}\right) p \right)$, it holds that $Q_{t}(s,a)\geq Q^*(s,a)$ and $Q_{t+1}(s,a)\leq Q_{t}(s,a)$ for any $t\geq 1$ and $(s,a)\in \mathcal{S}\times \mathcal{A}$ .
\end{proposition}
In the proof of Proposition~\ref{pro2} in Appendix~\ref{subsubsec_proof_pro2}, we introduce the desired event $E_{2}$ by \eqref{eq_def_E2}. Moreover, we use $\overline{E}_{2}$ to denote the complement event of $E_{2}$.
As will be shown later in \eqref{eq_thm2_0.5}, we have 
$$\mathbb{P}\left[E_2 \right]\geq \left(1-SA\left(4\check{J}(2\log_2(N_0 H)  +1 )  + \bar{J}\right) p \right),$$ and thus 
$$\mathbb{P}\left[ \overline{E}_2\right]\leq SA\left(4\check{J}(2\log_2(N_0 H)  +1 )  + \bar{J}\right) p. $$
The analysis will be done assuming the successful event $E_{2}$ throughout the rest of this section.

Since the type-\uppercase\expandafter{\romannumeral2} stages in \UCBSA are exactly the same as that in \UCBS, using the the same way as in the proof of~Lemma \ref{lemma_lvf}, we can prove the following lemma (and the proof is omitted).
\begin{lemma}\label{lemma_lvf1}
	Conditioned on  $E_{2}$, for any $\epsilon_{1}\in [\epsilon,\frac{1}{1-\gamma}]$, it holds that 
\[
\sum_{t=1}^{\infty}\mathbb{I} \left[ V_{t}(s_{t}) - V^*(s_{t})\geq \epsilon_{1}  \right]\leq \sum_{t=1}^{\infty}\mathbb{I} \left[ Q_{t}(s_{t},a_t) - Q^*(s_{t},a_t)\geq \epsilon_{1}  \right]\leq O\left(\frac{SAH^5\ln(\frac{4H}{\epsilon})\iota}{\epsilon_{1}^2}\right).
\]
\end{lemma}

Recall that $\mathcal{T}_1 = \{t| N_t(s_t,a_t)>N_0\}$. Similar as Lemma~\ref{lemma_lemma2_aux}, we have that (the proof is omitted)
\begin{lemma}\label{lemma_lvf2_aux}
	Conditioned on successful event $E_2$ , it holds that for any $t\in \mathcal{T}_{1}$ (if $\mathcal{T}_1$ is not empty)
	\begin{align}  Q_{t}(s_{t},a_t)-Q^*(s_{t},a_t)\leq \frac{\epsilon}{2H} . \nonumber \end{align}
\end{lemma}

  Define $\lambda_{t}$ to be the vector such that $\lambda_{t}(s) = \mathbb{I}\left[ \sum_{a}N_{t}(s,a)< N_1 \right]$ where $N_1:= c_{10}SAH^5B\ln(\frac{4H}{\epsilon})\iota$ for some large enough constant $c_{10}$. By Lemma \ref{lemma_lvf1}, $\lambda_{t}(s)=0$ implies that $V^{\reff}_{t}(s)=V^{\Reff}(s)$.
  
We then show that the Bellman error of the $Q$-function is properly bounded.
\begin{lemma}\label{lemma:bd_bellman} Define $l_{i}(s,a)$ to be the time the $i$-th visit of $(s,a)$ occurs and $\bar{N}_{t}(s,a)$ to be the visit count of $(s,a)$ before the current stage of $(s,a)$.
Conditioned on $E_2$, it holds that
\begin{align}
    Q_{t}(s,a)- r(s,a)-P_{s,a}V_{t} \leq   P_{s,a}(V_{\udl{\rho}_{t}}(s,a) -V_{t})+ P_{s,a}\tilde{\lambda}_{t}(s,a)\label{eq:bd_bellman}
\end{align}
for any $t\geq 1$ and any $(s,a)\in \mathcal{S}\times \mathcal{A}$, where
\begin{align}
    \tilde{\lambda}_{t}(s,a) : =\frac{1}{1-\gamma }\left(\frac{1}{\bar{N}_t(s,a)}\sum_{i=1}^{\bar{N}_t(s,a)} \lambda_{l_i(s,a)}  \right).\nonumber
\end{align}
\end{lemma}
The proof of Lemma~\ref{lemma:bd_bellman} is given in Section~\ref{sec:proof-bellman}.
We now define the bandit loss
\begin{align}
b^*_{t}(s,a) : = \min\{2\sqrt{2}\sqrt{\frac{\mathbb{V}(P_{s,a},V^*) \iota}{n_{t}(s,a)}} ,\frac{1}{1-\gamma} \}. \label{eq_def_banditloss}
\end{align}

By \eqref{eq:bd_bellman}, with the definition that  $\tilde{w}_t(s,a) := w_{t}(s,a) \cdot \mathbb{I}[N_t(s,a)< N_0] $ we can show that
\begin{align}
&V_{t}(s)-V^{\pi_{t}}(s) \nonumber
\\ &\leq  \sum_{s,a}\tilde{w}_{t}(s,a)  \left(  2\check{b}_{t}(s,a) +P_{s,a}\tilde{\lambda}_{t}(s,a) +\gamma P_{s,a}(V_{ \udl{\rho}_{t}(s,a)} -V_{t}   )      \right) \nonumber
\\ & \qquad \qquad + \sum_{s,a}w_{t}(s,a)\mathbb{I}[N_t(s,a)\geq N_0]\cdot (Q_t(s,a)-Q^*(s,a))+ \frac{\epsilon}{8} \nonumber
\\ &  = 2\sum_{s,a}\tilde{w}_{t}(s,a)b^*_{t}(s,a) +2\sum_{s,a}\tilde{w}_{t}(s,a)(\check{b}_{t}(s,a)-b^*_{t}(s,a))  +\gamma\sum_{s,a}\tilde{w}_{t}(s,a)P_{s,a}(V_{ \udl{\rho}_{t}(s,a)} -V_{t}    )\nonumber
\\ & \qquad \qquad +\sum_{s,a}\tilde{w}_{t}(s,a)P_{s,a}\tilde{\lambda}_{t}(s,a)\nonumber
\\ & \qquad \qquad +\sum_{s,a}w_{t}(s,a)\mathbb{I}[N_t(s,a)\geq N_0]\cdot (Q_t(s,a)-Q^*(s,a))
+\frac{\epsilon}{8} \nonumber
\\ & \leq 2\sum_{s,a}\tilde{w}_{t}(s,a)b^*_{t}(s,a)  +2\sum_{s,a}\tilde{w}_{t}(s,a) \clip (\check{b}_{t}(s,a)-b^*_{t}(s,a), \frac{\epsilon}{16H}) \nonumber 
\\ & \qquad \qquad   +\gamma\sum_{s,a}\tilde{w}_{t}(s,a)P_{s,a } \clip(V_{ \udl{\rho}_{t}(s,a) }-V_{t}   ,\frac{\epsilon}{16H}) + \sum_{s,a}\tilde{w}_{t}(s,a)P_{s,a}\mathrm{clip}(\tilde{\lambda}_{t}(s,a),\frac{\epsilon}{16H} ) \nonumber
\\ & \qquad \qquad +\sum_{s,a}w_{t}(s,a)\mathbb{I}[N_t(s,a)\geq N_0]\cdot \mathrm{clip}(Q_t(s,a)-Q^*(s,a),\frac{3\epsilon}{4H})+ \frac{7\epsilon}{8}\label{eq_thm2_1}
\\ & =  2\sum_{s,a}\tilde{w}_{t}(s,a)b^*_{t}(s,a) + \beta_t + \frac{7\epsilon}{8}.\label{eq_thm2_1.01}
\end{align}
where we re-define $\beta_t$ as follows.
\begin{align}
& \beta_{t}:=    \sum_{s,a}\tilde{w}_{t}(s,a) \left(2\mathrm{clip}(\check{b}_{t}(s,a) -b^*_{t}(s,a), \frac{\epsilon}{16H}) + \gamma P_{s,a} \mathrm{clip}( V_{\underline{\rho}_t(s,a)} -V_{t},\frac{\epsilon}{16H} ) \right.\nonumber\\
 & \qquad\qquad\qquad\qquad\qquad\qquad\qquad\qquad\qquad\qquad\qquad\qquad\qquad\qquad\qquad\left. + P_{s,a}\mathrm{clip}(\tilde{\lambda}_t(s,a),\frac{\epsilon}{16H} ) \right)
 \nonumber
 \\& \qquad \qquad  + \sum_{s,a}w_{t}(s,a)\mathbb{I}[N_t(s,a)\geq N_0]\cdot \mathrm{clip}(Q_t(s,a)-Q^*(s,a),\frac{3\epsilon}{4H}). \nonumber
\end{align}
Plugging in the definition of $\tilde{w}_t$, we get that
\begin{align}
 \\&\beta_t   = \sum_{s,a}w_{t}(s,a)\mathbb{I}[N_t(s,a)< N_0] \Big(2\mathrm{clip}(\check{b}_{t}(s,a) -b^*_{t}(s,a), \frac{\epsilon}{16H}) + \gamma P_{s,a} \mathrm{clip}( V_{\underline{\rho}_t(s,a)} -V_{t},\frac{\epsilon}{16H} )\nonumber
\\& \qquad \qquad  \qquad \qquad  \qquad \qquad  \qquad \qquad  \qquad \qquad  \qquad \qquad  \qquad \qquad  \qquad  +P_{s,a}\mathrm{clip}(\tilde{\lambda}_t(s,a),\frac{\epsilon}{16H} )  \Big)\nonumber
\\& \quad \quad +\sum_{s,a}w_{t}(s,a)\mathbb{I}[N_t(s,a)\geq N_0]\cdot \mathrm{clip}(Q_t(s,a)-Q^*(s,a),\frac{3\epsilon}{4H}). \nonumber
\end{align}
We also re-define the following notations,
 \begin{align}
 &\alpha_{t}:=\mathbb{I}[N_t(s_t,a_t)< N_0]P_{s_t,a_t} \mathrm{clip} ( V_{\underline{\rho}_{t}(s_t,a_t)      } -V_t ,\frac{\epsilon}{16H}), \nonumber
 \\ & \upsilon_{t} :=\mathbb{I}[N_t(s_t,a_t)< N_0]P_{s,a}\mathrm{clip}(\tilde{\lambda}_{t}(s,a),\frac{\epsilon}{16H}), \nonumber
 \\ &  \tilde{\beta}_{t}: =  \mathbb{I}[N_t(s_t,a_t)<N_0]\cdot \Big(2\mathrm{clip}(\check{b}_{t}(s_t,a_t)- b^*_{t}(s_{t},a_{t}),\frac{\epsilon}{16H}) + P_{s_t,a_t} \mathrm{clip} ( V_{\underline{\rho}_{t}(s_t,a_t)      } -V_t ,\frac{\epsilon}{16H}) \nonumber
 \\ & \qquad \qquad \qquad \qquad \qquad \qquad \qquad \qquad \qquad \qquad \qquad \qquad \qquad \qquad \qquad  +P_{s,a}\mathrm{clip}(\tilde{\lambda}_{t}(s,a),\frac{\epsilon}{16H}) \Big) \nonumber 
\\& \qquad \qquad  + \mathbb{I}[N_t(s_t,a_t)\geq N_0]\cdot \mathrm{clip}(Q_t(s_t,a_t)-Q^*(s_t,a_t),\frac{3\epsilon}{4H}). \nonumber
 \end{align}
Therefore, we have that
\begin{align}
&\tilde{\beta}_t =   \mathbb{I}[N_t(s_t,a_t)< N_0]\cdot 2\mathrm{clip}(\check{b}_{t}(s_t,a_t)- b^*_{t}(s_{t},a_{t}),\frac{\epsilon}{16H}) +\alpha_t + \upsilon_t  \nonumber
\\ & \qquad \qquad  \qquad \qquad \qquad \qquad \qquad \qquad \qquad + \mathbb{I}[N_t(s_t,a_t)\geq N_0]\cdot \mathrm{clip}(Q_t(s_t,a_t)-Q^*(s_t,a_t),\frac{3\epsilon}{4H}) .\nonumber
\end{align}
To handle the first term in \textbf{RHS} of \eqref{eq_thm2_1}, we prove that
\begin{lemma}\label{boundnew}Define $\Lambda =\left\lceil\log_{2}(\frac{256H^4}{\epsilon^2})\right\rceil $. With probability $(1-2H\Lambda p)$, it holds that
	\begin{align}
	\sum_{t\geq 1}\mathbb{I}\left[   \sum_{s,a}w_{t}(s,a)\mathbb{I}[N_t(s,a)< N_0]b^*_{t}(s,a)\geq \frac{\epsilon}{16}  \right]\leq O\left(    \frac{SAH^3\Lambda^3\iota}{\epsilon^2}  +\frac{SAH^4B\Lambda^2 \ln(N_0) }{\epsilon}    \right) .\nonumber
	\end{align}
\end{lemma}
We remark that our proof of Lemma~\ref{boundnew} is quite similar to the method of \emph{knowness} in \cite{lattimore2012pac}, in the sense that both methods rely on an argument based on the partition of the states. However, our way of partitioning seems to be simpler as we divide the states into different subsets only according to their numbers. The detailed proof is presented in Appendix~\ref{sec:proof-boundnew}. 

For the second term,   in Appendix~\ref{app:proof-lemma_bd_b_1}, we prove the pseudo-regret bounds as below.
\begin{lemma}\label{lemma_bd_b_1}
If we choose $B=H^3$, with probability $1 - SA\check{J} (2\mathbb{P}[\overline{E}_2] +4p   ) $  it holds that
\begin{align}
&	\sum_{t\geq 1}\mathrm{clip}( \check{b}_{t}(s_{t},a_{t})-b^*_{t}(s_{t},a_{t}) ,\frac{\epsilon}{16H} )\nonumber
\\& \leq  O\left(   \frac{SAH^2\iota}{\epsilon} \right) +\tilde{O}\left(   \frac{ S^{3/2}A^{3/2} H^{17/4} \iota  }{\epsilon^{1/2}} +  \frac{SAH^{59/12} \iota}{ \epsilon^{1/3}} +\frac{   S^{5/4}A^{5/4}H^{3}\iota     }{\epsilon^{1/4}}   +S^2A^2H^{9}\iota \right). \nonumber
\end{align}
\end{lemma}

Following the same arguments as  the proof of Lemma \ref{lemma_bound_alpha}, for the third term we show the following lemma (the proof of which is omitted).
\begin{lemma}\label{lemma_bd_alpha_1}
With probability $1-(\mathbb{P}\left[\overline{E}_2\right]+p)$  it holds that
	\begin{align}
  	\sum_{t\geq 1}\alpha_{t}\leq O\left(\frac{SAH^5\ln(\frac{4H}{\epsilon})\iota}{\epsilon B} +SABH^3+SAH\ln(N_0)\right)         . \nonumber
	\end{align}
\end{lemma}

Finally, in Appendix~\ref{app:pf-bd-ups}, we show the following lemma.
\begin{lemma}\label{lemma_bd_ups}
With probability $1-(\mathbb{P}\left[\overline{E}_2\right]+p)$, it holds that
\begin{align}
    \sum_{t\geq 1}\upsilon_{t}\leq O\left( \frac{H^2S(N_1+1)}{\epsilon} \right) .\nonumber
\end{align}
\end{lemma}

Similarly to the proof of Lemma~\ref{lemma_lvf2_aux}, we also have the following lemma.
\begin{lemma}\label{lemma_add_1}
With probability $1-(\mathbb{P}\left[\overline{E}_2\right]+p)$, for any $t$ it holds that 
\begin{align}
  \mathbb{I}[N_t(s_t,a_t)\geq N_0]\cdot \mathrm{clip}(Q_t(s_t,a_t)-Q^*(s_t,a_t),\frac{3\epsilon}{4H})  = 0.  \nonumber
\end{align}
\end{lemma}

By Lemmas  \ref{lemma_bd_b_1}, \ref{lemma_bd_alpha_1}, \ref{lemma_bd_ups} and \ref{lemma_add_1}, we obtain that 
\begin{lemma}
With probability $1 - \big(SA\check{J} (2\mathbb{P}[\overline{E}_2] +4p   )+3\mathbb{P}\left[ \overline{E}_2\right]+3p\big) $,  it holds that
	\begin{align}
	\sum_{t\geq 1}\tilde{\beta}_{t}\leq O\left(\frac{SAH^2\ln(\frac{4H}{\epsilon})\iota}{\epsilon} \right) +\tilde{O}\left(\frac{S^2A^2 H^{10}\iota }{\epsilon^{1/2}}\right).
	\end{align}
\end{lemma}

Following the same arguments in Section \ref{subsection:pet}, we obtain that with probability 
\[
1 - \left(SA\check{J} (2\mathbb{P}[\overline{E}_2] +4p   )+3\mathbb{P}\left[ \overline{E}_2\right]+34p \right) ,
\]
it holds that
\begin{align}
\sum_{t\geq 1}\mathbb{I}\left[   \beta_{t}> \frac{\epsilon}{8}    \right]   \leq O(\frac{SAH^2\ln(\frac{4H}{\epsilon})\iota}{\epsilon^2} ) +\tilde{O}(\frac{S^2A^2 H^{10}\iota }{\epsilon^{3/2}}).\label{eq_thm2_2}
\end{align}

By Proposition \ref{pro2},\eqref{eq_thm2_1.01} and  \eqref{eq_thm2_2}, we conclude that with probability $1 - \big(SA\check{J} (2\mathbb{P}[\overline{E}_2] +4p   )+3\mathbb{P}\left[ \overline{E}_2\right]+2H\Lambda p+3p \big) $, it holds that
\begin{align}
& \sum_{t\geq 1}\mathbb{I}\left[ V^*(s_{t})-V^{\pi_{t}}(s_{t})>\epsilon  \right]  \nonumber \\ & \leq             \sum_{t\geq 1}\mathbb{I}\left[   \sum_{s,a}w_{t}(s,a)b^*_{t}(s,a)>\frac{\epsilon}{8}  \right] +\sum_{t\geq 1}\mathbb{I}\left[  \beta_{t}> \frac{\epsilon}{4}    \right]       \nonumber
\\ & \leq O\left(\frac{SAH^3\Lambda^2\ln(\frac{4H}{\epsilon})\iota}{\epsilon^2} \right) +O\left(\frac{SAH^7\Lambda^2 \ln(N_0)}{\epsilon}\right)+ \tilde{O}\left(\frac{S^2A^2 H^{10}\iota }{\epsilon^{3/2}}\right) .\nonumber
\end{align}
The proof is finished by replacing $p$ with $\frac{p}{34 S^2A^2\check{J}^2 \log_{2}(N_0 H) +4H\Lambda}$.

\subsection{Missing Proofs in Appendix~\ref{App:C.2}}\label{App:C.3}

\subsubsection{Proof of Proposition \ref{pro2}} \label{subsubsec_proof_pro2}
\textbf{Proposition \ref{pro2} (restated).}\emph{ 
	With probability $\left(1-SA\left(4\check{J}(2\log_2(N_0 H)  +1 )  + \bar{J}\right) p \right)$, it holds that $Q_{t}(s,a)\geq Q^*(s,a)$ and $Q_{t+1}(s,a)\leq Q_{t}(s,a)$ for any $t\geq 1$ and $(s,a)\in \mathcal{S}\times \mathcal{A}$ .}
The rest of this subsection is devoted to the proof of Proposition \ref{pro2}.

Let $(s,a,j)$ be fixed.
Let $\udl{\mu}^{\reff}$ , $\udl{\check{\mu} }$, $\udl{\sigma}^{\reff}$, $\udl{\check{\sigma}}$ and $\udl{\check{b}}$ be the values of $\mu^{\reff}$, $\check{\mu}$, $\sigma^{\reff}$, $\check{\sigma}$ and  $\check{b}$ in \eqref{equpdate_alg2_1} in the $j$-th type-\uppercase\expandafter{\romannumeral1} update.
Define $\check{l}_{i}$ to be the time when the $i$-th visit in the $j$-th type-\uppercase\expandafter{\romannumeral1} stage of $(s,a)$ occurs and $l_{i}$ to be the time the $i$-th visit of $(s,a)$ occurs respectively. Let $\check{n}$ and $n$ be the shorthands of $\check{e}_{j}$ and $\sum_{i=1}^{j}\check{e}_{i}$ respectively.

Define
\begin{align}
   &  \chi^{(j)}_{1}(s,a):=  \frac{1}{ n }\sum_{i=1}^{n}\left(   V^{\reff}_{l_{i}}(s_{l_{i}+1}) -P_{s,a}V^{\reff}_{l_{i}}      \right) ;\nonumber
   \\ & \chi^{(j)}_{2}(s,a): =   \frac{1}{\check{n}}\sum_{i=1}^{\check{n}}\left(    W_{\check{l}_i}(s_{\check{l}_i+1}) -P_{s,a}W_{\check{l}_i}              \right).\nonumber
\end{align}

We consider the events:
\begin{align}
&\check{E}_{1}^{(j)}(s,a):= \left\{  |\chi^{(j)}_{1}(s,a)|  \leq 2\sqrt{2}\sqrt{ \frac{ \udl{\sigma}^{\reff}/n-(\udl{\mu}^{\reff}/n )^2    }{ n } \iota} +\frac{7H\iota^{3/4}}{ n^{3/4}  } +\frac{4H\iota}{n}    \right \} \nonumber
\\ & \text{and}\nonumber
\\ & \check{E}_{2}^{(j)}(s,a): =\left \{  |\chi^{(j)}_{2}(s,a)|
\leq 2\sqrt{2}\sqrt{\frac{ \udl{\check{\sigma}}/\check{n}  -( \udl{\check{\mu}} /\check{n} )^2   }{\check{n}}\iota}       +\frac{7H\iota^{3/4}}{\check{n}^{3/4} } +\frac{4H\iota}{\check{n}}  \right \},\nonumber
\end{align}
where $W_{t} = V_{t}- V^{\reff}_{t}$. If both $\check{E}_{1}^{(j)}(s,a)$ and $\check{E}_{2}^{(j)}(s,a)$ occurs, then we have that
\begin{align}
  &  r(s,a)+ \frac{\udl{\mu }^{\reff} }{n} +\frac{\udl{\check{\mu}}}{\check{n}}+ \udl{\check{b}}  \nonumber
  \\ &  = r(s,a) +P_{s,a}\left(\frac{1}{n}\sum_{i=1}^n V^{\reff}_{l_i}\right) +P_{s,a}\left(  \frac{1}{\check{n}}\sum_{i=1}^{\check{l}_i }  (V_{\check{l}_i}  -V^{\reff}_{\check{l}_i} )\right)+\chi_1^{(j)}(s,a)+ \chi_2^{(j)}(s,a)+ \udl{\check{b}}  \nonumber
  \\ & \geq r_h(s,a)+ P_{s,a}\left(\frac{1}{\check{n}}\sum_{i=1}^{\check{l}_i }  V_{\check{l}_i}  \right)  + \chi_1^{(j)}(s,a)+ \chi_2^{(j)}(s,a)+ \udl{\check{b}}  \label{eq_thm2_3}
  \\ & \geq r_h(s,a)+ P_{s,a} \left(\frac{1}{\check{n}}   \sum_{i=1}^{\check{l}_i }  V_{\check{l}_i}  \right) ,\label{eq_thm2_4}
\end{align}
where Inequality \eqref{eq_thm2_3} holds by the fact $V^{\reff}_{t}$ is non-increasing in $t$ and Inequality \eqref{eq_thm2_4} follows by the definition of $\udl{\check{b}}$.

On the other hand, for the $j'$-th type-\uppercase\expandafter{\romannumeral2} update, we  consider the following same events as in the proof of Proposition \ref{pro1},
\begin{align}
\bar{E}^{(j')}(s,a) = \left\{    \frac{1}{\bar{e}_{j'}}\sum_{i=1}^{\bar{e}_{j'}}V^*(s_{ \bar{l}_{i}+1  })+ \bar{b}^{(j)} \geq P_{s,a}V^*     \right\}.
\end{align}
Assuming $\bar{E}^{(j')}(s,a)$ holds, we then have 
\begin{align}
&r(s,a)+ \frac{\gamma}{\bar{e}_{j'}}\sum_{i=1}^{\bar{e}_{j'}}V_{\bar{l}_i }(s_{ \bar{l}_{i}+1  }) + \bar{b}^{(j)}  \nonumber
\\& \geq r(s,a) + \gamma P_{s,a}V^* +\gamma \left(\frac{1}{\bar{e}_{j'}}\sum_{i=1}^{\bar{e}_{j'}}(V_{\bar{l}_i }(s_{ \bar{l}_{i}+1  })  - V^*(  s_{ \bar{l}_{i}+1  } )) \right).\label{eq_thm2_5}
\end{align}
 Let 
 \begin{align}
 E_{2} = (\cap_{s,a,j}\check{E}_{1}^{(j)}(s,a) )\cap (\cap_{s,a,j}\check{E}_{2}^{(j)}(s,a)) \cap (\cap_{s,a,j'}\bar{E}^{(j')}(s,a) ). \label{eq_def_E2}
 \end{align}
Assuming $E_{2}$ holds, by the update rule \eqref{equpdate_alg2_1} and \eqref{equpdate_alg2_2} and noting that $V_{t}$ is non-increasing , for any $t\geq 2$ and $(s,a)$, it holds either $Q_{t}(s,a) = Q_{t-1}(s,a)$ or
\begin{align}
Q_{t}(s,a)\geq r_{s,a}+ \gamma P_{s,a}V^* +\sum_{t'<t} v_{t'}(V_{t'}-V^*) \nonumber
\end{align}
for some non-negative $S$-dimensional vectors $v_{1},v_{2},\dots,v_{t-1}$. Noting that $Q_{1}(s,a)= \frac{1}{1-\gamma}\geq Q^*(s,a)$ for any $(s,a)$, the conclusion follows easily by induction.

Therefore, it suffices to bound $\mathbb{P}\left[E_2\right]$.
\begin{lemma} For any $(s,a,j)$, $\mathbb{P}\left[ \check{E}_{1}^{(j)}(s,a) \right]\geq 1-2(\log_2(N_0H)+1)p$.
\end{lemma}
\begin{proof} Define $\mathbb{V}(x,y) = xy^2-(xy)^2$ for two vectors with the same dimension.
Noticing that $s_{l_{i}+1}$ is independent of $V^{\reff}_{l_{i}}$ conditioned on $\mathcal{F}_{l_{i}-1}$,  by Lemma \ref{self-norm} with $\epsilon = H$,  we have that with probability $(1-  2\log_{2}(nH  )p  )$, it holds that
\begin{align}
|\chi_{1}^{(j)}(s,a)|  &= \left|  \frac{1}{ n }\sum_{i=1}^{n}\left(   V^{\reff}_{l_{i}}(s_{l_{i}+1}) -P_{s,a}V^{\reff}_{l_{i}}      \right)\right| \nonumber
\\ &  \leq 2\sqrt{2}\sqrt{\frac{  (\sum_{i=1}^n \mathbb{V} (P_{s,a}, V^{\reff}_{l_{i}} )  ) \iota }{  n^2  }} + \frac{\sqrt{2H\iota}}{n}+ \frac{2H\iota}{n} \nonumber
\\ & \leq   2\sqrt{2}\sqrt{\frac{  (\sum_{i=1}^n \mathbb{V} (P_{s,a}, V^{\reff}_{l_{i}} )  ) \iota }{  n^2  }}  +\frac{4H\iota}{n}.\label{eq_thm2_6}
\end{align}
	
 By definition of $\udl{\sigma}^{\reff}$ and $\udl{\mu}^{\reff}$, we have that 
\begin{align}
\sum_{i=1}^{n}\mathbb{V}(P_{s,a},V^{\reff}_{l_i})  &= \sum_{i=1}^n \left( P_{s,a}(V^{\reff}_{l_i})^2-  (P_{s,a} V^{\reff}_{l_i})^2  \right)  \nonumber
\\ & = \sum_{i=1}^n ( V_{l_i}^{\reff}( s_{l_i+1}   )   )^2 -\frac{1}{n}\left(       \sum_{i=1}^n V_{l_i}^{\reff}( s_{l_i+1  } )       \right)^2+\chi_3  +\chi_{4}+ \chi_5\nonumber
\\ & = \udl{\sigma}^{\reff} -\frac{1}{n}(\udl{\mu}^{\reff})^2 +\chi_3+\chi_4+\chi_5,\nonumber
\end{align}
	where 
	\begin{align}
 &\chi_3  : =	\sum_{i=1}^n \left(   P_{s,a}(V^{\reff}_{l_i})^2  -( V^{\reff}_{l_i} (s_{l_i+1}) )^2  \right) \nonumber
 \\ & \chi_4: = \frac{1}{n} \left(         \sum_{i=1}^n V^{\reff}_{l_i} (s_{l_i+1})    \right)^2 -\frac{1}{n} \left(      \sum_{i=1}^n P_{s,a}V^{\reff}_{l_i}  \right) \nonumber
 \\ & \chi_5 = \frac{1}{n} \left(       \sum_{i=1}^n P_{s,a}V_{l_i}^{\reff}    \right)^2 -\sum_{i=1}^n (P_{s,a}V^{\reff}_{l_i})^2.\nonumber
	\end{align}

By Azuma's inequality, we have that
\begin{align}
&\mathbb{P}\left[|\chi_3|> H^2\sqrt{2n\iota} \right]\leq p \nonumber
\end{align}
and
\begin{align}
 & \mathbb{P}\left[  |\chi_{4}| > 2H^2\sqrt{2n\iota}       \right]\leq \mathbb{P}\left[  2H \cdot|\sum_{i=1}^n \left(  V^{\reff}_{l_i}(s_{l_i+1})-P_{s,a}V^{\reff}_{l_i}  \right)  |  > 2H^2\sqrt{2n\iota}\right] \leq p.\nonumber
\end{align}
On the other hand, by Cauchy-Schwartz inequality, we have $\chi_5\leq 0$. It then follow that
\begin{align}
\mathbb{P}\left[   \sum_{i=1}^n \mathbb{V}(P_{s,a},V^{\reff}_{l_i})> \udl{\sigma}^{\reff}-\frac{1}{n}( \udl{\mu}^{\reff} )^2+ 5H^2\sqrt{n\iota}    \right]\leq 2p.\label{eq_thm2_7}
\end{align}

Combining \eqref{eq_thm2_6} and \eqref{eq_thm2_7}, we have that
\begin{align}
\mathbb{P}\left[\check{E}^{(j)}_{1}(s,a)  \right] &\geq 1- \mathbb{P}\left[    |\chi_1^{(j)}(s,a)|>2\sqrt{2}\sqrt{\frac{  (\sum_{i=1}^n \mathbb{V} (P_{s,a}, V^{\reff}_{l_{i}} )  ) \iota }{  n^2  }}  +\frac{4H\iota}{n}\right]\nonumber
 \\& \quad -\mathbb{P}\left[   \sum_{i=1}^n \mathbb{V}(P_{s,a},V^{\reff}_{l_i})> \udl{\sigma}^{\reff}-\frac{1}{n}( \udl{\mu}^{\reff} )^2+ 5H^2\sqrt{n\iota}    \right] \nonumber
\\ & \geq 1-2(\log_2(nH)+1)p\nonumber
\\ & \geq 1-2(\log_2(N_0H)+1)p.\nonumber \end{align}
\end{proof}

Following similar arguments as above, we can prove that $\mathbb{P}\left[  \check{E}_2^{(j)}(s,a) \right]\geq 1-2(\log_2(N_0H)+1)p$ for any $1\leq j\leq \check{J}$. At last, by Azuma's inequality, $\mathbb{P}\left[ \bar{E}^{(j')}(s,a) \right]\geq 1-p$ for any $j'$ and $(s,a)$.  Via a union bound over $1\leq j\leq \check{ J}$ and $1\leq j'\leq \bar{J}$, we obtain that 
\begin{align}
    \mathbb{P}\left[E_2\right]\geq 1 -4SA\check{J}(\log_{2}(N_0 H)+1)p -SA\bar{J}p.\label{eq_thm2_0.5}
\end{align}
The proof is completed.

\subsubsection{Proof of Lemma \ref{lemma:bd_bellman}}\label{sec:proof-bellman}
\medskip
\noindent \textbf{Lemma~\ref{lemma:bd_bellman} (restated).}{ Define $l_{i}(s,a)$ to be the time the $i$-th visit of $(s,a)$ occurs and $\bar{N}_{t}(s,a)$ to be the visit count of $(s,a)$ before the current stage of $(s,a)$.
Conditioned on $E_{2}$,  it holds that
\begin{align}
 Q_{t}(s,a)- r(s,a)-P_{s,a}V_{t} \leq    P_{s,a}(V_{\udl{\rho}_{t}}(s,a) -V_{t})+\frac{1}{1-\gamma} P_{s,a}\left(\frac{1}{n}\sum_{i=1}^n \lambda_{l_i}  \right).\nonumber
\end{align}
for any $t\geq 1$ and any $(s,a)\in \mathcal{S}\times\mathcal{A}$.
}

Let $(s,a,j)$ be fixed. We use the same notations as that of Section \ref{subsubsec_proof_pro2}. For any $t$ in the $j+1$-th type-\uppercase\expandafter{\romannumeral1} stage, by the arguments to derive \eqref{eq_thm2_4}, we have that
\begin{align}
    Q_{t}(s,a)& = r(s,a) +\frac{\udl{\mu^{\reff}}} {n}+ \frac{\udl{\check{u}}}{\check{n}}+ \udl{\check{b}} \nonumber
    \\ &\leq  r(s,a)  +P_{s,a}\left(\frac{1}{n}\sum_{i=1}^n V^{\reff}_{l_i}\right) +P_{s,a}\left(  \frac{1}{\check{n}}\sum_{i=1}^{\check{l}_i }  (V_{\check{l}_i}  -V^{\reff}_{\check{l}_i} )\right)  \nonumber
    \\ & \leq r(s,a)+P_{s,a}V_{t} + P_{s,a}(V_{\udl{\rho}_{t}}(s,a) -V_{t})+ P_{s,a}\left(\frac{1}{n}\sum_{i=1}^n V^{\reff}_{l_i} -V^{\Reff}  \right)\nonumber
    \\ & \leq   r(s,a)+P_{s,a}V_{t} + P_{s,a}(V_{\udl{\rho}_{t}}(s,a) -V_{t})+\frac{1}{1-\gamma} P_{s,a}\left(\frac{1}{n}\sum_{i=1}^n \lambda_{l_i}  \right)  .\label{eq:pf-bellman1}
\end{align}
The proof is completed.

\subsubsection{Proof of Lemma \ref{boundnew}} \label{sec:proof-boundnew}
\medskip
\noindent \textbf{Lemma~\ref{boundnew} (restated).} {\it Define $\Lambda =\left\lceil\log_{2}(\frac{256H^4}{\epsilon^2})\right\rceil $. With probability $(1-2H\Lambda p)$, it holds that
	\begin{align}
	\sum_{t\geq 1}\mathbb{I}\left[   \sum_{s,a}w_{t}(s,a)b^*_{t}(s,a)>\frac{\epsilon}{8}  \right]\leq O\left(    \frac{SAH^3\Lambda^3\iota}{\epsilon^2}  +\frac{SAH^4B\Lambda^2 \ln(N_0) }{\epsilon}    \right) .\nonumber
	\end{align}
}

The rest of this subsection is devoted to the proof of Lemma~\ref{boundnew}.

Define $\mathcal{S}_{t,0} :=\{(s,a)| n_{t}(s,a) <\iota \}$, $\mathcal{S}_{t,u} := \{ (s,a)|  2^{u-1}\iota\leq n_{t}(s,a)<2^{u}\iota  \}$ for $u=1,2,\dots,\Lambda =\lceil \log_{2}(\frac{256H^4}{\epsilon^2})\rceil$ and $\overline{\mathcal{S}}_t := \{ (s,a)| n_{t}(s,a)> \frac{H^4}{\epsilon^2}\}$ . Furthermore, we define 
\begin{align}
\beta^*_{t,u}:= \sum_{(s,a)\in \mathcal{S}_{t,u}}w_{t}(s,a)b_{t}^*(s,a) \nonumber
\end{align}
 and 
 \begin{align}
 \beta^*_{t} : = \sum_{u}\beta^*_{t,u} =\sum_{s,a}w_{t}(s,a) b_{t}^*(s.a) . \nonumber
 \end{align}
 
By the definition of $b_{t}^*(s,a)$, we obtain that for $1\leq u\leq \Lambda$, 
\begin{align}
\beta^*_{t,i} &= \sum_{(s,a)\in \mathcal{S}_{t,u}}w_{t}(s,a)b_{t}^*(s,a) \nonumber
\\ & \leq 2\sqrt{2\iota}\sum_{(s,a) \in \mathcal{S}_{t,u}}w_{t}(s,a)\sqrt{\frac{\mathbb{V}(P_{s,a},V^*) }{n_{t}(s,a)}}\nonumber
\\ & \leq  2\sqrt{\frac{2}{2^{u-1}}}\sum_{(s,a)\in \mathcal{S}_{t,u}} w_{t}(s,a)\sqrt{\mathbb{V}(P_{s,a},V^*)}\nonumber
\\ & \leq 2\sqrt{\frac{2}{2^{u-1}}} \cdot \sqrt{ \sum_{(s,a)\in \mathcal{S}_{t,u}  } w_{t}(s,a) } \cdot \sqrt{\sum_{ (s,a)\in \mathcal{S}_{t,u}   }w_{t}(s,a)\mathbb{V}(P_{s,a},V^*) },\label{eq_thm2_8}
\end{align}
and for $0\leq u \leq \Lambda$,
\begin{align}
\beta^*_{t,u} \leq  \frac{1}{1-\gamma}\sum_{(s,a)\in  \mathcal{S}_{t,u}}w_{t}(s,a) .\nonumber
\end{align}
 Define $w_{t,u}: = \sum_{(s,a)\in \mathcal{S}_{t,u}} w_{t}(s,a)$ and $\nu_{t}  = \sum_{s,a} w_{t}(s,a)\mathbb{V}(P_{s,a},V^*) $.
 Note that 
 \begin{align}
 \nu_{t} &= \sum_{s,a}w_{t}(s,a) (P_{s,a}(V^*)^2-  (P_{s,a}V^*)^2 )  \nonumber
 \\ & = \sum_{s,a}w_{t}(s,a)P_{s,a}(V^*)^2 -\frac{1}{\gamma^2} \sum_{s,a}w_{t}(s,a)(Q^*(s,a)-r(s,a) )^2  \nonumber
 \\ & \leq \sum_{s,a}w_{t}(s,a)P_{s,a}(V^*)^2 - \sum_{s,a}w_{t}(s,a)(Q^*(s,a)-r(s,a) )^2 \nonumber
 \\ & \leq \sum_{s,a}w_{t}(s,a)(P_{s,a}(V^*)^2 - (Q^*(s,a))^2)  + \frac{2H}{1-\gamma} \nonumber
 \\ & =\sum_{s,a}w_{t}(s,a) (P_{s,a}(V^*)^2 -(V^*(s))^2    ) + \sum_{s,a}w_{t}(s,a) ((V^*(s))^2-(Q^*(s,a))^2)  )+ \frac{2H}{1-\gamma}  \nonumber
 \\ & \leq \sum_{s,a}w_{t}(s,a) (P_{s,a}(V^*)^2 -(V^*(s))^2    ) + \frac{2}{1-\gamma}\sum_{s,a}w_{t}(s,a) (V^*(s)-Q^*(s,a)) + \frac{2H}{1-\gamma}  \nonumber
 \\ & \leq \frac{1}{(1-\gamma)^2}+   \frac{2}{1-\gamma}\sum_{s,a}w_{t}(s,a) (V^*(s)-Q^*(s,a)) +\frac{2H}{1-\gamma} \label{eq_thm2_9}
 \\ & \leq          \frac{1}{(1-\gamma)^2} +           \frac{2}{(1-\gamma)}(V^*(s_{t})-V^{\pi_{t}}(s_{t}) )+\frac{2H}{1-\gamma} \label{eq_thm2_10}
 \\ & \leq 5H^2.  \label{eq_thm2_11}
 \end{align}
Here Inequality   \eqref{eq_thm2_9} holds by the fact that 
\begin{align}
\sum_{s,a}w_{t}(s,a)(P_{s,a}-\textbf{1}_{s})(V^*)^2 &=   \sum_{s,a}\big(\mathbb{I}\left[a = \pi_{t}(s) \right] \sum_{i=0}^{H-1} \mathbf{1}_{s_t}^{\top} (\gamma P_{\pi_t})^i \mathbf{1}_{s}\big)\cdot  (P_{s,a}-\textbf{1}_{s})(V^*)^2 \nonumber
\\ & = \sum_{s,a} \mathbb{I}\left[a = \pi_{t}(s) \right]  \big(\mathbf{1}_{s_t}^{\top} (\gamma P_{\pi_t})^H \mathbf{1}_{s} -\mathbb{I}\left[ s=s_{t} \right]\big)  (V^*(s))^2\nonumber
\\ & \leq \frac{1}{(1-\gamma)^2},\nonumber
\end{align}
and Inequality \eqref{eq_thm2_10} is due to the bound on the following  telescoping sum,
 \begin{align}
 V^{*}(s_{t})-V^{\pi_{t}}(s_{t}) & = \sum_{s,a}\big(\mathbb{I}\left[a = \pi_{t}(s) \right] \sum_{i=0}^{\infty} \mathbf{1}_{s_t}^{\top} (\gamma P_{\pi_t})^i \mathbf{1}_{s}\big) \cdot (V^*(s)-Q^*(s,a)) \nonumber
 \\ &  \geq \sum_{s,a}w_{t}(s,a)(V^*(s)-Q^*(s,a)).\nonumber
 \end{align}
 
Combining \eqref{eq_thm2_11} with the fact that  $\sum_{(s,a)\in \overline{\mathcal{S}}_t} w_{t}(s,a)b^*_{t}(s,a)\leq \frac{\epsilon}{16}$,  we obtain that ,  if $\beta^*_{t}> \frac{\epsilon}{8} $, there exists $u$ such that  $\beta^*_{t,u} > \frac{\epsilon}{16\Lambda}$, which implies that $w_{t,u}>\max\{ \frac{1}{10240} \cdot \frac{2^{u-1}\epsilon^2}{H^2 \Lambda^2 } ,\frac{\epsilon(1-\gamma	)}{16\Lambda} \}$.

We will  bound the number of steps in which   there exists $u$ satisfying  $w_{t,u}>\max\{ \frac{1}{10240} \cdot \frac{2^{u-1}\epsilon^2}{H^2 \Lambda^2 } ,\frac{\epsilon(1-\gamma	)}{16\Lambda} \}$ by following lemma.

\begin{lemma}\label{lemma_newbound_1}
For any $k\in \{1,2,\dots,H\}$ and $u\in \{ 1,2,\dots,\Lambda\}$,  with probability $1-p$, 
\begin{align}
\sum_{t\geq 0} \mathbb{I}\left[    w_{tH+k,u}> \frac{1}{10240} \cdot \frac{2^{u-1}\epsilon^2}{H^2 \Lambda^2 }     \right]\leq O \left(      \frac{SAB H^4\Lambda^2 \ln(N_0)}{2^{u-1}\iota\epsilon^2} +\frac{SAH^2\Lambda^2\iota}{\epsilon^2}      \right)  . \label{eq_thm2_12}
\end{align}
Moreover, for any $u\geq 0$, with probability $1-p$,
\begin{align}
\sum_{t\geq 0}\mathbb{I} \left[ w_{tH+k,u}>\frac{\epsilon(1-\gamma)}{16\Lambda}    \right] \leq  O\left(\frac{H\Lambda }{\epsilon}(SAH^2B\ln(N_0)+SAH +2^{u+2}SA\iota ) \right)   .\label{eq_thm2_15}
\end{align}
\end{lemma}
\begin{proof}
Define 
\begin{align}
\tilde{U}_{t,u} =\mathbb{I}\left[  \exists (s,a), i\in \{1,2,\dots,H-1\}, \text{ such that } \mathcal{S}_{t+i,u}\neq \mathcal{S}_{t,u} \text{ or }  Q_{t+i}(s,a)\neq Q_{t}(s,a)        \right], \nonumber
\end{align}
 and 
 \begin{align}\hat{w}_{t}(s,a) =(1-\tilde{U}_{t,u})\sum_{i=0}^{H-1} \mathbb{I}\left[ (s_{t+i},a_{t+i}) \in \mathcal{S}_{t+i,u} \right]+H \tilde{U}_{t,u}. \nonumber
 \end{align}
 Note that $\hat{w}_{tH+k}$ is measurable with respect to $\mathcal{F}_{t}^{k} = \mathcal{F}_{(t+1)H+k-1}$ and $\mathbb{E}\left[\hat{w}_{tH+k} |\mathcal{F}^{t-1}_{k} \right]\geq w_{tH+k}$, we then have that by Lemma \ref{martingale_gap},
 \begin{align}
 &\mathbb{P}\Big[ \sum_{t\geq 0}w_{tH+k}> 8SAH^2B\ln(N_0)+8SAH +2^{u+2}SA\iota  , \nonumber \\
 & \qquad \qquad \qquad  \sum_{t\geq 0}\hat{w}_{tH+k}\leq 2SAH^2B\ln(N_0)+2SAH +2^{u}SA\iota  \Big]\leq p.\label{eq_thm2_13}
 \end{align}
 On the other hand, we have that
 \begin{align}
 \sum_{t\geq 0}\hat{w}_{tH+k} &\leq H\sum_{t\geq 0}  \hat{U}_{tH+k}+\sum_{t\geq 1}\mathbb{I}\left[(s_{t},a_{t})\in \mathcal{S}_{t,u}  \right] \nonumber
 \\ & \leq 2SAH^2B\ln(N_0)+2SAH  +\sum_{t\geq 1}\mathbb{I}\left[(s_{t},a_{t})\in \mathcal{S}_{t,u}  \right]  \label{eq_thm2_14}
  \\ & \leq  2SAH^2B\ln(N_0)+2SAH +2^{u}SA\iota , \label{eq_thm2_14.5}
 \end{align}
 where Inequality \eqref{eq_thm2_14} is because $\mathcal{S}_{t,u}$ changes at most $2SA$ times in $t$,  and Inequality \eqref{eq_thm2_14.5} is by the fact that $ 2^{u-1}\iota \leq n_{t}(s,a)< 2^u\iota$ implies that $2^{u}\iota \leq N_{t}(s,a) <2^{u+1}\iota$. It then follows that
 \begin{align}
 \mathbb{P}\left[    \sum_{t\geq 0}w_{tH+k}> 8SAH^2B\ln(N_0)+8SAH +2^{u+2}SA\iota  \right]\leq p,\nonumber
 \end{align}
which means
 \begin{align}
 \mathbb{P}\left[ \sum_{t\geq 0}   \mathbb{I}\left[    w_{tH+k,u}> \frac{1}{10240} \cdot \frac{2^{u-1}\epsilon^2}{H^2 \Lambda^2 }     \right]> 10240\left(      \frac{16SAB H^4\Lambda^2 \ln(N_0)}{2^{u-1}\iota\epsilon^2} +\frac{8SAH^2\Lambda^2\iota}{\epsilon^2}      \right)   \right]\leq p \nonumber
 \end{align}
and 
\begin{align}
\mathbb{P}\left[   \sum_{t\geq 0} \mathbb{I}\left[   w_{tH+k,u}> \frac{\epsilon(1-\gamma)}{16\Lambda}    \right]>   \frac{16H\Lambda }{\epsilon}(8SAH^2B\ln(N_0)+8SAH +2^{u+2}SA\iota )    \right]\leq p.\nonumber
\end{align}
The proof is completed.
\end{proof}

For $u$ such that $2^{u}\leq \frac{BH^2\ln(N_0)}{\iota}$ or $u=0$, we plug $u$ and $k=1,2,\dots,H$ into \eqref{eq_thm2_15} and obtain that with probability $1-Hp$,
\begin{align}
\sum_{t\geq 1}\mathbb{I}\left[ w_{t,u}>\frac{\epsilon(1-\gamma)}{16\Lambda}     \right]\leq O\left(       \frac{SAH^4B\Lambda \ln(N_0) }{\epsilon}            \right).\label{eq_thm2_16}
\end{align}
For $u$ such that $2^{u}> \frac{BH^2\ln(N_0)}{\iota}$, we plug $u$ and $k=1,2,\dots,H$ into \eqref{eq_thm2_12} and obtain that with probability $1-Hp$,
\begin{align}
\sum_{t\geq 1}\mathbb{I}\left[  w_{t,u}> \frac{1}{10240} \cdot \frac{2^{u-1}\epsilon^2}{H^2 \Lambda^2 }  \right]\leq O\left(  \frac{SAH^3\Lambda^2\iota}{\epsilon^2} \right).\label{eq_thm2_17}
\end{align}

Via a union bound over $u$, we have that with probability $1-2H\Lambda p$, it holds that 
\begin{align}
     \sum_{t\geq 1}\mathbb{I}\left[   \beta_{t}^* >\frac{\epsilon}{8}           \right]     & \leq \sum_{t\geq 1}\mathbb{I}\left[\exists u, w_{t,u}>  \max\{ \frac{1}{10240} \cdot \frac{2^{u-1}\epsilon^2}{H^2 \Lambda^2 } ,\frac{\epsilon(1-\gamma	)}{8\Lambda} \}         \text{~~and~~} w_{t,0}>\frac{\epsilon(1-\gamma)}{8\Lambda}       \right]\nonumber
\\& \leq O\left(    \frac{SAH^3\Lambda^3\iota}{\epsilon^2}  +\frac{SAH^4B\Lambda^2 \ln(N_0) }{\epsilon}    \right).\label{eq_thm2_18}
\end{align}

\subsubsection{Proof of Lemma~\ref{lemma_bd_b_1}} \label{app:proof-lemma_bd_b_1}

\medskip
\noindent \textbf{Lemma~\ref{lemma_bd_b_1} (restated).}
{\it
	 With probability $1 - SA\check{J} (2\mathbb{P}[\overline{E
	 }_2] +4p   ) $,  it holds that
	\begin{align}
&	\sum_{t\geq 1}\mathrm{clip}( \check{b}_{t}(s_{t},a_{t})-b^*_{t}(s_{t},a_{t}) ,\frac{\epsilon}{16H} )
	\\& \leq    O\left(   \frac{SAH^2\iota}{\epsilon} \right) +\tilde{O}\left(   \frac{ S^{3/2}A^{3/2} H^{17/4} \iota  }{\epsilon^{1/2}} +  \frac{SAH^{59/12} \iota}{ \epsilon^{1/3}} +\frac{   S^{5/4}A^{5/4}H^{3}\iota     }{\epsilon^{1/4}}   +S^2A^2H^{9}\iota \right). \nonumber
	\end{align}
}

The rest of this subsection is devoted to the proof of Lemma~\ref{lemma_bd_b_1}.

Let $s,a,j$ be fixed. We follow the notations in Appendix~\ref{subsubsec_proof_pro2}.  For $t$ in the $(j+1)$-th type-\uppercase\expandafter{\romannumeral1} stage of $(s,a)$, recalling the definition 
\begin{align}
\check{b}_{t}(s_{t},a_{t}) &= \min\{2\sqrt{2} \left(\sqrt{\frac{  \udl{ \check{\sigma}}/ \check{n}  - ( \udl{\check{\mu}} /\check{n} )^2     }{ \check{n}  }\iota}  +\sqrt{     \frac{\udl{\sigma}^{\reff}/n     -(\udl{\mu}^{\reff}  /n)^2  }{n}      \iota} \right)  \nonumber
\\ & \quad \quad \quad +   7\left(    \frac{H\iota^{3/4}}{n^{3/4}}  +    \frac{H\iota^{3/4}}{\check{n}^{3/4}}   \right)  +  5 \left(   \frac{H\iota}{n}  +\frac{H\iota}{\check{n}}      \right) ,\;  \frac{1}{1-\gamma} \},\nonumber
\end{align}
we have that
\begin{align}
&\clip(\check{b}_{t}(s_{t},a_{t})-b^*_{t}(s_{t},a_{t}),\frac{\epsilon}{16H} )\nonumber \\& \leq \underbrace{4\clip( 2\sqrt{2}\left(\sqrt{     \frac{\udl{\sigma}^{\reff}/n     -(\udl{\mu}^{\reff}  /n)^2  }{n}      \iota  } -\sqrt{\frac{\mathbb{V}(P_{s,a},V^*) }{n}\iota}\right) ,\frac{\epsilon}{64H})}_{ \textcircled{1}} + \underbrace{4\clip(  2\sqrt{2} \sqrt{\frac{  \udl{ \check{\sigma}}/ \check{n}  - ( \udl{\check{\mu}} /\check{n} )^2     }{ \check{n}  }\iota}  ,\frac{\epsilon}{64H} )}_{\textcircled{2} } \nonumber
\\ & \quad + \underbrace{4\clip(    7\left(    \frac{H\iota^{3/4}}{n^{3/4}}  +    \frac{H\iota^{3/4}}{\check{n}^{3/4}}   \right) ,\frac{\epsilon}{64H}  ) }_{      \textcircled{3} }+\underbrace{4\clip(  5\left(   \frac{H\iota}{n}  +\frac{H\iota}{\check{n}}      \right),\frac{\epsilon}{64H}) }_{  \textcircled{4}  }  ,\label{eq_thm2_19}
\end{align}
and the trivial bound 
\begin{align}
\clip(\check{b}_{t}(s_t,a_t)- b^*_{t}(s_t,a_t),\frac{\epsilon}{16H} )\leq \frac{1}{1-\gamma} .\label{eq_thm2_20}
\end{align}
Here, \eqref{eq_thm2_19} is because $\clip(a+b,2\epsilon)\leq 2\clip(a,\epsilon) +2\clip(b,\epsilon)$ for any non-negative  $a,b,\epsilon$.

Let $V^{\reff}_{t}$ be the value of $V^{\reff}$ immediately before the beginning of the $t$-th step and $V^{\Reff} := \lim_{t\to \infty}V^{\reff}_{t}$ (by the update rule of Algorithm \ref{alg2}, this limit exists). Recall that $\lambda_{t}$ is defined as  the vector such that $\lambda_{t}(s) = \mathbb{I}\left[ \sum_{a}N_{t}(s,a)< N_1 \right]$.
By Lemma \ref{lemma_lvf1} with $\epsilon_{1}=\omega := \frac{1}{\sqrt{B}}$ (assuming $\epsilon \leq \frac{1}{\sqrt{B}}$), we have that 
\begin{align}
\mathbb{P}\left[ \forall t\geq 1,  V^{\reff}_{t}(s_{t})-V^*(s_t)\leq   H\lambda_{t}(s_{t})+ \omega
\right]\geq \mathbb{P}\left[E_2\right]  .\label{eq_thm2_21}
\end{align}


We will deal with the four terms in \textbf{RHS} of \eqref{eq_thm2_19} separately.

\paragraph{The \textcircled{1} term} To handle this term, we introduce a lemma to bound $\frac{\udl{ \sigma}^{\reff} }{n} - (\frac{\udl{\mu}^{\reff} }{n })^2 - \mathbb{V}(P_{s,a},V^*)$.
\begin{lemma}\label{lemma_last_1}
With probability  $1-(\mathbb{P}[\overline{E}_2] +4p)$, it holds that 
\begin{align}
&\frac{\udl{ \sigma}^{\reff} }{n} - (\frac{\udl{\mu}^{\reff} }{n })^2 - \mathbb{V}(P_{s,a},V^*) \leq 9\sqrt{2}H^3 \sqrt{\frac{\iota}{n}} +\frac{1}{n}\left(   2H^2SA(\check{J}+ \bar{J}) + 10H^2SN_1 \right )+4H\omega.\nonumber
\end{align}
\end{lemma}
\begin{proof}
Note that 
\begin{align}
\frac{\udl{ \sigma}^{\reff} }{n} - (\frac{\udl{\mu}^{\reff} }{n })^2 - \mathbb{V}(P_{s,a},V^*)= \frac{1}{n}(\chi_6+\chi_7+\chi_8+\chi_{9}), \label{eq_thm2_22}
\end{align}
where
\begin{align}
& \chi_{6}: = \sum_{i=1}^n  \left(  (  V^{\reff}_{l_i}(s_{l_i+1}) )^2 -P_{s,a}(V^{\reff}_{l_i} )^2     \right) ,\nonumber
\\ & \chi_7 : = \frac{1}{n}\left(     \sum_{i=1}^{n}P_{s,a}V^{\reff}_{l_i}      \right)^2-\frac{1}{n}\left(     \sum_{i=1}^n V_{l_i}^{\reff}(s_{l_i+1})  \right)^2  \nonumber
\\ & \chi_8 := \sum_{i=1}^{n}(P_{s,a}V^{\reff}_{l_i})^2 -\frac{1}{n}\left(   \sum_{i=1}^n P_{s,a}V^{\reff}_{l_i}  \right)^2 ,\nonumber
\\ & \chi_9 :=\sum_{i=1}^n \mathbb{V}(P_{s,a},V^{\reff}_{l_i})-n\mathbb{V}(P_{s,a},V^*).\nonumber
\end{align}

According to Azuma's inequality, with probability $(1-2p)$ it holds that
\begin{align}
& |\chi_{6}|\leq H^2\sqrt{2n\iota} ,\label{eq_thm2_23}
\\ & |\chi_7|\leq 2H \left| \sum_{i=1}^{n} \left(V^{\reff}_{l_i}(s_{l_i+1}) -P_{s,a}V^{\reff}_{l_i} \right)  \right|\leq 2H^2\sqrt{2n\iota}.\label{eq_thm2_24}
\end{align}

On the other hand, by direct computation, we have that
\begin{align}
\chi_{8} & = \sum_{i=1}^{n}(P_{s,a}V^{\reff}_{l_i})^2- \frac{1}{n}\left(   \sum_{i=1}^n P_{s,a}V^{\reff}_{l_i}  \right)^2 \nonumber
\\ & \leq  \sum_{i=1}^{n}(P_{s,a}V^{\reff}_{l_i})^2- \frac{1}{n}\left(   \sum_{i=1}^n P_{s,a}V^{\Reff}  \right)^2  \label{eq_thm2_25}
\\ &  =\sum_{i=1}^{n} \left( (P_{s,a}V^{\reff}_{l_i} )^2  - (P_{s,a}V^{\Reff})^2    \right)\nonumber
\\ & \leq 2H^2 \sum_{i=1}^{n}P_{s,a}\lambda_{l_i}      \label{eq_thm2_26} 
\\ &  = 2H^2 \left(    \sum_{i=1}^n \lambda_{l_i}(s_{l_i+1})  +\sum_{i=1}^{n}(P_{s,a}-\textbf{1}_{s_{l_i+1}})\lambda_{l_i}  \right) \nonumber
\\ & = 2H^2 \sum_{i=1}^n (\lambda_{l_i}(s_{l_i+1})- \lambda_{l_i+1} (s_{l_i+1}) )   +2H^2\sum_{i=1}^n \lambda_{l_i+1}(s_{l_i+1})  +2H^2\sum_{i=1}^{n}(P_{s,a}-\textbf{1}_{s_{l_i+1}})\lambda_{l_i}  \nonumber
\\ & \leq  2H^2SA(\check{J}+ \bar{J}) +2H^2SN_1 +2H^2\sum_{i=1}^{n}(P_{s,a}-\textbf{1}_{s_{l_i+1}})\lambda_{l_i} ,\label{eq_thm2_27}
\end{align}
where Inequality \eqref{eq_thm2_25} is by the fact that $V^{\reff}_{t}\geq V^{\Reff}$ for any $t\geq 1$, Inequality \eqref{eq_thm2_26} is by the definition of $\lambda_{t}$
and Inequality \eqref{eq_thm2_27} holds because  $\lambda_{t}\neq \lambda_{t+1}$ implies an update occurs at the $t$-th step and $\sum_{t\geq 1}\lambda_{t}(s_t)\leq SN_1$.
Therefore,  by Azuma's inequality it holds that
\begin{align}
    \mathbb{P}\left[ \chi_8 > 2H^2SA(\check{J}+ \bar{J})+ 2H^2SN_1+2H^3\sqrt{2n\iota}\right] 
    & \leq  \mathbb{P}\left[  2H^2\sum_{i=1}^{n}(P_{s,a}-\textbf{1}_{s_{l_i+1}})\lambda_{l_i}>2H^3\sqrt{2n\iota}  \right]    \leq  p.\label{eq_thm2_27.5}
\end{align}

At last,  the term $\chi_{9}$ could be bounded by
\begin{align}
\chi_{9}  &= \sum_{i=1}^n \mathbb{V}(P_{s,a},V^{\reff}_{l_i})-n\mathbb{V}(P_{s,a},V^*)  \nonumber
\\ & \leq \frac{4H}{n}\sum_{i=1}^{n}P_{s,a}(V^{\reff}_{l_i}-V^*) \nonumber
\\ & = 4H\sum_{i=1}^{n}(    V^{\reff}_{l_i}(s_{l_i+1}) -  V^{\reff}_{l_i+1}(s_{l_i+1}) + V^{\reff}_{l_i+1}(s_{l_i+1})  -V^{*}(s_{l_i+1})  )+4H\sum_{i=1}^n  (P_{s,a}-  \textbf{1}_{s_{l_i+1}} )(V^{\reff}_{l_i}-V^{*} ) \nonumber
\\ & \leq 4H^2S + 4H\sum_{i=1}^{n} (    V^{\reff}_{l_i+1}(s_{l_i+1})  -V^{*}(s_{l_i+1}) )  +4H\sum_{i=1}^n  (P_{s,a}-  \textbf{1}_{s_{l_i+1}} )(V^{\reff}_{l_i}-V^{*} ) , \label{eq_thm2_28}
\end{align} 
where Inequality \eqref{eq_thm2_28} is by  the fact that the number of updates of $V^{\reff}$ is at most $S$.
Similarly, we have that 
\begin{align}
    & \mathbb{P}\left[ \chi_9>   4H^2S  +4H^2SN_1 +4Hn\omega   +4H^2\sqrt{2n\iota}   \right]  \nonumber
    \\ & \leq \mathbb{P}\left[  \sum_{i=1}^{n} (    V^{\reff}_{l_i+1}(s_{l_i+1})  -V^{*}(s_{l_i+1}) ) > \sum_{i=1}^{n}(H \lambda_{l_i+1  }(s_{l_i+1}) +\omega ) \right] \nonumber\\
    &\qquad\qquad\qquad \qquad\qquad\qquad \qquad\qquad\qquad + \mathbb{P}\left[ \sum_{i=1}^n  (P_{s,a}-  \textbf{1}_{s_{l_i+1}} )(V^{\reff}_{l_i}-V^{*} )>  H\sqrt{2n\iota}  \right] \nonumber
    \\ & \leq \mathbb{P}[\overline{E}_2]+p,\label{eq_thm2_30}
\end{align}
where \eqref{eq_thm2_30} holds by \eqref{eq_thm2_21}.

Combining \eqref{eq_thm2_22},  \eqref{eq_thm2_23}, \eqref{eq_thm2_24}, \eqref{eq_thm2_27.5} and \eqref{eq_thm2_30}, with probability $1-(\mathbb{P}[\overline{E}_2] +5p)$ it holds that 
\begin{align}
&\frac{\udl{ \sigma}^{\reff} }{n} - (\frac{\udl{\mu}^{\reff} }{n })^2 - \mathbb{V}(P_{s,a},V^*)\nonumber
\\&\leq \frac{1}{n}\left(3H^2\sqrt{2n\iota} +2H^2SA(\check{J}+ \bar{J})+ 2H^2SN_1+2H^3\sqrt{2n\iota} +4H\left(  S+ SN_1   \right) +4H^2\sqrt{2n\iota} \right)+4Hw \nonumber
\\& \leq 9\sqrt{2}H^3 \sqrt{\frac{\iota}{n}} +\frac{1}{n}\left(   2H^2SA(\check{J}+ \bar{J}) + 10H^2SN_1 \right )+4Hw.\nonumber
\end{align}
\end{proof}

By Lemma \ref{lemma_last_1}, with probability $1-(\mathbb{P}\left[\overline{E}_2\right] +4p)$ it holds that
\begin{align}
&\left(\sqrt{     \frac{\udl{\sigma}^{\reff}/n     -(\udl{\mu}^{\reff}  /n)^2  }{n}      \iota  } -\sqrt{\frac{\mathbb{V}(P_{s,a},V^*) }{n}\iota}\right) \nonumber\\
&\qquad\qquad\qquad\qquad\qquad \leq \sqrt{  \frac{9\sqrt{2}H^3 \iota^{3/2}}{n^{3/2} }  +\frac{(2H^2SA(\check{J}+ \bar{J}) + 10H^2SN_1)\iota }{n^2}   +\frac{4H\omega \iota}{n}}.\label{eq_thm2_31}
\end{align}
As a result, for $n> N_2:= c_{3}\frac{H^3\omega \iota}{\epsilon^2}+ c_{4}\frac{H^{10/3}\iota }{\epsilon^{4/3}}+c_{5}\frac{H\sqrt{ (H^2SA (\check{J}+ \bar{J})+ H^2SN_1)\iota  } }{\epsilon}$ with sufficient large constants $c_{4}$ and $c_{5}$, it holds that
\begin{align}
2\sqrt{2}\left(\sqrt{     \frac{\udl{\sigma}^{\reff}/n     -(\udl{\mu}^{\reff}  /n)^2  }{n}      \iota  } -\sqrt{\frac{\mathbb{V}(P_{s,a},V^*) }{n}\iota}\right)<\frac{\epsilon}{64H}.\label{eq_thm2_32}
\end{align}




\paragraph{The \textcircled{2} term} 
Direct computation gives that 
\begin{align}
	\frac{  \udl{ \check{\sigma}}/ \check{n}  - ( \udl{\check{\mu}} /\check{n} )^2     }{ \check{n}  }\leq \frac{ \udl{\check{\sigma}}} {\check{n}^2 } & = \frac{1}{\check{n}^2}   \sum_{i=1}^{\check{n}} \left( V_{\check{l}_i }(s_{\check{l}_i+1 })- V^{\reff }_{\check{l}_i}(s_{ \check{l}_i+1 })   \right)^2  \leq  \frac{1}{\check{n}^2}   \sum_{i=1}^{\check{n}} \left( V^{\reff }_{\check{l}_i}(s_{ \check{l}_i+1 })  - V^*(s_{\check{l}_i+1 })  \right)^2 .\label{eq_thm2_33}
\end{align} 
Also note that 
\begin{align}
&\left|\sum_{i=1}^{\check{n}} \left( \left( V^{\reff }_{\check{l}_i}(s_{ \check{l}_i+1 })  - V^*(s_{\check{l}_i+1 })  \right)^2  -\left( V^{\reff }_{\check{l}_i+1}(s_{ \check{l}_i+1 })  - V^*(s_{\check{l}_i+1 })  \right)^2   \right)\right| \nonumber\\
& \qquad\qquad\qquad \leq  2H \cdot \left| \sum_{i=1}^{\check{n}}  \left(V^{\reff }_{\check{l}_i}(s_{ \check{l}_i+1 })-V^{\reff }_{\check{l}_i +1}(s_{ \check{l}_i+1 })\right)\right| \leq 2H^2(SA(\check{J} + \bar{J})) .
\end{align}
It then follows that
\begin{align}
&\mathbb{P}\left[     	\frac{  \udl{ \check{\sigma}}/ \check{n}  - ( \udl{\check{\mu}} /\check{n} )^2     }{ \check{n}  } >  \frac{H^2(2 SN_1+2SA(\check{J}+\bar{J})) }{\check{n}^2}+ \frac{2\omega^2}{\check{n}}   \right] \nonumber
\\ \leq ~ & \mathbb{P}\left[    \sum_{i=1}^{\check{n}} \left( V^{\reff }_{\check{l}_i}(s_{ \check{l}_i+1 })  - V^*(s_{\check{l}_i+1 })  \right)^2  >  H^2(2 SN_1+2SA(\check{J}+\bar{J})) + 2 \omega^2 \check{n}\right]    \nonumber
\\ \leq ~ & \mathbb{P}\left[    \sum_{i=1}^{\check{n}} \left( V^{\reff }_{\check{l}_i+1}(s_{ \check{l}_i+1 })  - V^*(s_{\check{l}_i+1 })  \right)^2  >  2 H^2 SN_1 + 2 \omega^2 \check{n}\right]    \nonumber\\
\leq ~ & \mathbb{P}\left[    \sum_{i=1}^{\check{n}} \left( V^{\reff }_{\check{l}_i+1}(s_{ \check{l}_i+1 })  - V^*(s_{\check{l}_i+1 })  \right)^2  >   \sum_{i=1}^{\check{n} }\left(H \lambda_{\check{l}_i+1 }(   s_{\check{l}_i+1 })  + \omega \right)^2 \right]    \nonumber
\\  \leq ~& \mathbb{P}[\overline{E}_{2}], \nonumber 
\end{align}
where the last inequality is due to \eqref{eq_thm2_21}. Therefore, we have that 
\begin{align}
    \mathbb{P}\left[  \sqrt{   	\frac{  \udl{ \check{\sigma}}/ \check{n}  - ( \udl{\check{\mu}} /\check{n} )^2     }{ \check{n}  }  } >  \sqrt{\frac{H^2(2 SN_1+2SA(\check{J}+\bar{J})) }{\check{n}^2}+ \frac{2\omega^2}{\check{n}}  } \right] \leq  \mathbb{P}[\overline{E}_{2}]. \label{eq_thm2_34}
\end{align}
Note that $\check{n}\geq \frac{n}{2HB}$. For $n> N_3 = c_{6} \frac{\omega^2 H^3B \iota}{\epsilon^2} +c_{7} \frac{\sqrt{H^4BSN_1\iota}}{\epsilon}$ with large enough constants $c_{6}$ and $c_{7}$, we have that the following inequality holds with probability at least $1-\mathbb{P}[\overline{E}_2]$,
\begin{align}
 2\sqrt{2} \sqrt{\frac{  \udl{ \check{\sigma}}/ \check{n}  - ( \udl{\check{\mu}} /\check{n} )^2     }{ \check{n}  }\iota}< \frac{\epsilon}{64H} \label{eq_thm2_35}.
\end{align}

\paragraph{The \textcircled{3} term} For $n>N_4:=c_{8}\frac{H^{11/3}B\iota}{\epsilon^{4/3 }}$ with large enough constant $c_{8}$, we have 
\begin{align}
 7\left(    \frac{H\iota^{3/4}}{n^{3/4}}  +    \frac{H\iota^{3/4}}{\check{n}^{3/4}}   \right)  <\frac{\epsilon}{64H}.\label{eq_thm2_36}
\end{align}

\paragraph{The \textcircled{4} term}  For $n>N_5:=c_{9} \frac{H^3B\iota}{\epsilon}$ with large enough constant $c_{8}$, we have
\begin{align}
 5\left(   \frac{H\iota}{n}  +\frac{H\iota}{\check{n}}      \right)< \frac{\epsilon}{64H}.\label{eq_thm2_37}
\end{align}

 Combining \eqref{eq_thm2_19} with the bounds \eqref{eq_thm2_31}, \eqref{eq_thm2_32}, \eqref{eq_thm2_34}, \eqref{eq_thm2_35}, \eqref{eq_thm2_36} and \eqref{eq_thm2_37}, using the trivial bound $\clip(\check{b}_{t}(s_{t},a_{t})-b^*_{t}(s_{t},a_{t}),\frac{\epsilon}{16H} )\leq 1/(1-\gamma)$ for early stages, and summing over all possible $s,a,j$ with a union bound, we obtain that with probability $1 - SA\check{J} (2\mathbb{P}[\overline{E}_2] +4p   ) $,
 
 \begin{align}
\sum_{t\geq 1}\clip(\check{b}_{t}(s_{t},a_{t})-b^*_{t}(s_t,a_t),\frac{\epsilon}{16H} ) \leq O(\mathcal{M}_1 + \mathcal{M}_2 + \mathcal{M}_3 + \mathcal{M}_4), \label{eq_thm2_40}
\end{align}
where (noting that $\check{n} \geq n/(2HB)$ in  \eqref{eq_thm2_34}, \eqref{eq_thm2_36} and \eqref{eq_thm2_37}) 
\begin{align}
    \mathcal{M}_1 &= \sum_{s, a} \left(H \iota  + \sum_{n=\max\{ \left\lfloor \iota \right\rfloor ,1 \} }^{N_2}\sqrt{  \frac{9\sqrt{2}H^3 \iota^{3/2}}{n^{3/2} }  +\frac{(2H^2SA(\check{J}+ \bar{J}) + 10H^2SN_1)\iota }{n^2}   +\frac{4H\omega \iota}{n}}\right),\\
    \mathcal{M}_2 &= \sum_{s, a} \left(  H \iota  + \sum_{n= \max\{ \left\lfloor \iota \right\rfloor, 1 \}}^{N_3}\sqrt{  \frac{ H^4B^2(2 SN_1+2SA(\check{J}+\bar{J})) }{n^2}+ \frac{2HB\omega^2}{n} }\right),\\
     \mathcal{M}_3 &= \sum_{s, a} \left(H\iota  + \sum_{n=\max\{ \left\lfloor \iota \right\rfloor, 1\} }^{N_4}   \left(\frac{H\iota^{3/4}}{n^{3/4}}  +    \frac{H^{7/4}B^{3/4}\iota^{3/4}}{n^{3/4}}\right) \right), \\
     \mathcal{M}_4 &= \sum_{s, a} \left(H\iota  + \sum_{n= \max\{\left\lfloor \iota \right\rfloor, 1 \}}^{N_5}   \left(   \frac{H\iota}{n}  +\frac{H^2 B\iota}{n}      \right) \right).
\end{align}
Straightforward calculation shows that
\begin{align}
    \mathcal{M}_1 &\leq SA\cdot O\left(   H\iota+   N_2^{1/4}H^{3/2}\iota^{3/4}+  \ln(\frac{N_2}{\iota})\sqrt{H^2SA\check{J}+H^2SN_1 } + \sqrt{N_2H\omega\iota}    \right) \nonumber
    \\ & \leq  O\left( \frac{SAH^{5/4} \iota}{\epsilon} \right)       +\tilde{O}\Big(   \frac{SAH^{17/12}\iota}{\epsilon^{2/3}} + \frac{ (S^{3/2}A^{3/2}H^{7/4}+ S^{3/2}A^{5/4}H^{7/2}+SAH^{15/8}) \iota }{\epsilon^{1/2}} \nonumber \\
    & \quad \quad  +\frac{SA H^{7/3}\iota }{\epsilon^{1/3}}  + \frac{  (S^{5/4}A^{5/4}H^{5/2} +S^{5/4}A^{9/8}H^{3}  )\iota }{\epsilon^{1/4}}+ S^2A^2H^3\iota +S^2A^{3/2}H^{7/2}\iota\Big) ,   
  \\\  \mathcal{M}_2 &\leq  SA\cdot O\left(H\iota+    \ln(\frac{N_3}{\iota})\sqrt{ H^2B^2 (H^2SN_1+H^2SA\check{J})  } + \sqrt{N_3 HB\omega^2\iota} \right)   \nonumber
  \\ & \leq O \left(    \frac{SAH^2\iota}{\epsilon}     \right) +\tilde{O} \Big(    \frac{  S^{3/2}A^{5/4}H^{17/4}\iota }{\epsilon^{1/2}} +   S^2A^{3/2}H^{9}\iota+S^2A^2H^7\iota   \Big) , 
  \\  \mathcal{M}_3 &\leq SA\cdot O\left( H\iota + N_4^{1/4}  H^{7/4}B^{3/4}\iota^{3/4}     \right)   \leq  O\left(    \frac{SA H^{59/12}\iota }{\epsilon^{1/3}} +SAH\iota \right) \\
    \mathcal{M}_4 &\leq  SA\cdot O\left( H\iota+ \ln(\frac{N_5}{\iota})H^2B\iota    \right)\leq \tilde{O}\left(SAH^5\iota \right). 
\end{align}
Finally, together with \eqref{eq_thm2_40}, we conclude that
\begin{align}
   & \sum_{t\geq 1}\clip(\check{b}_{t}(s_{t},a_{t})-b^*_{t}(s_t,a_t),\frac{\epsilon}{16H} ) \nonumber \\
   & \leq  O\left(   \frac{SAH^2\iota}{\epsilon} \right) +\tilde{O}\left(   \frac{ S^{3/2}A^{3/2} H^{17/4} \iota  }{\epsilon^{1/2}} +  \frac{SAH^{59/12} \iota}{ \epsilon^{1/3}} +\frac{   S^{5/4}A^{5/4}H^{3}\iota     }{\epsilon^{1/4}}   +S^2A^2H^{9}\iota \right).\label{eq_thm2_38}
\end{align}
 
 \subsubsection{Proof of Lemma \ref{lemma_bd_ups}}\label{app:pf-bd-ups}
 \medskip
\noindent \textbf{Lemma~\ref{lemma_bd_b_1} (restated).}
{ With probability $1-(\mathbb{P}\left[\overline{E}_2\right]+p)$, it holds that
\begin{align}
    \sum_{t\geq 1}\upsilon_{t}\leq  64\log(\frac{16N_0 H^2}{\epsilon})N_1 .\nonumber
\end{align}
}

By definition, we have that
\begin{align}
 &   \sum_{t\geq 1}\upsilon_t = \sum_{t\geq 1}\sum_{s}P_{s_{t},a_{t},s}\mathrm{clip}\left( \frac{1}{1-\gamma}\left( \frac{1}{\bar{N}_t(s,a)}\sum_{i=1}^{\bar{N}_{t}(s,a)}\lambda_{l_i(s_t,a_t)}(s) \right) ,\frac{\epsilon}{16H}\right) \nonumber
 \\ & \leq H\sum_{s}\sum_{t\geq 1}P_{s_{t},a_{t},s}\mathrm{clip}\left( \left( \frac{1}{\bar{N}_t(s,a)}\sum_{i=1}^{\bar{N}_{t}(s,a)}\lambda_{l_i(s_t,a_t)}(s) \right) ,\frac{\epsilon}{8H^2}\right).\label{eq-bd-ups1}
\end{align}
Let $\tilde{T}(s,a,s')$ be the visit count of $(s,a)$ before the smallest time $t$ such that $\lambda_{t}(s')=0$. Then we have that
\begin{align}
    \frac{1}{\bar{N}_t(s,a)}\sum_{i=1}^{\bar{N}_{t}(s,a)}\lambda_{l_i(s_t,a_t)}(s)  \leq \mathbb{I}\left[\bar{N}_{t}(s,a) \leq (1+\frac{1}{H})\tilde{T}(s,a,s')  \right] + \frac{\tilde{T}(s,a,s')}{\bar{N}_t(s,a)}.\nonumber
\end{align}
Noting that $\bar{N}_t(s,a)\leq N_t(s,a)\leq (1+\frac{1}{H})\bar{N}_t(s,a)$, we obtain that
\begin{align}
    \mathrm{clip}\left( \left( \frac{1}{\bar{N}_t(s,a)}\sum_{i=1}^{\bar{N}_{t}(s,a)}\lambda_{l_i(s_t,a_t)}(s) \right) ,\frac{\epsilon}{8H^2}\right)\leq \mathbb{I}\left[N_t(s,a) \leq 4\tilde{T}(s,a,s')  \right]  + \mathrm{clip}(\frac{2\tilde{T}(s,a,s') }{N_t(s,a)},\frac{\epsilon}{8H^2}).\nonumber
\end{align}

Combining this with \eqref{eq-bd-ups1},  with probability $1-p$ it holds that
\begin{align}
    & \sum_{t\geq 1}\upsilon_t  \leq   H\sum_{s}\sum_{t\geq 1}P_{s_t,a_t,s'}\mathbb{I}\left[ N_t(s_t,a_t)\leq 4\tilde{T}(s_t,a_t,s') \right] +H\sum_{s'}\sum_{t\geq 1}P_{s_t,a_t,s'}\mathrm{clip} (\frac{2\tilde{T}(s_t,a_t,s') }{N_t(s_t,a_t)},\frac{\epsilon}{8H^2})  \nonumber
     \\ & \leq 4H\sum_{s,a,s'}P_{s,a,s'}\tilde{T}(s,a,s')+    4H\sum_{s,a,s'}P_{s,a,s'} \tilde{T}(s,a,s') \log(\frac{16\tilde{T}(s,a,s')H^2}{\epsilon}) \nonumber
     \\ & \leq 8\log(\frac{16N_0 H^2}{\epsilon})\sum_{s,a,s'}P_{s,a,s'}\tilde{T}(s,a,s') \nonumber
     \\ & = 8\log(\frac{16N_0 H^2}{\epsilon})\sum_{s'}\sum_{t\geq 1}P_{s_t,a_t,s'}\lambda_t(s')\nonumber 
     \\ & \leq 32\log(\frac{16N_0 H^2}{\epsilon})(\sum_{t\geq 1}\lambda_{t}(s_{t+1})
     \\ & \leq   64\log(\frac{16N_0 H^2}{\epsilon})N_1.
\end{align}
The second last inequality holds with probability $1-p$ by Lemma \ref{lemma_berbound}, and the last inequality is by the facts  $\sum_{t\geq 1}\lambda_t(s_t) \leq SN_1$  and $\sum_{t\geq 1}(\lambda_{t}(s_{t+1} -\lambda_{t+1}(s_{t+1}) )\leq S$. The proof is completed.

\end{document}